\documentclass[11pt]{article}

\usepackage[margin=1in]{geometry}   % 1-inch margins
\usepackage{amsmath,amssymb,amsthm} % maths
\usepackage{natbib}
\usepackage[dvipsnames]{xcolor}
\usepackage{hyperref}
\usepackage{url}
\usepackage{setspace}     

\newcommand{\norm}[1]{\left \lVert #1 \right \rVert}

\usepackage{amsmath,amsfonts,amsthm,bm}
\usepackage{bbm}
\usepackage[ruled,vlined]{algorithm2e}

\def\eqref#1{equation~\ref{#1}}
\def\1{\bm{1}}

\DeclareMathAlphabet{\mathsfit}{\encodingdefault}{\sfdefault}{m}{sl}
\SetMathAlphabet{\mathsfit}{bold}{\encodingdefault}{\sfdefault}{bx}{n}

\newcommand{\E}{\mathbb{E}}
\newcommand{\Xscr}{\mathcal{X}}
\newcommand{\Yscr}{\mathcal{Y}}
\newcommand{\Hscr}{\mathcal{H}}

\newcommand{\R}{\mathbb{R}}

\theoremstyle{definition}
\newtheorem{exmp}{Example}[section]
\newtheorem{proposition}{Proposition}[section]
\newtheorem{theorem}{Theorem}[section]

\newtheorem{lemma}{Lemma}[section]
\newtheorem{corollary}{Corollary}[section]

\usepackage[pdftex]{graphicx}

\AtBeginDocument{\onehalfspacing}
\title{Self-Normalized Resets for Plasticity in Continual Learning\thanks{Code: \url{https://github.com/ajozefiak/SelfNormalizedResets}}}
\author{Vivek F. Farias  \\
Sloan School of Management\\
Massachusetts Institute of Technology \\
\texttt{vivekf@mit.edu} \\
\and
Adam D. Jozefiak \\
Operations Research Center \\
Massachusetts Institute of Technology \\
\texttt{jozefiak@mit.edu} \\}

\date{This version: \today}

\begin{document}
\maketitle

\begin{abstract}
Plasticity Loss is an increasingly important phenomenon that refers to the empirical observation that as a neural network is continually trained on a sequence of changing tasks, its ability to adapt to a new task diminishes over time. We introduce Self-Normalized Resets (SNR), a simple adaptive algorithm that mitigates plasticity loss by resetting a neuron’s weights when evidence suggests its firing rate has effectively dropped to zero. Across a battery of continual learning problems and network architectures, we demonstrate that SNR consistently attains superior performance compared to its competitor algorithms. We also demonstrate that SNR is robust to its sole hyperparameter, its rejection percentile threshold, while competitor algorithms show significant sensitivity. SNR’s threshold-based reset mechanism is motivated by a simple hypothesis test that we derive. Seen through the lens of this hypothesis test, competing reset proposals yield suboptimal error rates in correctly detecting inactive neurons, potentially explaining our experimental observations. We also conduct a theoretical investigation of the optimization landscape for the problem of learning a single ReLU. We show that even when initialized adversarially, an idealized version of SNR learns the target ReLU, while regularization based approaches can fail to learn. 
\end{abstract}

\section{Introduction}

{\em Plasticity Loss} is an increasingly important phenomenon studied broadly under the rubric of continual learning \citep{dohare2024loss}. This phenomenon refers to the empirical observation that as a neural network is continually trained on a sequence of changing tasks, its ability to adapt to a new task diminishes over time. While this is distinct from the problem of catastrophic forgetting (also studied under the rubric of continual learning \citep{goodfellow2013empirical, kirkpatrick2017overcoming}), it is of significant practical importance. In the context of pre-training language models, an approach that continually trains models with newly collected data is preferable to training from scratch \citep{ibrahim2024simple,wu2024continual}. On the other hand, the plasticity loss phenomenon demonstrates that such an approach will likely lead to models that are increasingly unable to adapt to new data. Similarly, in the context of reinforcement learning using algorithms like TD, where the learning tasks are inherently non-stationary, the plasticity loss phenomenon results in actor or critic networks that are increasingly unable to adapt to new data \citep{lyleunderstandingRL}. Figure~\ref{fig:PM_intro} illustrates plasticity loss in the `Permuted MNIST' problem introduced by \citet{goodfellow2013empirical}. 

% \textcolor{red}{TODO: Perhaps we should briefly comment on the causes of plasticity here, or after immediately after the first two sentences below.}

 One formal definition of plasticity measures the ability of a network initialized at a specific set of parameters to fit a random target function using some pre-specified optimization procedure. In this sense, random parameter initializations (eg. \citet{lyle2024disentangling}) are known to enjoy high plasticity. This has motivated two related classes of algorithms that attempt to mitigate plasticity loss. The first explicitly `resets' neuron's that are deemed to have low `utility' \citep{dohare2023loss,sokar2023dormant}. A reset re-initializes the neurons input weights and bias according to some suitable random initialization rule, and sets the output weights to zero; algorithms vary in how the utility of a neuron is defined and estimated from online data. A second class of algorithms perform this reset procedure implicitly via regularization \citep{ash2020warm,kumar2023maintaining}. These latter algorithms differ in their choice of what to regularize towards, with choices including the original network initialization; a new randomly drawn initialization; or even zero. The aforementioned approaches to mitigating plasticity loss attempt to adjust the training process; other research has studied the role of architectural and optimizer hyperparameter choices. Across all of the approaches to mitigating plasticity loss described above, no single approach is yet to emerge as both robust to hyperparameter choices, and simultaneously performant across benchmark problems.

Several correlates of plasticity loss have been identified such as neuron inactivity, feature or weight–rank collapse, increasing weight norms, and loss of curvature in the loss surface \citep{dohare2021continual, lyle2023understanding, sokar2023dormant, lewandowski2023curvature, kumar2020implicit}. Despite the extensive documentation of the phenomenon, the literature remains almost entirely empirical, offering little theoretical insight into \textit{why} plasticity fades or \textit{how} best to restore it. To address this gap, we analyze the simplest non-trivial continual learning model: a single ReLU-activated neuron trained by gradient descent on adversarially selected targets. We show that gradient descent, even with $L_2$ regularization, incurs non-vanishing average regret, whereas gradient descent augmented with an idealized reset-oracle achieves $\mathcal{O}(d^2\ln{T}/T)$ average regret. Although the problem of learning a single ReLU-activated neuron has been studied extensively, no prior work considers the case where targets are chosen adversarially with knowledge of the initialization. Our paper is the first to analyze this problem under this adversarial regime.

Having thus characterized plasticity loss for the fundamental primitive of a single ReLU-activated neuron, we turn to the practical question of how to detect when a neuron has effectively ``died'' as training data arrives sequentially. To this end, we cast the problem of detecting a dead neuron as an optimal hypothesis‐test problem. Specifically, when minimizing type-1 and type-2 error and a delay penalty of detecting a dead neuron, we show that the optimal reset criterion is to define a reset-threshold that is proportional to the neuron's nominal firing rate $p$. This optimal solution is precisely our self-normalized resets (SNR) algorithm. In contrast to SNR, existing reset-criteria define a single reset-threshold or reset-frequency for a neural network. We make evident the efficacy of SNR in relation to such reset-criteria by showing that for the problem of detecting neuron death of two neurons with nominal firing rates $p_1 < p_2$, the error rate under an (optimal) fixed threshold grows arbitrarily larger than the error rate under SNR as (a) the penalty for delays shrinks to zero or (b) the gap between the nominal firing rates $p_1$ and $p_2$ widens. Concretely, we introduce the SNR algorithm and distill the key intuition behind it as follows:

% \textcolor{red}{TODO: It is obvious, but should we more forcefully say that we run the hypothesis that plasticity loss is due to neuron death?}

Given some point process consider the task of distinguishing between the hypotheses that this point process has a positive rate (the null hypothesis), or a rate that is identically zero with a penalty for late rejection or acceptance. An optimal test here takes the following simple form: 
we reject the null hypothesis as soon as the time elapsed without an event exceeds some percentile of the inter-arrival time under the null hypothesis and otherwise accept immediately upon an event. Viewing the firing of a neuron as such a point process, we propose to reset a neuron based on a rejection of the hypothesis that the the neuron is firing at a positive rate. We use the histogram of past inter-firing times as a proxy of the inter-arrival time distribution under the null hypothesis. This exceedingly simple algorithm is specified by a single hyperparameter: the rejection percentile threshold. We refer to this procedure as self-normalized resets (SNR) and argue this is a promising approach to mitigating plasticity loss.

% \textcolor{blue}{TODO: We could add a subsection on the theory in order to summarize and introduce contributions and lay some high-level context. This is for the learning of a single ReLU activated-neuron and the optimal hypothesis test.}

% \textcolor{red}{TODO: The summary of contributions is quite sparse for a journal submission and perhaps we should put the theory first before the experiments.}

\subsection{Contributions}

With the objective of studying the phenomenon of plasticity loss from a theoretical standpoint and thus motivating an efficient remedy, we make the following contributions.

\begin{enumerate}
    \item \textbf{The Self-Normalized Resets Algorithm.} We introduce the Self-Normalized Resets (SNR) algorithm and argue its efficacy from both a theoretical and empirical standpoint.
    \item \textbf{Characterizing Plasticity Loss and the Efficacy of Neuron-Resets.} We conduct a theoretical investigation of the optimization landscape for the problem of learning a single RELU-activated neuron with gradient descent. While existing literature has solely considered the case of non-adversarial target weights in which the target weights are first fixed and then the network weights are sampled, we introduce the setting of adversarial target weight selection in which the network weights are first sampled and then the target weights are selected adversarially with knowledge of the network weight initialization. Crucially, this setting captures the essence of continual learning, whereby consecutive learning tasks can be unrelated or even adversarially selected. Under this setting, we show in Theorem \ref{cor:NegativeResultFullIntro} that gradient descent, even with $L_2$ regularization, attains non-vanishing average regret. In addition, in Theorem \ref{thm:PositiveRegretGuaranteeIntro} we show that gradient descent with access to an idealized reset-oracle attains $O(d^2\ln{T}/T)$ expected average regret. Our theoretical analysis characterizes plasticity loss, for the simple primitive of a single neuron, as being a consequence of network weights learned for one task yielding a poor local minimum for a subsequent task. Moreover, our analysis establishes neuron resets to be an effective remedy to plasticity loss while $L_2$-regularization is insufficient. 
    \item \textbf{Deriving an Optimal Neuron-Reset Mechanism.} To answer the question of what is an effective reset scheme, we model the problem of detecting neuron death as that of an optimal hypothesis test with the objective of minimizing type-1 and type-2 error and a delay penalty (scaled by $\lambda$). In Proposition \ref{prop:prop1} we show that for a neuron with a nominal firing rate $p > 0$, the optimal reset threshold is the $1-\lambda(p-\lambda)^{-1}$ percentile of a Geometric$(p)$ distribution; this is precisely the SNR reset heuristic. We juxtapose the SNR reset criterion with existing neuron-reset schemes, which define a fixed reset-threshold or reset-frequency $r^*$, by considering the problem of detecting neuron death of two neurons with nominal firing rates $p_1 < p_2$. Specifically, in Proposition \ref{prop:prop2} we show that the total error incurred by a fixed threshold $r^*$ to that under SNR scales like \[
\Omega\left(
\exp
\left(
\log(\alpha(1-\alpha)^{-1})
\left( 
-\frac{1}{2} + \frac{1}{2} \frac{\log (1-p_1)}{\log(1-p_2)}
\right)
\right)
\right)
\] where $\alpha$ relates the delay penalty factors according to $\lambda_1 = \alpha p_1$ and $\lambda_2 = \alpha p_2$.

    \item \textbf{Empirically Validating SNR.} We demonstrate superior performance on four benchmark problems studied in \citep{dohare2023loss,kumar2023maintaining} using MLP, CNN, and ViT architectures. Interestingly, there is no single closest competitor to SNR across these problems. Many competing approaches also show significant sensitivity to the choice of hyperparameters; SNR does not. Additionally we introduce a new problem, Permuted Shakespeare, to elucidate similar plasticity loss phenomena in the context of language models, and show similar relative merits for SNR.
\end{enumerate}

% \begin{enumerate}
% \item We demonstrate superior performance on four benchmark problems classes studied in \citep{dohare2023loss,kumar2023maintaining}. Interestingly, there is no single closest competitor to SNR across these problems. Many competing approaches also show significant sensitivity to the choice of hyperparameters; SNR does not. We introduce a new problem to elucidate similar plasticity loss phenomena in the context of language models, and show similar relative merits for SNR. 
% \item We conduct a theoretical investigation of the optimization landscape for the problem of learning a single ReLU. We show that while (an idealized version of) SNR learns the target ReLU, regularization based approaches can fail to learn in this simple setting.
 % \end{enumerate}
 
 \subsection{Related Literature} 

 %The problem of catastrophic forgetting was the initial focus of continual learning, generating a series of algorithms, such as Elastic Weight Consolidation, that regularize a network's weights towards those learned on earlier tasks \citep{goodfellow2013empirical, kirkpatrick2017overcoming}.
 
 The phenomenon of plasticity loss was discovered in the context of transfer learning \citep{ash2020warm, zilly2021plasticity, achille2017critical}. \citet{achille2017critical} showed that pre-training a network on blurred CIFAR images reduces its ability to learn on the original images. In a similar vein, \citet{ash2020warm} showed that pre-training a network on 50\% of a training set followed by training on the complete training set reduces accuracy relative to a network that forgoes the pre-training step. More recent literature has focused on problems that induce plasticity loss while training on a sequence of hundreds of changing tasks, such as Permuted MNIST and Continual ImageNet in \citet{dohare2021continual}, capturing the necessity to learn indefinitely. 
 
 % Since these initial observations a series of algorithms for maintaining plasticity have been developed along with benchmark problems that aim to elucidate its causes.
 
 % More recent literature, such as \citet{dohare2021continual}, has introduced problems that induce plasticity loss while training a network for hundreds of tasks, rather than simply a two-stage training process. Plasticity loss is not only a phenomenon arising in a two-stage learning setup
 
 \textbf{Correlates of Plasticity Loss.}  The persistence of plasticity loss across a swathe of benchmark problems has elucidated a search for its cause. Several correlates of plasticity loss have been well observed, namely neuron inactivity, feature or weight rank collapse, increasing weight norms, and loss of curvature in the loss surface \citep{dohare2021continual, lyle2023understanding, sokar2023dormant, lewandowski2023curvature, kumar2020implicit}. The exact cause of plasticity loss remains unclear and \citet{lyle2023understanding} have shown that for any correlate an experiment can be constructed in which its correlation with plasticity loss is negative. Nonetheless, these correlates have inspired a series of algorithms and interventions with varying degrees of success in alleviating the problem. However, none is consistently performant across architectures and benchmark problems. 
 
 %Our empirical results provide evidence that carefully resetting inactive neurons is sufficient in maintaining plasticity in feedforward and convolutional networks.
 
 \textbf{Reset Methods.} Algorithms that periodically reset inactive or low-utility neurons have emerged as a promising approach \citep{dohare2023loss, sokar2023dormant, nikishin2022primacy}. Continual Backprop (CBP) \citep{dohare2023loss} is one such method which tracks a utility for each neuron, and according to some reset frequency $r$, it resets the neuron with minimum utility in each layer. CBP's utility is a discounted average product of a neuron's associated weights and activation, a heuristic inspired by the literature on network pruning. Another algorithm is ReDO \citep{sokar2023dormant}, where on every ${1}/{r}^{\text{th}}$ mini-batch, ReDO computes the average activity of each neuron and resets those neurons whose average activities are small relative to other neurons in the corresponding layer, according to a threshold hyperparameter. Two defining characteristics of CBP and ReDO are a fixed reset rate and that neurons are reset relative to the utility of other neurons in their layer. As we will see, these proposals result in sub-optimal error rates, in a sense we make precise later. 
 
 \textbf{Regularization Methods.} L2 regularization has been shown to reduce plasticity loss, but is insufficient in completely alleviating the phenomenon \citep{dohare2021continual,lyle2023understanding}. While L2 regularization limits weight norm growth during continual learning, it can exacerbate weight rank collapse due to regularization towards the origin. One successful regularization technique is Shrink and Perturb (S\&P) \citep{ash2020warm}, which periodically scales the network's weights by a shrinkage factor $p$ followed by adding random noise to each weight with scale $\sigma$. Another approach is to perform L2 regularization towards the initial weights referred to as L2 Init \citep{kumar2023maintaining}. These methods can be viewed as variants of L2 regularization that regularize towards a random initialization and the original initialization, respectively. These methods limit the growth of weight norms while maintaining weight rank and neuron activity by regularizing towards a high-plasticity parameterization.
 
 \textbf{Architectural and Optimizer Modifications.} Architectural modifications such as layer normalization \citep{ba2016layer} and the use of concatenated ReLU activations have been shown to improve plasticity to varying degrees across network architectures and problem settings \citep{lyle2023understanding, kumar2023maintaining}. Additionally, tuning Adam hyperparameters to improve the rate at which second moment estimates are updated has been explored with some success in \citet{lyle2023understanding}. 

 \textbf{Learning a single ReLU-Activated Neuron with Gradient Descent.} Learning a single ReLU-activated neuron with gradient-based methods is a fundamental problem that has been extensively studied in the theory literature. Early works analyzed the bias-free case, for example, \cite{yehudai2020learning} showed that with a well-spread input distribution (e.g. uniform over a sphere), gradient methods can recover the target weight vector with high probability. \cite{vardi2021learning} extended this line of work to a neuron with bias, finding that the addition of a bias significantly complicates the optimization landscape. Nevertheless, under appropriate assumptions on the data distribution (such as isotropic or sufficiently full support in every direction) and on the initialization, they proved that gradient descent will still converge to the global minimum, with high probability at a linear rate. Beyond single neurons, more general models like max-affine regression (which includes learning a ReLU as a special case) have also been studied – \cite{kim2024max} show that first-order methods can achieve linear convergence for such piecewise-linear regression problems, assuming i.i.d. sub-Gaussian inputs and a suitable initialization. Importantly, all these results assume the target neuron’s parameters are fixed a priori (non-adversarial) and the data are drawn from a benign distribution. To our knowledge, no prior work addresses the setting of adversarially selected target weights – this paper is the first to analyze gradient descent for learning a ReLU neuron in an adversarial target regime, where the ground-truth weight vector may be chosen by an adversary (with knowledge of the initialization).

\subsection{Roadmap of the Paper}

We outline the roadmap for the rest of this paper.

 In Section \ref{sec:algo} we formally define plasticity loss as an empirical phenomenon that corresponds to an increasing average expected regret over a stream of training data. In addition, we formally define the Self-Normalized Resets algorithms and present a simple benchmark problem, Permuted MNIST, that illustrates the phenomenon of plasticity loss and the efficacy of SNR.
 
 In Section \ref{sec:theory} we present our theoretical analysis. In Section \ref{sec:LearningReLUPrelim} we formally introduce the problem of learning a single ReLU-activated neuron. In Section \ref{sec:AdversarialTargetWeights} we formally define the setting of adversarial target weights and present Theorem \ref{cor:NegativeResultFullIntro}. In Section \ref{sec:ResetOracle} we introduce the idealized reset oracle, define the gradient descent algorithm augmented with such a reset oracle, and present Theorem \ref{thm:PositiveRegretGuaranteeIntro}. In Section \ref{sec:hyp_test} we formally introduce the hypothesis test problem, that captures the problem of neuron-death detection, and we present Proposition \ref{prop:prop1} and Proposition \ref{prop:prop2}.

 In Section \ref{sec:experiments} we present our empirical results demonstrating the efficacy of SNR across a series of continual learning benchmarks and architectures. In addition, we discuss pertinent phenomena that we observe in our experiments.

 Finally, in Section \ref{sec:discussion} we conclude with a Discussion and Limitations section, summarizing the our key contributions and noting limitations.
 
 % In Section ... we introduce the problem of learning a single ReLU-activated neuron with gradient descent. In Section ... we define the setting non-adversarial target weights are state the state-of-the-art results and resulting $O(...)$ regret guarantee. In Section ... we define the setting of adversarial target weights and present our negative result, Theorem XX, which demonstrates that learning a single-ReLU activated neuron, even with $L_2$-regularization results in non-vanishing regret. In Section XX we introduce our idealized reset oracle and define gradient descent augmented with such a reset-oracle. We then present our key theoreticalcontribution, Theorem Xx, which establishes that gradient descent with access to a reset oracle attains vanishing regret $O(d^2\ln{T}/T)$. }
\section{Algorithm}\label{sec:algo}

To make ideas precise, consider a sequence of training examples $(X_t,Y_t) \in \Xscr \times \Yscr$, drawn from some distribution $\mu_t$. Denote the network by $f: \Xscr \times \Theta \rightarrow \Yscr$, and let $l: \Yscr \times \Yscr \rightarrow \R$ be our loss function. Denote by $H_t \in \Hscr_t$, the history of network weights and training examples up to time $t$, and assume access to an optimization oracle $O_t: \Hscr_{t-1} \rightarrow \Theta$ that maps the history of weights and training examples to a new set of network weights. As a concrete example, $O_t$ might correspond to stochastic gradient descent. 

Let $\theta^*_t$ minimize 
$\E_{\mu_t}
[ 
l(f(X_t; \theta),Y_t)
]
$, denote $\Theta_t = O_t(H_{t-1})$, and consider average expected regret 
\[
\frac{1}{T} 
\sum_t 
\E_{\mu_t}
[ 
l(f(X_t; \Theta_t),Y_t)
]
 - 
\E_{\mu_t}
[ 
l(f(X_t; \theta^*_t),Y_t)
]  
\] 
{\em Plasticity loss} describes the phenomenon where, for certain continual learning processes $\Theta_t$, such as those corresponding to SGD or Adam, average expected regret increases over time, even for benign choices of $\mu_t$.\footnote{Specifically, one canonical choice of the sequence of measures $\mu_t$ considered in all of the literature on this topic is dividing $T$ into intervals, each of length, say $\Delta$, and having $\mu_t$ be constant and equal to $\mu_i$ over the $i$th such interval. If $\mu_i$ is itself drawn randomly from some distribution of measures and $\Delta$ scales faster than a constant with $T$, we would expect average expected regret to scale like a constant; this is certainly the case if the optimization problem defining $\theta_t^*$ is convex; in which case that constant is zero.}    
To make these ideas concrete, it is worth considering an example of the above phenomenon reported first by \citet{dohare2021continual}.

%Given some loss function $l: \Theta \times \Xscr \rightarrow \R$, let $\theta^*_t$ minimize loss under $\mu_t$. Denoting by $H_t$ the history of example and losses observed up to time $t$, we assume access to an optimization oracle, $\Theta_t: H_t \rightarrow \Theta$, that maps the history of training examples and observed losses to a new set of model weights in $\Theta$. For example, $\Theta_t$ might correspond to running a single stochastic gradient step on each consecutive training example. We will often refer to the stochastic process $\Theta_t$ as a continual learning process, and with a slight abuse of notation suppress the dependence on $H_t$.  
%
%Within this formalism, consider average expected regret under a continual learning processes, $\Theta_t$:
%\[
%\frac{1}{T} 
%\sum_t \E_{\mu_t}[l(X_t,\Theta_t)] - \E_{\mu_t}[l(X_t,\theta^*_t)].  
%\] 
%{\em Plasticity loss} describes the phenomenon where, under certain continual learning processes (such as those obtained by running Adam or SGD), average expected regret increases over time. To make these ideas concrete, it is worth considering an example of the above phenomenon reported first by \cite{}: 

\begin{exmp}[The Permuted MNIST problem]
\label{eg:mnist}
Consider a sequence of `tasks' presented sequentially to SGD, wherein each task consists of 10000 images from the MNIST dataset with the pixels permuted. SGD trains over a single epoch on each task before the subsequent task is presented. Figure~\ref{fig:PM_intro} measures average accuracy on each task; we see that average accuracy decreases over tasks. The figure also shows a potential correlate of this phenomenon: the number of `dead' or inactive neurons\footnote{this notion is formalized in Section~\ref{sec:hyp_test}} in the network increases as training proceeds, diminishing the network's effective capacity. 
\end{exmp}

\begin{figure}[htbp]\label{fig:PM}
        \centering
        \includegraphics[width=0.38\linewidth]{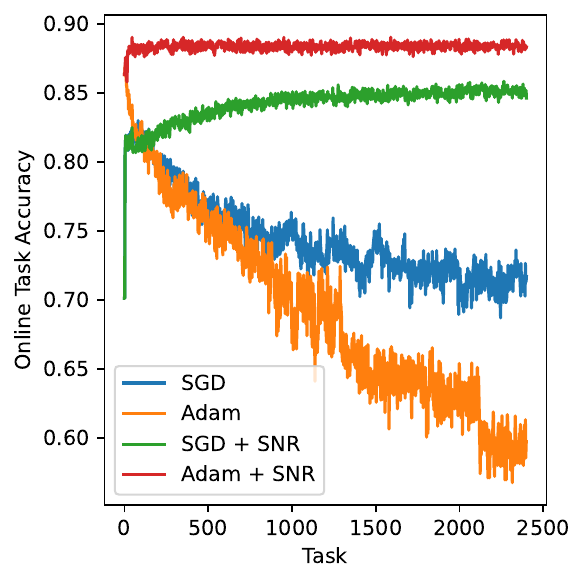}
        \includegraphics[width=0.38\linewidth]{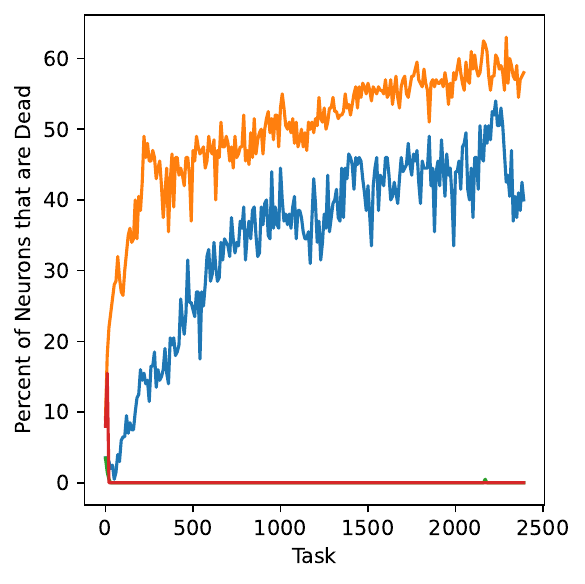}
        \caption{Illustration of plasticity loss and its mitigation by SNR during training of a multilayer perceptron on the Permuted MNIST problem for a single random seed. For this figure, a neuron is declared dead if it has not fired for the last 1000 consecutive training examples.}
        \label{fig:PM_intro}
    \end{figure}

One hypothesis that seeks to explain plasticity loss is that the network weights obtained from minimizing loss over some task yield poor initializations for a subsequent task, leading to the inactive neurons we observe in the above experiment. On the other hand random weight initializations are known to work well \citep{glorot2010understanding}, suggesting a natural class of heuristics: re-initialize inactive neurons. Of course, the crux of any such algorithm is determining whether a neuron is inactive in the first place, and doing so as quickly as possible.

To motivate our algorithm, SNR, consider applying the network $f(\cdot; \theta_t^*)$ to a hypothetical sequence of training examples drawn i.i.d. from $\mu_t$ indexed by $s$. Let $Z^{\mu_t} _{s,i}$ indicate the sequence of activations of neuron $i$, and let $A^{\mu_t}_i$ be a random variable distributed as the random time between any two consecutive activations over this hypothetical sequence of examples. Now turning to the {\em actual} sequence of training examples, let $a^t_i$ count the time since the last firing of neuron $i$ prior to time $t$.  Our (idealized) proposal is then exceedingly simple: reset neuron $i$ at time $t$ iff $\mathbb{P}(A_i^{\mu_t} \geq a^t_i) \leq \eta$ for some suitably small threshold $\eta$. We dub this algorithm {\em Self-Normalized Resets} and present it as Algorithm~\ref{alg:SNR}. The algorithm requires a single hyper parameter, $\eta$. Of course in practice, the distribution of $A_i^{\mu_t}$ is unknown to us, and so an implementable version of Algorithm~\ref{alg:SNR} simply approximates this distribution with the histogram of inter-firing times of neuron $i$ prior to time $t$.\footnote{As opposed to tracking the histogram itself, we simply track the mean inter-firing time, and assume $A_i^{\mu_t}$ is geometrically distributed with that mean. This requires tracking just one parameter per neuron. Our mean estimate is itself computed over a fixed length trailing window motivated by the change-point detection approach of \cite{besbes2011minimax}; the length of this window is a hyper-parameter.}  

\begin{algorithm}[t]
\label{alg:SNR}
\DontPrintSemicolon  % Uncomment if you don't want semicolons to print
\caption{SNR: Self-Normalized Resets}
\SetKwInput{KwInput}{Input}                % Set the Input
\SetKwInput{KwOutput}{Output}              % set the Output
\KwInput{Reset percentile threshold $\eta$}
\textbf{Initialize:} Initialize weights $\theta_0$ randomly. Set inter-firing time $a_i = 0$ for each neuron $i$ \;
\For{\( {\rm each \ training \ example} \ x_t \)}{
        \textbf{Forward Pass}: Evaluate $f(x_t;\theta_t)$. Get neuron activations $z_{t,i}$ for each neuron $i$ \;
        \textbf{Update inter-firing times:} For each neuron $i$, $a_i \gets a_i+1$ if $z_{t,i} = 0$. Otherwise, $a_i \gets 0$ \;
        \textbf{Optimize}: $\theta_{t+1} \gets O_t(H_t)$  \;
        \textbf{Resets:} For each neuron $i$, reset if 
        $\mathbb{P}(A_i^{\mu_t} \geq a_i) \leq \eta $. 
}
\end{algorithm}

\section{Theoretical Analysis}\label{sec:theory}

In this section we provide a two-fold theoretical analysis. We begin in Section \ref{sec:ReLUAnalysis} by analyzing the problem of learning a single ReLU-activated neuron with gradient descent under an adversarial selection of the target weights. This is the simplest possible (component of a) neural network and it captures the essence of continual learning, whereby the adversarial selection of the target weights models task transitions. We begin by showing that gradient descent with L2 regularization attains $\Omega(1)$ expected average regret. This analysis characterizes plasticity loss, at the neuron-level, as being a consequence of network weights learned for previous tasks forming poor initializations for subsequent tasks. Secondly, we provide a positive result demonstrating that gradient descent with access to a reset oracle attains $O(d^2\ln{T} / T)$ expected average regret, demonstrating that neuron resets are an effective mechanism for identifying poor initializations and thus remedying plasticity loss.

Having characterized plasticity loss and shown that resets are an effective remedy in Section \ref{sec:ReLUAnalysis}, in Section \ref{sec:hyp_test} we answer the question of what makes an effective reset oracle in practice when data is arriving sequentially. To this end, we frame the problem of detecting an inactive neuron as that of an optimal hypothesis test. We show that an optimal reset scheme must have a reset threshold that is proportional to the neuron's firing rate; this is precisely the SNR mechanism. In contrast, existing competitor reset schemes define a fixed reset threshold or frequency. Our experiments in Section \ref{sec:experiments} show that such reset schemes attain performance inferior to that of SNR. To shed intuition on this observation, we show that under this optimal hypothesis test model, one can incur arbitrarily large error relative to SNR if attempting to detect two dead neurons with distinct firing rates with a single threshold.

\subsection{Characterizing Plasticity Loss in a Single ReLU and the Promise of Resets}\label{sec:ReLUAnalysis}

In this section, we provide a theoretical analysis of learning a single ReLU-activated neuron with gradient descent and with access to a reset oracle, where the target neuron's weights are chosen adversarially after the network's weight initialization is fixed. This adversarial selection serves as an abstraction of task transitions in a continual learning setting, allowing for scenarios where successive tasks may be unrelated—or even deliberately chosen to be adversarial. In contrast, prior work has solely considered settings where the target neuron's weights are fixed before the network's initialization is sampled \citep{vardi2021learning, kim2024max, yehudai2020learning}. While such literature demonstrate the effectiveness of first-order methods for learning a single neuron and the benefits of modern weight initializations, they do not capture the challenges of task transitions, where the initialization for a new task may be influenced by prior learning.

\subsubsection{Preliminaries}\label{sec:LearningReLUPrelim}

To study the learning dynamics of a single ReLU-activated neuron, we consider the expected squared error:  
\begin{equation}\label{eqn:Objective}
    L(w) = \mathbb{E}[(\sigma(w^{\top}x) - \sigma(v^{\top}x))^2],
\end{equation}  
where \( \sigma(z) = \max(0,z) \) is the ReLU activation function, \( v \) is the target neuron's weight vector, and the expectation is taken over \( x \in \mathbb{R}^{d+1} \), sampled from \( \mathcal{N}(0, I_d) \times \{1\} \), where \( I_d \) is the \( d \)-dimensional identity matrix. The gradient of the loss is then given by  
\begin{equation*}
    \nabla L(w) = \mathbb{E}[2(\sigma(w^{\top}x) - \sigma(v^{\top}x)) x \mathbbm{1}_{\{w^{\top}x \geq 0\}}].
\end{equation*} where we define the subgradient of \( \sigma \) as \( \sigma'(z) = \mathbbm{1}_{\{z \geq 0\}} \), which corresponds to a particular choice of subgradient at \( z = 0 \). Next, we introduce the \( L_2 \)-regularized loss function \( \hat{L}(w) \), given by  
\begin{equation}\label{eqn:RegularizedObjective}
    \hat{L}(w) = \mathbb{E}\bigl[(\sigma(w^{\top} x) - \sigma(v^{\top} x))^2\bigr] + \frac{\lambda}{2} \|w\|^2,
\end{equation}
where \( \lambda \geq 0 \) is the regularization parameter. Setting \( \lambda = 0 \) recovers the original loss function \( L(w) \). The additional \( L_2 \) regularization term contributes an additive term \( \lambda w \) to the gradient, yielding  
\begin{equation*}
    \nabla \hat{L}(w) 
    = \nabla L(w) + \lambda w
    = \mathbb{E} \bigl[
        2(\sigma(w^{\top} x) - \sigma(v^{\top} x)) x \mathbbm{1}_{\{w^{\top} x \geq 0\}}
    \bigr] + \lambda w.
\end{equation*}

% We introduce the gradient descent update for minimizing a general loss function \( \mathcal{L}(w) \), which may represent either the standard loss \( L(w) \) or the \( L_2 \)-regularized loss \( \hat{L}(w) \). 

For an iterate $w_t$, step size $\alpha$, and a general loss function \( \mathcal{L}(w) \), which may represent either the standard loss \( L(w) \) or the \( L_2 \)-regularized loss \( \hat{L}(w) \), the gradient descent update is defined as 
\begin{equation*}
    w_{t+1} = w_t - \alpha \nabla \mathcal{L}(w_t).
\end{equation*}

Finally, we define the \textit{average regret} in learning the target neuron's weights as  
\begin{equation}\label{eqn:regretdefn}
    R_T = \frac{1}{T} \sum_{t=0}^{T-1} \left(L(w_t) - L(w^*)\right) = \frac{1}{T} \sum_{t=0}^{T-1} L(w_t),
\end{equation}  
where the second equality follows from the fact that the optimal weight vector for (\ref{eqn:Objective}) is \( w^* = v \), and thus  
\begin{equation*}
    L(w^*) = L(v) = 0.
\end{equation*} We also define the \textit{expected average regret}, \( \mathbb{E}[R_T] \), which accounts for randomness in the learning process. Specifically, the expectation is taken over any stochasticity introduced by the learning algorithm, including the initialization of \( w_0 \) and any reinitializations of \( w_t \) due to the reset oracle. Thus, the goal of learning a ReLU-activated neuron is to minimize (expected) average regret.

\subsubsection{Non-Adversarial Target-Weights}

Prior work on gradient‐descent dynamics has almost exclusively considered the regime in which the target weights $v$ are fixed before the network weights $w_0$ are drawn from a prescribed initialization distribution \cite{vardi2021learning, yehudai2020learning, kim2024max}. For clarity, Algorithm \ref{alg:GradientDescentNonAdversarial} formalizes gradient descent under this non-adversarial setup.

\begin{algorithm}[H]
\caption{Gradient Descent for Minimizing \(\mathcal{L}(w)\) with Non-Adversarial Target Weights}
\label{alg:GradientDescentNonAdversarial}
\SetKwInOut{Init}{Weight Initialization}
\SetKwInOut{Target}{Target Selection}
\Target{Target weights \(v\) are fixed}
\Init{Initialize \(w_0\) from some distribution}
\SetAlgoLined
\For{$t\leftarrow 0,1,2,\dots$}{
  Compute gradient \(\nabla \mathcal{L}(w_t)\)\;
  Update weights: \(w_{t+1}\leftarrow w_t - \alpha\,\nabla \mathcal{L}(w_t)\)\;
}
\end{algorithm}

\citet{vardi2021learning} show that gradient descent (Algorithm \ref{alg:GradientDescentNonAdversarial}) under mild distributional conditions which capture our setting of $x\sim N(0,I_d)\times\{1\}$, converges \emph{linearly} to the unique global minimizer \(w^*=v\):

\begin{theorem}[High-Probability Convergence, Corollary 5.5 of \citet{vardi2021learning}]\label{thm:VardiConvergenceRandom}
Under the boundedness (Assumption 5.1) and spread (Assumption 5.3) conditions on the first \(d\) coordinates, initialize
\[
  w_0 = (\tilde w_0,\,0),\qquad
  \tilde w_0\;\sim\;\mathrm{Uniform}\bigl(\mathbb S^{d-1}(\rho)\bigr),
  \quad\rho=\frac{M}{c^2}.
\]
Then for any step‐size \(\alpha\le \gamma/c^4\), there exist constants \(M,c,\gamma>0\) (depending on the distribution) such that with probability at least \(1-e^{-\Omega(d)}\),
\[
  \|w_t - v\|^2
  \;\le\;
  \bigl(1 - \gamma\,\alpha\bigr)^t\,\|w_0 - v\|^2.
\]
\end{theorem}

The high-probability, linear-rate convergence established in Theorem \ref{thm:VardiConvergenceRandom} implies an $O(1/T)$ expected average regret when training a single ReLU neuron with non-adversarial targets. In contrast, we show below that if the target weights can be chosen after the network is initialized, that is adversarially, the average regret no longer vanishes, even under very mild conditions.

\subsubsection{Adversarial Target-Weights}\label{sec:AdversarialTargetWeights}

 In continual learning, the weight initialization for a task can be those learned for a preceding task and can yield a poor initialization or local minimum. We formalize this notion by introducing the setting of adversarially selected target weights $v$. We formally define ($L_2$-regularized) gradient descent under adversarial selection of target weights below in Algorithm \ref{alg:GradientDescentAdversarial}.

\begin{algorithm}[H]
\caption{Gradient Descent for Minimizing \(\mathcal{L}(w)\) with Adversarial Target Weights}
\label{alg:GradientDescentAdversarial}
\SetKwInOut{Init}{Weight Initialization}
\SetKwInOut{Target}{Target Selection}
\Init{Initialize \(w_0\) from some distribution}
\Target{Target weights \(v\) chosen adversarially with knowledge of \(w_0\)}
\SetAlgoLined
\For{$t\leftarrow 0,1,2,\dots$}{
  Compute gradient \(\nabla \mathcal{L}(w_t)\)\;
  Update weights: \(w_{t+1}\leftarrow w_t - \alpha\,\nabla \mathcal{L}(w_t)\)\;
}
\end{algorithm}

Below in Theorem, \ref{cor:NegativeResultFullIntro} we show that ($L_2$-regularized) gradient descent can attain $\Omega(1)$ average regret under this adversarial model.

\begin{theorem}\label{cor:NegativeResultFullIntro}
    Let \(L(\cdot)\) be the objective function for learning a single neuron (see Eq.~\eqref{eqn:Objective}), and let \(w_0 \in \mathbb{R}^{d+1}\) be an initialization satisfying \((w_0)_{d+1} \le 0\). Define the target weight vector \(v \in \mathbb{R}^{d+1}\) so that \(v_{1:d} = -(w_0)_{1:d}\) and \(v_{d+1} < -(w_0)_{d+1}\). Then, for any step size \(\alpha < \frac{1}{2(d+1)}\) and any regularization parameter \(\lambda < 2d\), performing gradient descent on the L2-regularized objective \(\hat{L}(\cdot)\) from \(w_0\), i.e. Algorithm \ref{alg:GradientDescentAdversarial}, satisfies
    \[
        \forall T \ge 0, \quad R_T \;\ge\; L(0) \;>\; 0.
    \]
\end{theorem}

We interpret this result as $w_0$ being the optimal solution for a preceding task and thereby the starting point for the current task at time $t=0$. It is a priori unclear that there is an efficient way to detect when network weights $w_t$ are in a poor local minimum, and moreover, how to remedy it. Fortunately, weight-resets in conjunction with a reset-oracle solve this issue. 

\subsubsection{Adversarial Target Weights with a Reset Oracle}\label{sec:ResetOracle}

We define the \textit{reset oracle} \( \mathcal{R}_{\delta}(w) \) as  
\begin{equation*}
    \mathcal{R}_{\delta}(w) = 
    \begin{cases}
        1, & \text{if } \mathbb{E}[\sigma(w^{\top} x)] \leq \delta, \\
        0, & \text{if } \mathbb{E}[\sigma(w^{\top} x)] > \delta.
    \end{cases}
\end{equation*}
The reset oracle \( \mathcal{R}_{\delta} \) is parameterized by a threshold \( \delta \) and returns 1 (True) if and only if the expected activation of the neuron, defined by the weight vector \( w \), falls below \( \delta \). This reset oracle is an idealization of the reset-criteria of reset-algorithms like CBP, ReDO, and SNR.

Next, we introduce the gradient descent algorithm with the reset oracle \( \mathcal{R}_{\delta} \) for minimizing the loss function \( L(w) \). The reset oracle monitors the neuron's expected activation and determines whether a reset is necessary. If the reset condition is met, the weights are reinitialized from the original distribution; otherwise, the standard gradient descent update is applied to minimize \( L(w) \).

\begin{algorithm}[H]
\caption{Gradient Descent with Reset Oracle for Minimizing \(L(w)\)}
\label{alg:GradientDescentWithResetOracle}
\SetKwInOut{Init}{Weight Initialization}
\SetKwInOut{Target}{Target Selection}
\Init{Initialize $w_0$ from some distribution}
\Target{Target weights \(v\) chosen adversarially with knowledge of \(w_0\)}
\SetAlgoLined
\For{$t\leftarrow 0,1,2,\dots$}{
  \eIf{$\mathcal{R}_{\delta}(w_t) = 0$}{
    Sample $w_{t+1}$ from the initial distribution\;
  }{
    Compute gradient $\nabla L(w_t)$\;
    Update weights: $w_{t+1}\leftarrow w_t - \alpha\,\nabla L(w_t)$\;
  }
}
\end{algorithm}

% \begin{algorithm}[H]
%     \caption{Gradient Descent with Reset Oracle for Minimizing \( L(w) \)}
%     \label{alg:GradientDescentWithResetOracle}
%     \begin{algorithmic}[1]
%         \State \textbf{Initialize:} \( w_0 \) from an initial distribution.
%         \State \textbf{Adversarial Selection:} Target weights \( v \) may be chosen adversarially with knowledge of \( w_0 \).

%         \For{\( t = 0,1,2,\dots \)}
%             \If{\( \mathcal{R}_{\delta}(w_t) = 0 \)}
%                 \State Sample \( w_{t+1} \) from the initial distribution.
%             \Else
%                 \State Compute gradient: \( \nabla L(w_t) \)
%                 \State Update weights: \( w_{t+1} = w_t - \alpha \nabla L(w_t) \)
%             \EndIf
%         \EndFor
%     \end{algorithmic}
% \end{algorithm}

Below, we present our main theoretical contributions. When the target weights \( v \) are selected adversarially with knowledge of the network initialization \( w_0 \), Theorem \ref{thm:PositiveRegretGuaranteeIntro} establishes that gradient descent with the reset oracle (Algorithm \ref{alg:GradientDescentWithResetOracle}) achieves vanishing expected average regret.

\begin{theorem}\label{thm:PositiveRegretGuaranteeIntro}
Let $w_0$ be sampled from \(\mathcal{D}_{w_0} = \text{Unif}(l\cdot S^{d-1}\times\{0\})\) for some positive constant $l>0$, where $l\cdot S^{d-1}\subseteq \mathbb{R}^d$ is the $l$-radius sphere. Let \(v\in\mathbb{R}^{d+1}\) (possibly chosen adversarially with knowledge of \(w_0\)) be the target neuron's weights satisfying
\[
-v_{d+1} \le \frac{\|v_{1:d}\|}{2\sqrt{2\pi}},\quad \|v_{1:d}\|=2\pi l,\quad \|v\|\ge C_1,\quad \|v_{1:d}\|\ge C_2\|v\|,
\]
for some constants \(C_1,C_2>0\), and suppose
\[
\delta<\left(\frac{C_2^2}{8\pi^2\Bigl(\sqrt{3}+\frac{2}{C_1}\Bigr)}\right)^3.
\]
Then, minimizing \(L(\cdot)\) via gradient descent with step size \(\alpha\le \frac{1}{2(d+1)}\), initialized at \(w_0\), and employing a reset oracle \(\mathcal{R}_\delta\), i.e. Algorithm \ref{alg:GradientDescentWithResetOracle}, yields
\[
\forall\, T\ge0,\quad \mathbb{E}[R_T]\le  \mathcal{O}\!\Bigl(\frac{d^2\ln T}{T}\Bigr),
\]
where the expectation is taken over the randomness of the the initialization of $w_0$ and the resets.
\end{theorem}

\textbf{Remarks.} For simplicity of exposition, Theorem~\ref{thm:PositiveRegretGuaranteeIntro} assumes that \( w_{1:d} \) is sampled uniformly from the sphere of radius \( l = \frac{\|v_{1:d}\|}{2\pi} \). Alternatively, one may sample \( w_{1:d} \) from \( \mathcal{N}(0,\sigma^2 I_d) \) with \( \sigma^2 = \frac{\|v_{1:d}\|^2}{4\pi^2 d} \), in which case standard concentration of measure arguments imply that  
\[
\mathbb{P}\!\Bigl(\|w\|^2 \in \Bigl[(1-\epsilon)\frac{\|v_{1:d}\|^2}{4\pi^2}, (1+\epsilon)\frac{\|v_{1:d}\|^2}{4\pi^2}\Bigr]\Bigr) \geq 1 - e^{-\Omega(d\epsilon^2)}.
\]
Additionally, the assumption \( -v_{d+1} \leq \frac{\|v_{1:d}\|}{2\sqrt{2\pi}} \) ensures that the bias term \( v_{d+1} \) is not arbitrarily negative relative to \( \|v_{1:d}\| \). The assumptions of Theorem~\ref{thm:PositiveRegretGuaranteeIntro}, including the initialization distribution of \( w_0 \), the latter bound on \( v_{d+1} \), and the constraints on \( \|v\| \) and \( \|v_{1:d}\| \), are identical, up to the choice of constants, to those used in the gradient descent analysis of \cite{vardi2021learning}, the most state-of-the-art work on learning a single ReLU-activated neuron in the non-adversarial setting. While the assumptions of Theorem~\ref{thm:PositiveRegretGuaranteeIntro} align with those in prior work, our analysis of gradient descent with resets takes an independent approach. 

% Below, we highlight the key steps and technical challenges involved in our proof (TODO).

\subsection{Motivating SNR and Comparison to Other Reset Schemes}
\label{sec:hyp_test}

Here we motivate the SNR heuristic and compare it to other proposed reset schemes. Consider the following simple hypothesis test: we observe a discrete time process $Z_s \in \{0,1\}$ which under the null hypothesis $H_0$ is a Bernoulli process with mean $p >0$. The alternative hypothesis $H_1$ is that the mean of the process is identically zero. A hypothesis test must, at some stopping time $\tau$, either reject ($X_\tau = 1$) or accept ($X_\tau = 0$) the null; an optimal such test would choose to minimize the sum of type-1 and type-2 errors (the `error rate') and a penalty for delays: 
\[
\mathbb{P}(X_\tau=1|H_0) + \mathbb{P}(X_\tau=0|H_1)
+ \lambda 
\left(
\E[\tau|H_0] + \E[\tau|H_1] 
\right)
\]
Here the multiplier $\lambda > 0$ penalizes the delay in a decision. If $\lambda < p/2$, the optimal test takes a simple form: for some suitable threshold $\bar T$, reject the null iff $Z_s = 0$ for all times $s$ up to $\bar T$:  
\begin{proposition}\label{prop:prop1}
Let $\bar T$ be the $1- \lambda(p-\lambda)^{-1}$ percentile of a {\rm Geometric}$(p)$ distribution. Then the optimal hypothesis test takes the form $X_\tau = \mathbf{1}\{Z_\tau = 0\}$ where $\tau = {\rm min}(s:Z_s = 1) \wedge \bar T$.   
\end{proposition}

Notice that if $\lambda \propto p$, the percentile threshold above is independent of $p$. Applying this setup to the setting where under the null, we observe the firing of neuron $i$ under i.i.d. training examples from $\mu_t$ and a neuron is considered `dead' or inactive if the alternate hypothesis is true, imagine that $p = \mathbb{P}(Z^{\mu_t}_{s,i}=1)$. Further, we assume $\lambda = \alpha p$ ($\alpha < 1/2$); a reasonable assumption which models a larger penalty for late detection of neurons that are highly active. 
%leading to a percentile threshold that is independent of $p$ itself. 
It is then optimal to declare neuron $i$ `inactive'  if the length of time it has not fired exceed the $1 - \alpha(1-\alpha)^{-1}$ percentile of the distribution of $A_i^{\mu_t}$. This is the underlying motivation for the SNR heuristic. 

\textbf{Comparison with Reset Schemes:} Neuron reset heuristics such as \citet{sokar2023dormant} define (sometimes complex) notions of neuron `utility' to determine whether or not to re-initialize a neuron. The utility of every neuron is computed over every consecutive (say) $r$ minibatches, and neurons with utility below a threshold are reset. To facilitate a comparison, consider the setting where neurons that do not fire at all over the course of the $r$ mini batches are estimated to have zero utility, and that only neurons with zero utility are re-initialized. 

This reveals an interesting comparison with SNR. The schemes above will re-initialize a neuron after inactivity over a period of time that is  {\em uniform} across all neurons. On the other hand, SNR will reset a neuron after it is inactive for a period that corresponds to a fixed percentile of the inter-firing time distribution of that neuron. Whereas this percentile is fixed across neurons, the corresponding period of inactivity after which a neuron is reset will vary across neurons: shorter for neurons that tend to fire frequently, and longer for neurons that fire less frequently. 

We can make this comparison precise in the context of the hypothesis testing setup above: specifically, consider two neurons with null firing rates $p_1$ and $p_2$ respectively ($p_1 < p_2$), and delay multipliers, $\lambda$, of $\alpha p_1$ and $\alpha p_2$ respectively. By Proposition~\ref{prop:prop1}, under SNR, the first is reset if it is inactive for time at least $\log(\alpha(1-\alpha)^{-1})/\log (1-p_1)$ and the second if it is inactive for time at least $\log(\alpha(1-\alpha)^{-1})/\log (1-p_2)$. In contrast, for a fixed threshold scheme such as \citet{sokar2023dormant}, either neuron would be reset after being inactive for some fixed threshold, say $r^*$. Assume $r^*$ is set to minimize the sum of the error rates of the two neurons while keeping the total delay identical to that for SNR. The proposition below compares the error rate between the two schemes:

\begin{proposition}\label{prop:prop2}
The ratio of total error rate with a fixed threshold $r^*$ to that under SNR scales like
\[
\Omega\left(
\exp
\left(
\log(\alpha(1-\alpha)^{-1})
\left( 
-\frac{1}{2} + \frac{1}{2} \frac{\log (1-p_1)}{\log(1-p_2)}
\right)
\right)
\right)
\]
\end{proposition}

Now recall that $\alpha < 1/2$ and $p_1 < p_2$. The result above then shows that: (a) the error rate under an (optimal) fixed threshold can grow arbitrarily larger than the error rate under SNR as the penalty for delay shrinks to zero and (b) the rate at which this gap grows itself scales with the difference in the nominal firing rates of the neurons under consideration. This provides insight into the relative merits of using a scheme like SNR in lieu of existing reset proposals: {\em resets under SNR detect changes in the firing rate of a neuron faster and more accurately; this matters particularly in situations where there is wide disparity in the nominal firing rates of neurons across the network.}

\section{Experiments}\label{sec:experiments}

We evaluate the efficacy and robustness of SNR on a series of benchmark problems from the continual learning literature, measuring regret with respect to prediction accuracy $l(y,y') = \mathbf{1}\{y\neq y'\}$. As an overview, we will seek to make the following points:

\noindent \textbf{Inactive neurons are an important correlate of plasticity loss:} This is true across several architectures: vanilla MLPs, CNNs and transformers. 
\newline
\noindent \textbf{Lower average loss:} Across a broad set of problems/ architectures from the literature, SNR consistently achieves lower average loss than competing algorithms.
\newline
\noindent \textbf{No consistent second-best competitor:} Among competing algorithms, none emerge as consistently second best to SNR.
\newline
\noindent \textbf{Robustness to hyper-parameters:} The performance of SNR is robust to the choice of its single hyper parameter (the rejection percentile threshold). This is less so for competing algorithms.  

\subsection{Experimental Setup}

Each problem consists of tasks $\mathcal{T}_1, \mathcal{T}_2,\ldots, \mathcal{T}_N$, each of which contains training examples in $\mathcal{X}\times\mathcal{Y}$. A network is trained for a fixed number of epochs per task to minimize cross-entropy loss. We perform an initial hyperparameter sweep over 5 seeds to determine the optimal choice of hyperparameters (see Appendix \ref{sec:hyperparam_sweep}). For each algorithm and problem, we select the hyperparameters that attain the lowest average loss and repeat the experiment on 5 new random seeds. A random seed determines the network's parameter initialization, the generation of tasks, and any randomness in the algorithms evaluated. We evaluate both SGD and Adam as the base optimization algorithm, as earlier literature has argued that Adam can be less performant than SGD in some continual learning settings \citep{dohare2023loss, ashley2021does}. We evaluate on the following problems: 
%We evaluate performance on the benchmark problems Permuted MNIST, Random Label MNIST, Random Label CIFAR, and Continual ImageNet, which have been utilized in \citep{dohare2023loss, kumar2023maintaining}. To evaluate the merits of SNR for a transformer based language model, we introduce a new problem called Permuted Shakespeare. We outline the details of each experiment below. 

    \noindent \textbf{Permuted MNIST (PM) \citep{goodfellow2013empirical,dohare2021continual,kumar2023maintaining}:}  A subset of 10000 image-label pairs from the MNIST dataset are sampled for an experiment. A task consists of a random permutation applied to each of the 10000 images. The network is presented with 2400 tasks appearing in consecutive order. Each task consists of a single epoch and the network receives data in batches of size 16.
    \newline
      \noindent \textbf{Random Label MNIST (RM) \citep{kumar2023maintaining,lyle2023understanding}:} A subset of 1200 images from the MNIST dataset are sampled for an experiment. An experiment consists of 100 tasks, where each tasks is a random assignment of labels, consisting of 10 classes, to the 1200 images. A network is trained for 400 epochs on each task with a batch size 16. 
      \newline
     \noindent \textbf{Random Label CIFAR (RC) \citep{kumar2023maintaining,lyle2023understanding}:} A subset of 128 images from the CIFAR-10 dataset are sampled for an experiment. An experiment consists of 50 tasks, where each tasks is a random assignment of labels, consisting of 10 classes, to the 128 images. An agent is trained for 400 epochs on each task with a batch size 16.
     \newline
     \noindent \textbf{Continual Imagenet (CI) \citep{dohare2023loss,kumar2023maintaining}:} An experiment consists of all 1000 classes of images from the ImageNet-32 dataset \citep{chrabaszcz2017downsampled} containing 600 images from each class. Each task is a binary classification problem between two of the 1000 classes, selected at random. The experiment consists of 500 tasks and each class occurs in exactly one task. Each task consists of 1200 images, 600 from each class, and the network is trained for 10 epochs with a batch size of 100. 
     \newline
    \noindent \textbf{Permuted Shakespeare (PS):} We propose this problem to facilitate studying the transformer architecture in analogy to the MNIST experiments. An experiment consists of 32768 tokens of text from Shakespeare's Tempest. For any task, we take a random permutation of the vocabulary of the Tempest and apply it to the text. The network is presented with 500 tasks. Each task consists of 100 epochs and the network receives data in batches of size 8 with a context widow of width 128. We evaluate over 9 seeds.

This experimental setup, for all but Permuted Shakespeare, follows that of \citep{kumar2023maintaining}, with the exceptions of Permuted MNIST which has its task count increased from 500 to 2400, Random Label MNIST which has its task count increased from 50 to 100, and Random Label CIFAR which has its dataset reduced from 1200 to 128 images. \citet{lyle2023understanding} consider variants of the Random Label MNIST and CIFAR problems by framing them as MDP environments for DQN agents. During training, the DQN agents are periodically paused to assess their plasticity by training them on separate, randomly generated regression tasks using the same image datasets.

\subsubsection{Algorithms and Architetcures}

Our baseline in all problems consist simply of using SGD or Adam as the optimizer with no further intervention. We then consider several interventions to mitigate plasticity loss. First, we consider algorithms that employ an explicit reset of neurons: these include SNR, Continual Backprop (CBP) \citep{dohare2021continual}, and ReDO \citep{sokar2023dormant}. Among algorithms that attempt to use regularization, we consider vanilla L2 regularization, L2 Init \citep{kumar2023maintaining}, and Shrink and Perturb \citep{ash2020warm}. Finally, as a potential architectural modification we consider the use of Layer Normalization \citep{ba2016layer}. 

We utilize the following network architectures:  
\newline
\noindent \textbf{MLP}: For Permuted MNIST and Random Label MNIST we use an MLP identical to that in \citet{kumar2023maintaining} which in turn is a slight modification to that in \citet{dohare2023loss}. 
\newline
\textbf{CNN}: For Random Label CIFAR and Continual ImageNet we use a CNN architectures identical to that in \citet{kumar2023maintaining} which in turn is a slight modification to that in \citet{dohare2023loss}. 
\newline
\textbf{Transformer:} We use a decoder model with a single layer consisting of 2 heads, dimension 16 for each head, and with 256 neurons in the feed forward layer with ReLU activations. We deploy this architecture on the Permuted Shakespeare problem using the GPT-2 BPE tokenizer (limited to the set of unique tokens present in the sampled text).
\newline
\textbf{ViT}: In addition, we benchmark a Vision Transformer (ViT) on the Continual ImageNet problem. Our ViT follows the specifications of \citet{leeslow} and \citet{lewandowski2024learning}, using 4×4 input patches, three attention heads, and 12 transformer layers with an embedding dimension $d_{\text{model}} = 192$. This architecture comprises approximately 5.3 million non-embedding parameters. We refer to this variant as CI-ViT to distinguish it from the CNN-based experiments.

\subsection{Results and Discussion}

We separately discuss the results for the first four problems (PM, RM, RC, CI) followed by Permuted Shakespeare and CI-ViT; we observe additional phenomena in the latter experiments which merit separate discussion. 

\begin{table}[htbp]
\centering
\begin{tabular}{lcccc|cccc}
\hline
\textbf{Optimizer} & \multicolumn{4}{c|}{\textbf{SGD}} & \multicolumn{4}{c}{\textbf{Adam}} \\
\hline
\textbf{Algorithm} & \textbf{PM} & \textbf{RM} & \textbf{RC} & \textbf{CI} & \textbf{PM} & \textbf{RM} & \textbf{RC} & \textbf{CI} \\
\hline
No Intv.     & $0.71$  & $0.11$  &  0.18 & $0.78$ & $0.64$    & $0.11$  &  0.15 & $0.58$ \\
SNR          & $\textbf{0.85}$    & $\textbf{0.97}$  &  \textbf{0.99}& $\textbf{0.89}$ & $\textbf{0.88}$    & $\textbf{0.98}$  &  \textbf{0.98} & $\textbf{0.85}$ \\
CBP          & $0.84$    & $0.95$  &  0.96 & $0.84$ & $\textbf{0.88}$    & $0.95$  &  0.33 & $0.82$ \\
ReDO         & $0.83$    & $0.72$  & 0.98 & $0.87$ & $0.85$    & $0.67$  & 0.74 & $0.80$ \\
L2 Reg.      & $0.82$    & $0.80$  & 0.95 & $0.83$ & $\textbf{0.88}$    & $0.95$  & 0.97 & $0.80$ \\
L2 Init      & $0.83$   & $0.91$  & 0.97 & $0.83$ & $\textbf{0.88}$    & $0.96$  & \textbf{0.98} & $0.83$ \\
S\&P         & $0.83$ & 0.92 & 0.97 & $0.85$ & $\textbf{0.88}$ & 0.96 & 0.97 & $0.81$ \\
Layer Norm.  & $0.69$    & $0.14$  & 0.96 & $0.82$ & $0.66$    & $0.11$  &  0.96 & $0.58$ \\
\hline
\end{tabular}
% \caption{Average accuracy on the last 10\% of tasks for each algorithm on the benchmark continual learning problems Permuted MNIST (PM), Random Label MNIST (RM), Random Label CIFAR (RC), and Continual ImageNet (CI). Each algorithm is trained for 5 seeds, with both Adam and SGD, with its best performing hyperparameters after an initial hyperparameter sweep over 5 seeds.}
\caption{Average accuracy on the last 10\% of tasks on the benchmark continual learning problems over 5 seeds. Standard deviations are provided in the extended Table~\ref{tab:average_accuracy_complete}.}
% , with both Adam and SGD, with its best performing hyperparameters after an initial hyperparameter sweep over 5 seeds.}
\label{tab:average_accuracy}
\end{table}

Table \ref{tab:average_accuracy} shows that across all four problems (PM, RM, RC, CI) and both SGD and Adam, SNR consistently attains the largest average accuracy on the final 10\% of tasks.  For each competitor algorithm there is at least one problem on which SNR attains superior accuracy by at least 5 percentage points with SGD and at least 2 percentage points with Adam. We also see that there is no consistent second-best algorithm with the SGD optimizer while L2 Init is consistently the second-best with the Adam optimizer. 

Another notable property of SNR is its robustness to the choice of rejection percentile threshold. In contrast, its competitors are not robust to the choice of their hyperparemter(s). On all but RM with SGD, SNR experiences a decrease of at most 2 percentage points in average accuracy when varying its rejection percentile threshold across the range of optimal thresholds found across experiments. On the other hand, increasing the hyperparameter strength by a single order of magnitude, as is common in a hyperparameter sweep,  from the optimal value for L2 Init, S\&P, and CBP results in a decrease in average accuracy by at least 72 percentage points. See Appendix \ref{sec:hyperparam_sweep} for detailed tables. 

Next, we turn our attention to the Permuted Shakespeare problem. For the no intervention network with Adam, we see dramatic plasticity loss, as the average loss increases from about 0.15 on the first few tasks to 3.0242 on the last 50 tasks; see Figure \ref{fig:PS_loss}. In Figure \ref{fig:PS_weight_norms} and Figure \ref{fig:PS_neuron_death} we see that this plasticity loss is correlated with increasing weight norms in the self-attention and feedforward layers, persistent neuron inactivity, and a collapse in the entropy of the self-attention probabilities. 

\begin{table}[htbp]
\centering
\begin{tabular}{lccc}
\hline
\textbf{Algorithm} & \textbf{All Tasks} & \textbf{First 50 Tasks} & \textbf{Last 50 Tasks} \\
\hline
L2         & $0.2762 \hspace{0.5em} (0.0309)$    & $0.1560 \hspace{0.5em} (0.0236)$    & $0.3101 \hspace{0.5em}  (0.0425)$ \\
SNR+L2     & $0.2177 \hspace{0.5em}  (0.0196)$    & $0.1370 \hspace{0.5em}  (0.0169)$    & $0.2551 \hspace{0.5em}  (0.0454)$ \\
No Intv.   & $2.7397 \hspace{0.5em} (0.0140)$    & $1.8164 \hspace{0.5em}  (0.0295)$    & $3.0147 \hspace{0.5em}  (0.0250)$ \\
L2 Init    & $1.5052 \hspace{0.5em}  (0.0437)$    & $1.2931 \hspace{0.5em}  (0.0486)$    & $1.5262 \hspace{0.5em} (0.0420)$ \\
SNR        & $2.6872 \hspace{0.5em}  (0.0222)$    & $1.7338 \hspace{0.5em}  (0.0295)$    & $3.0242 \hspace{0.5em}  (0.0408)$ \\
CBP &       2.4922 \hspace{0.5em}  (0.0171) & 1.3732 \hspace{0.5em}   (0.0234) & 2.9410 \hspace{0.5em} (0.0188) \\
L2* & 0.1506 \hspace{0.5em}  (0.0279) & 0.1874 \hspace{0.5em}  (0.0348) & 0.1092 \hspace{0.5em}  (0.0429) \\
SNR+L2* & \textbf{0.1402} \hspace{0.5em}  (0.0246) & \textbf{0.1549} \hspace{0.5em}  (0.0107) & \hspace{0.5em}  \textbf{0.0909} \hspace{0.5em} (0.0177) \\
ReDO & 2.3258 \hspace{0.5em} (0.0356) &  1.3689 \hspace{0.5em} (0.0686) & 2.8119 \hspace{0.5em} (0.0373)\\
\hline
\end{tabular}
\caption{Average loss measured on the final epoch of each task with standard deviations over 9 seeds on the Permuted Shakespeare problem. Note, L2* denotes L2 regularization applied only to the attention weights.}
% \caption{Average loss measured on the final epoch of each task with standard deviations over 9 seeds) on the Permuted Shakespeare problem. Results are reported for the best performing hyperparameter(s) of each algorithm, selected after an initial hyperparameter sweep over 5 seeds.}
\label{tab:PS_loss}
\end{table}

\begin{figure}[htbp]
    \centering
    \includegraphics[width=0.48\linewidth]{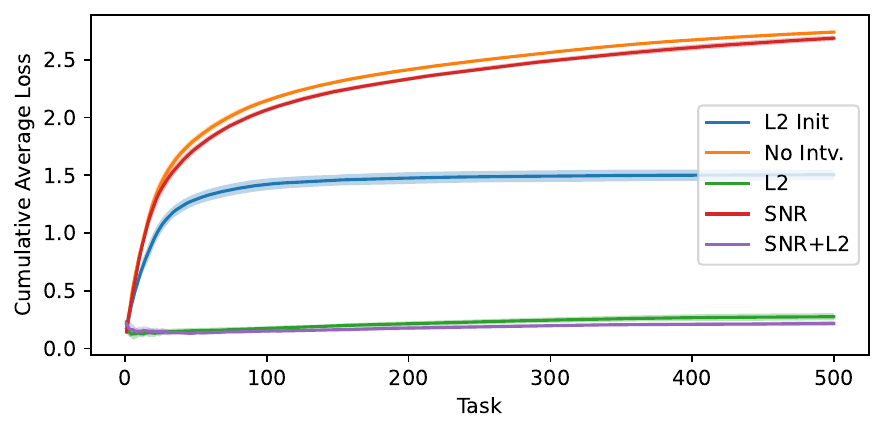}
    \includegraphics[width=0.48\linewidth]{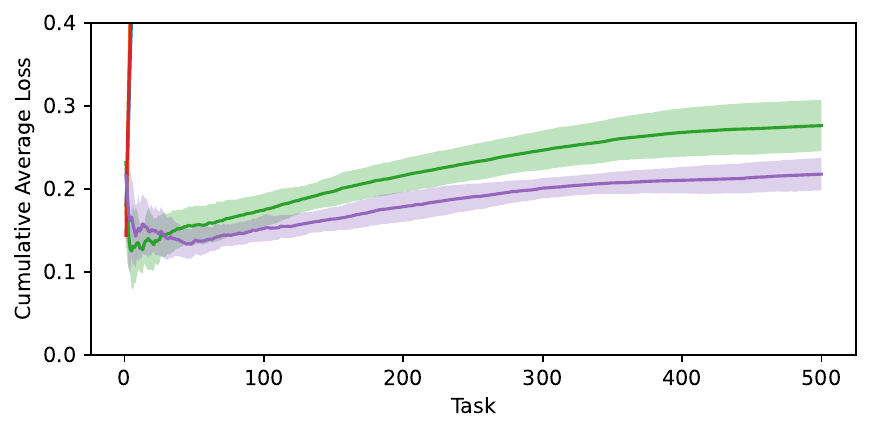}
    \caption{Cumulative average of the loss measured on the final epoch of each task on the Permuted Shakespeare problem. Right panel zooms in to L2 and SNR+L2}
    \label{fig:PS_loss}
\end{figure}

We see that resets are by themselves insufficient in mitigating plasticity loss, providing at most a marginal improvement over no intervention. This is unsurprising since neurons are only present in the feedforward layers, unlike the MLP and CNN architectures in the earlier experiments. As such, regularization appears necessary and we see that, over the last 50 tasks, L2 regularization attains an average loss of 0.3101 in contrast to 3.0147 and 3.0242 for no intervention and SNR. This improvement in performance coincides with stable weight norms and non-vanishing average entropy of self-attention probabilities for L2 regularization. In contrast to the earlier problems, L2 Init fares worse than L2 regularization and experiences substantial loss of plasticity, although to a lesser extent than the no intervention network.

While L2 regularization addresses weight blowup, neuron death remains present; see Figure~\ref{fig:PS_neuron_death}. The average loss with L2 increases from 0.1560, over the first 50 tasks, to 0.3101, over the final 50 tasks. This prompts us to consider using SNR in addition to L2 regularization. This largely eliminates neuron death (see the right panel of Figure~\ref{fig:PS_neuron_death}), while stabilizing weight norms and maintaining entropy of self-attention probabilities, providing the lowest loss (0.2551) over the final 50 tasks. As an ablation, we evaluate performance when applying L2 regularization only to the attention weights, which we denote by L2* and SNR+L2* in Table \ref{tab:PS_loss}.

\begin{figure}[htbp]
    \centering
    \includegraphics[width=0.48\linewidth]{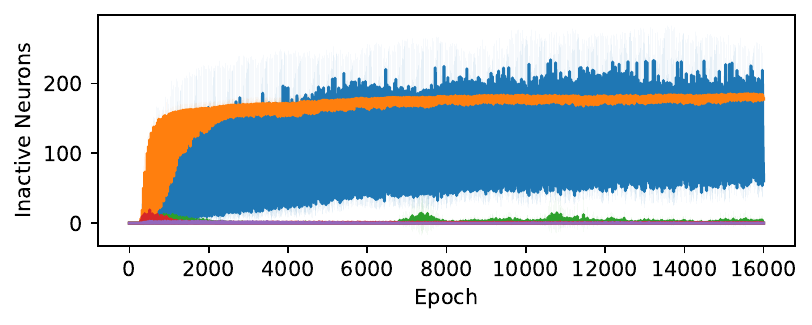}
    \includegraphics[width=0.48\linewidth]{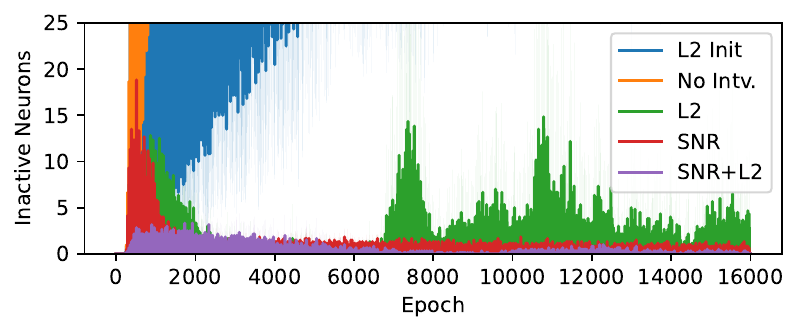}
    \caption{Number of inactive neurons for 5 consecutive training steps on the PS experiment.}
    \label{fig:PS_neuron_death}
\end{figure}

    \begin{figure}[htbp]
    \centering
    \includegraphics[width=0.24\linewidth]{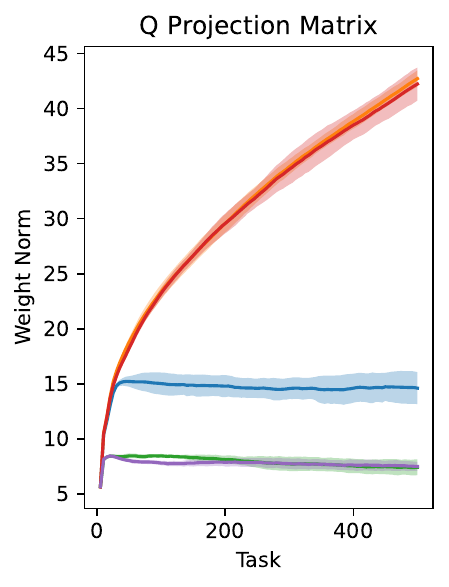}
    \includegraphics[width=0.24\linewidth]{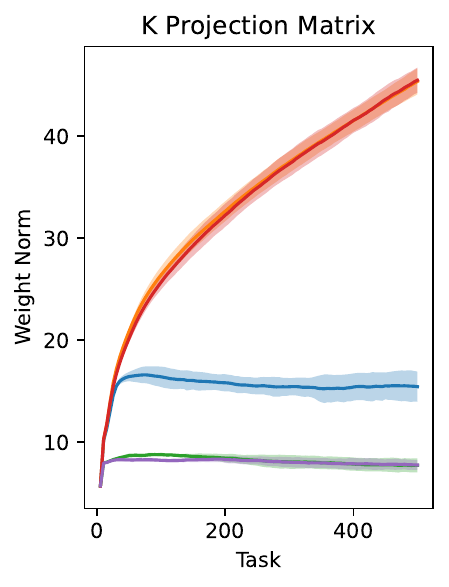}
    \includegraphics[width=0.24\linewidth]{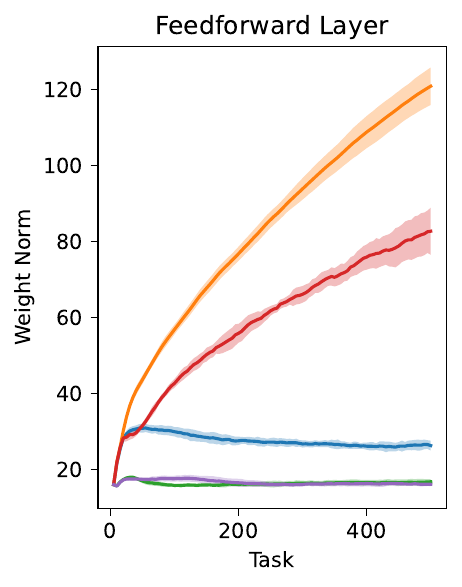}
    \includegraphics[width=0.24\linewidth]{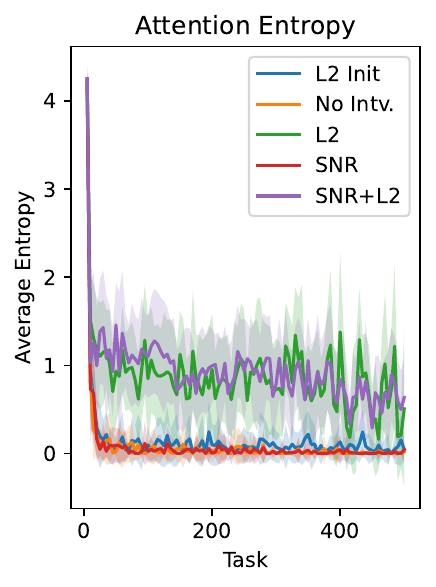}
    \caption{Weight norms and attention-score entropy on the Permuted Shakespeare experiment.}
    \label{fig:PS_weight_norms}
\end{figure}

Finally, we turn to our CI-ViT experiment, which reproduce many of the same behaviors observed in the Permuted Shakespeare benchmark, a consequence of the transformer architecture. As shown in Figures \ref{fig:ViT_acc}, \ref{fig:ViT_weight_norms}, and \ref{fig:ViT_attention_neuron_death}, plasticity loss in CI-ViT is accompanied by steadily growing weight norms in both the KQV projection matrices and the feed-forward MLP layers, by neuron inactivity, and by a collapse in self‐attention entropy. As with the PS experiment, resets alone remain insufficient to prevent plasticity loss when training a ViT on Continual ImageNet. Table \ref{tab:ViTResults} shows that CBP, ReDO, and SNR attain average training accuracies of 0.7037, 0.6380, and 0.6555, respectively, over the final epoch of the last 50 tasks, whereas the regularization-based methods L2 and L2-Init achieve 0.8140 and 0.9196. All algorithms  except No Intervention limit feedforward weight norm growth due to regularization or neuron-resets. However, No Intervention and the purely reset-based methods experience unbounded weight norm growth of the KQV projection matrices and a collapse of the entropy of the self-attention probabilities. As stated earlier, this is unsurprising as reset-based methods do not target self-attention layers. To remedy this, we consider the variants SNR+L2 and SNR+L2*. We find that SNR-L2*, like in PS, is the most performant algorithm, attaining an average training accuracy of 0.93372 over the final epoch of the last 50 tasks, while its closest competitor, L2-Init attains 0.91955. Similarly to the MLP and CNN based experiments, we observe that for the transformer-based experiments, PS and CI-ViT, there is no consistent closest competitor to SNR+L2* as L2 is its closest competitor in PS and L2-Init is its closest competitor in CI-ViT.

\begin{table}[htbp]
\label{tab:ViTResults}
\centering
\begin{tabular}{lc|}
\hline
\textbf{Algorithm} & \textbf{Training Accuracy (Std. Dev.)} \\ \hline
No Intv. & 0.64315 (0.01403) \\ \hline
L2 & 0.81399 (0.00811) \\ \hline
L2Init & 0.91955 (0.00385) \\ \hline
S\&P & 0.54589 (0.00072)\\ \hline
CBP & 0.70365 (0.03138) \\ \hline
ReDO & 0.63802 (0.01893) \\ \hline
SNR & 0.65547 (0.02421) \\ \hline
SNR+L2 & 0.87326 (0.00714) \\ \hline
SNR+L2* & \textbf{0.93372} (0.00920) \\ \hline
\end{tabular}
\caption{Average training accuracy on the last epoch of the final 50 tasks, averaged over 5 seeds with standard deviations in brackets, for the CI-ViT experiment.}
\label{tab:training_accuracies}
\end{table}

\begin{figure}[htbp]
    \centering
    \includegraphics[width=0.65\linewidth]{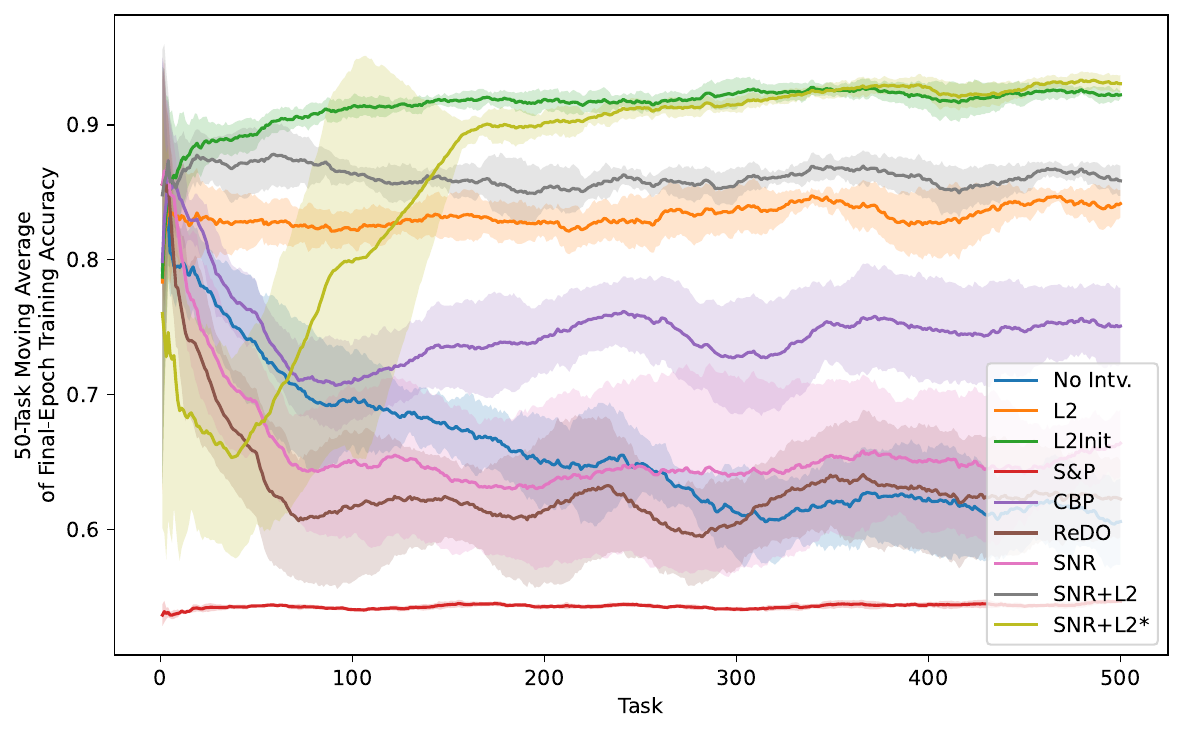}
    \caption{50-task moving average of final-epoch training accuracy of the CI-ViT experiment.}
    \label{fig:ViT_acc}
\end{figure}

\begin{figure}[htbp]
    \centering
    \includegraphics[width=0.4\linewidth]{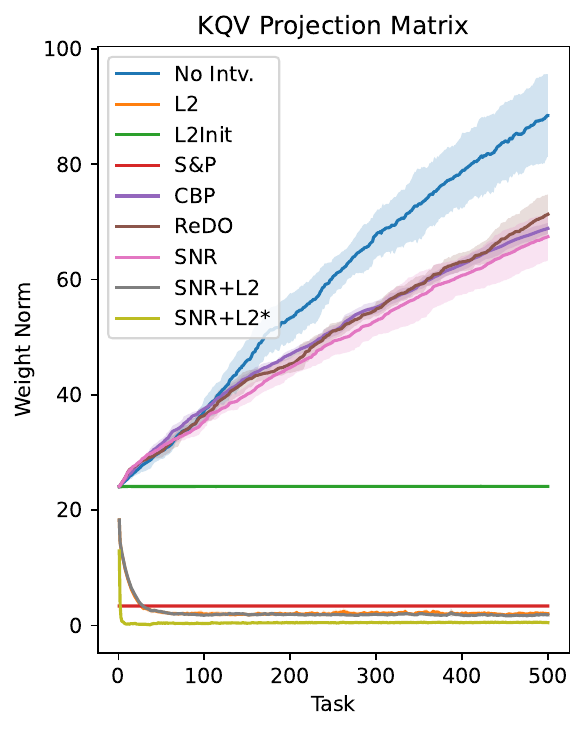}
    \includegraphics[width=0.4\linewidth]{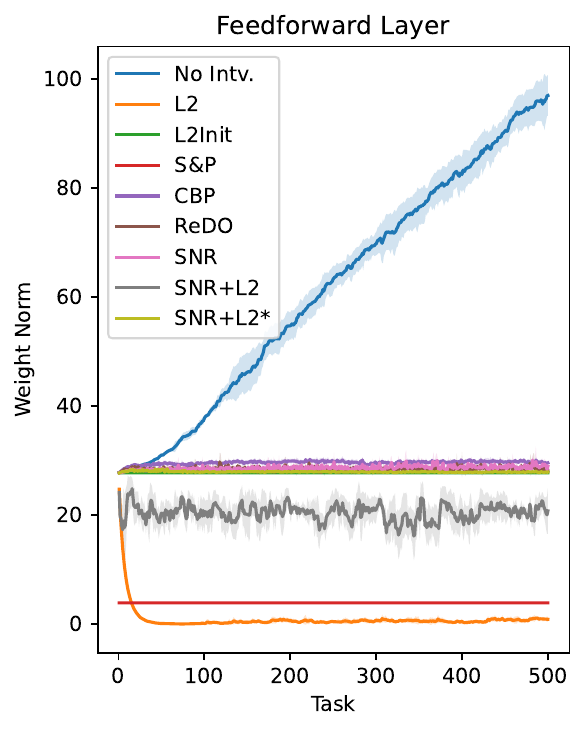}
    \caption{Weight norms on the Continual ImageNet experiment.}
    \label{fig:ViT_weight_norms}
\end{figure}

\begin{figure}[htbp]
    \centering
    \includegraphics[width=0.4\linewidth]{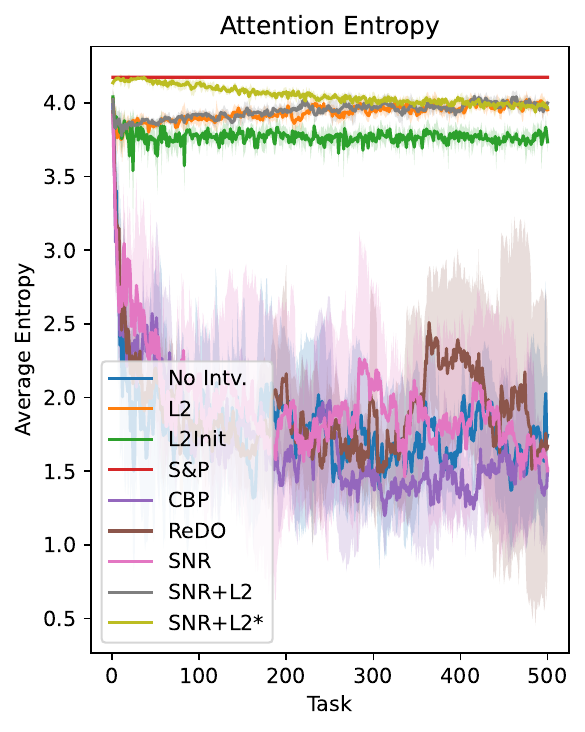}
    \includegraphics[width=0.4\linewidth]{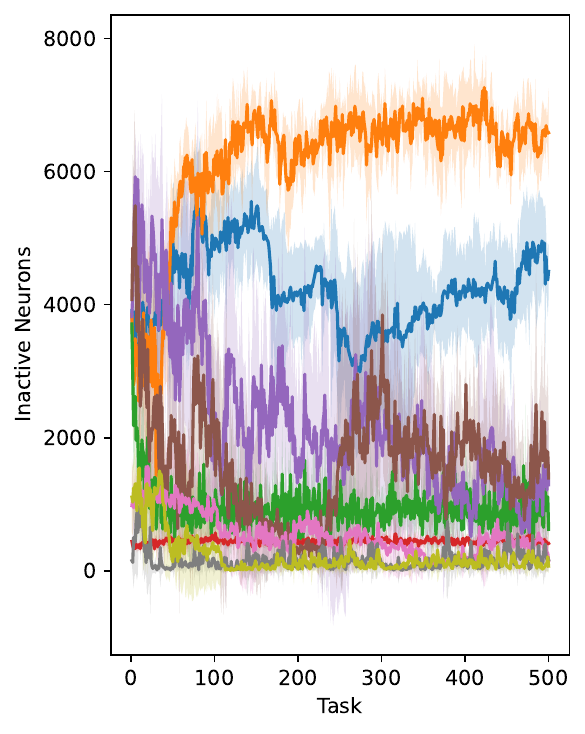}
    \caption{Attention-score entropy and number of inactive neurons at any timestep, averaged over each task, on the Continual ImageNet experiment. For reference, our ViT model has 9216 neurons in total.}
    \label{fig:ViT_attention_neuron_death}
\end{figure}

\subsubsection{Scaled Permuted Shakespeare} 

While the scale of our Permuted Shakespeare problem serves as a simple benchmark problem for evaluating a series of continual learning algorithms and hyperparameter choices for language models, it is also of interest to investigate the effect of model and dataset scale on plasticity. To this end, we scale the number of non-embedding weights in our transformer network by a factor of $N=16$, increasing the number of heads to 8 and number of neurons to 1024. In line with scaling laws \citep{kaplan2020scaling}, we increase the size of our dataset by a factor of $16^{0.74}$, specifically to 254'976 tokens per task. The rest of the problem setup remains unchanged; we train the network for 100 epochs on 500 tasks in sequence. To facilitate the larger token count, we train on a sample of 254'976 tokens worth of text from the complete set of plays by William Shakespeare. 

We limit our experiment to 4 random seeds, scales 1 and 16, evaluating only SNR+L2 and L2 with hyperparameters $\eta = 0.05$ and $\lambda = 10^{-4}$, presenting our results in Table \ref{tab:PS_scaled}. We first note that as model and dataset size grow, the gap in average loss between L2-regularization and SNR+L2-regularization grows substantially. Simultaneously, we see a dramatic increase in the proportion of inactive neurons with L2-regularization. At any time step, on average $0.06\%$ of neurons are inactive at scale $N=1$ while $32.9\%$ are inactive at scale $N=16$. These results suggest that resets can play a critical role in maintaining plasticity in large-scale language models. See Appendix A which shows a scale 256 experiment, but run only up to 50 tasks, showing similar merits.

\begin{table}[htbp]
\centering
\begin{tabular}{l|c|c|c}
\hline
\textbf{Algorithm} & \textbf{Loss (All Tasks)} & \textbf{Loss (Last 50 Tasks)} & \textbf{Dead Neuron Rate} \\
\hline
L2 - Scale 1 & 0.159 \hspace{0.5em} (0.025) & 0.152 \hspace{0.5em} (0.016) & 0.06\% \\
SNR+L2 - Scale 1 & 0.154 \hspace{0.5em} (0.008) & 0.141 \hspace{0.5em} (0.031) & 0.00\% \\
\hline
L2 - Scale 16 & 0.377 \hspace{0.5em} (0.016) & 0.410 \hspace{0.5em} (0.052)& 32.9\% \\
SNR+L2 - Scale 16 & 0.324 \hspace{0.5em} (0.028) & 0.332 \hspace{0.5em} (0.073) & $1.41\cdot 10^{-6}$\%\\
\hline
\end{tabular}
\caption{Average loss on the final epoch of each task for the scaled Permuted Shakespeare experiments, with means and standard deviations reported over 4 seeds.}
\label{tab:PS_scaled}
\end{table}

\subsubsection{Generalization}

\cite{leeslow} recently propose separating training loss from test loss in studying plasticity; the extant literature and the present work focuses largely on measuring training loss on each task. As a complement, we briefly consider measuring test loss on a holdout set for each task. We consider three setups: first, we consider the PM problem, and measure test loss on a set of 10000 image label-pairs for each task; we annotate this setup PM-G1. Our second setup is a modification of PM-G1: in each task we add noise to the labels by randomly re-labeling a fraction of the training images. We decrease the fraction of images with noisy labels from 50\% on the first task to 0 on the last; we annotate this task PM-G2. Finally in analogy to PM-G2, we consider a variant of the RC problem (RC-G1) where we decay the fraction of random labels from 50\% to 0 linearly across tasks. PM-G2 and RC-G1 are analogous to the setup in \cite{leeslow}.

Detailed results of these experiments are presented in Appendix~\ref{sec:appendix_generalization}; we draw the following conclusions: first, SNR displays similar relative merits to competing algorithms as in our main body of experiments on training loss. Second, through a careful study of variants of the setup in PM-G2 where we alter the number of passes through the dataset relevant to each task, we point out an important issue that requires careful attention if one is to study test loss: specifically, confounding the effects of overfitting with plasticity loss.

\section{Discussion and Limitations}\label{sec:discussion}

A common explanation for plasticity loss is that the network weights obtained from minimizing loss over some task yield poor initializations for a subsequent task -- this explanation continues to motivate the vast majority of algorithms that attempt to combat plasticity loss. Specifically, regularization based algorithms can be viewed as regularizing towards a `good' initial set of weights, whereas reset based algorithms can be viewed as explicitly reinitializing weights to good random initializations.

\textbf{Theoretical Motivation: }Our theoretical development showed that regularization methods themselves are insufficient at combatting the plasticity loss problem, at least in the case of learning ReLUs (Theorems \ref{thm:PositiveRegretGuaranteeIntro} and \ref{cor:NegativeResultFullIntro}). Further we showed that existing proposals to detect whether it was appropriate to reset a ReLU (based on ad hoc notions of a neuron's utility) could have high type 1 error rates relative to an optimal resetting mechanism (Proposition~\ref{prop:prop2}). SNR was designed to minimize the error rate in solving this detection problem (Proposition~\ref{prop:prop1}).

\textbf{SOTA Performance: } Across a range of experiments we see that neuron inactivity is an important correlate of plasticity loss, and that SNR consistently outperforms other reset based methods as well as regularization methods (Table~\ref{tab:average_accuracy}, Figure~\ref{tab:PS_loss}). We also observed that SNR was robust to its hyper-parameters while the adhoc utility schemes employed by other reset methods tended to make them somewhat brittle (Appendix~\ref{sec:hyperparam_sweep}). Finally, the extant literature has focused on measuring plasticity loss through the lens of training loss, but it is also natural to ask about test loss. We see that SNR enjoys similar relative merits in this context as well (Table~\ref{tab:generalization_results}), but one has to be careful to not confound issues of overfitting with plasticity loss (Figure~\ref{fig:CIFAR_16_100_compare}).

\textbf{Attention Menchanisms: }Whereas neuron death is an important correlate of plasticity loss, in the case of language models, a crucial component of the architecture -- the attention mechanism -- does not involve neurons so that the notion of resetting is irrelevant to that component. We have shown that plasticity loss occurs nonetheless and that it is associated with a collapse in the entropy of the attention layer along with neuron death in the feedforward layers (Figures 3 and 4). As such, we see that combining our resets along with L2 regularization of the attention mechanism is important to preserve plasticity (Table~\ref{tab:PS_loss}). Language models also give us a natural substrate to consider the plasticity loss phenomenon vis-a-vis scaling laws, and we see similar relative merits across several orders of magnitude of model scales (Table~\ref{tab:PS_scaled}, Appendix~\ref{tab:scale256}).

\textbf{Limitations: }We recognize a few limitations in the present work. First, whereas we have shown the plasticity loss phenomenon over several orders of model scale our largest models have up to 5M parameters. This is large in the context of the literature on plasticity loss but small in a practical sense. We believe that overcoming this requires understanding ways of exploring this phenomenon without the brute force effort of feeding a model a large sequence of tasks (which is fundamentally serial and hard to accelerate). A second limitation is the theoretical characterization of the phenomenon itself; the approach in the literature has been largely phenomenological -- fixing conjectured root causes for what is really an increase in dynamic regret over time. While Section 4 took a first step, future work ought to connect the plasticity loss phenomenon tightly with the optimization landscape.

\bibliographystyle{plainnat}
\bibliography{bibliography}

\appendix
\section{Additional Experimental Results}

\subsection{Even Larger Scale Experiment for Permuted Shakespeare}

We increase the scale further for the Permuted Shakespeare experiment from our original model by a factor of 256, resulting in a model with 5.1 million non-embedding parameters; results are in Table \ref{tab:scale256}. In accordance with scaling laws, we simultaneously increase the number of tokens or examples per task by a factors of $256^{0.74}$, resulting in 1.9 million tokens or 15'500 training examples, with a context window of 128, for the 256-scale model. Due to computational constraints, we reduce the number of tasks from 500 to 50, but keep a training regime of 100 epochs per task. At the largest scale, over the 50 tasks, our model is trained for 100 epochs over 99 million tokens or 775'000 distinct 128-token-long training examples.

The scale of our largest model is comparable to the scale of the largest transformers considered in recent literature on plasticity loss, namely ViT Tiny (5 million parameters) in \citet{leeslow}. The largest dataset on which \citet{leeslow} train on is the Tiny Imagenet dataset consisting of 100'000 training examples, on which models are trained for 100 epochs; our largest experiment consists of 775'000 training examples. It is important to consider the scale of datasets used in continual learning experiments as the phenomenon of plasticity loss has been shown to be correlated with the amount of data on which a network trains \citep{kumar2023continual, dohare2021continual}.

\begin{table}[h!]\label{tab:Scale256}
\centering
\begin{tabular}{lccc}
\hline
\textbf{Algorithm}  & \textbf{Train Loss at Scale 256} \\
\hline
SNR+L2 & 0.4576 \hspace{0.5em} (0.0017) \\
% SNR+L2 & $0.1949 (0.2309)$ & $0.3919 (0.3587)$ & $0.4550 (0.0111)$ \\
% CBP & $0.2146 (0.2387)$ & $0.4968 (0.3717)$ & $0.4727 (0.0130)$ \\
% L2 & $0.2278 (0.2686)$ & $0.6247 (0.3660)$ & $0.4671 (0.0141)$ \\
\hline
% L2 &  0.2085 (0.1675) & 0.6952 (0.1507) & 0.4644 (0.0046) \\
L2 & 0.4681 \hspace{0.5em} (0.0025) \\
% ReDO & $0.2129 (0.2560)$ & $0.4697 (0.3569)$ & $0.4939 (0.0491)$ \\
\hline
\end{tabular}
\caption{Average training loss on the final epoch over all 50 tasks of the 256-scale Permuted Shakespeare problem with standard deviations reported for 5 random seeds.}
\label{tab:scale256}
\end{table}

\textbf{It is important to note that relative to Table~\ref{tab:PS_scaled} where we show plasticity loss for 500 tasks, the above experiment has only been run to 50 tasks (due to computational budget constraints) and already shows a gap between SNR+L2 and L2.}

\subsection{Generalization Results\label{sec:appendix_generalization}}

 Table~\ref{tab:generalization_results} shows test loss on the last 10\% of tasks for PM-G1, PM-G2 and RC-G1. Figure 5 shows the test accuracy for PM-G1 and PM-G2. We see that SNR enjoys similar relative merits to its competitors in PM-G1 and PM-G2. For RC-G1, it appears that there is no plasticity loss and as such all of the different algorithms perform similarly. Notice that this is not contrary to our observations in RC -- there the labels across tasks were entirely unrelated.

% 230-232 should be in blue
% 152-

We consider next a slight tweak on RC-G1. Specifically, each task in RC-G1 consists of 16 epochs; what happens if we do more? The bottom three curves of Figure 6 show the results of doing 100 epochs for each task. We see here that test loss for no intervention and SNR begin to suffer. As it turns out this is misleading -- specifically, we see that the test accuracy for all three algorithms is significantly below test accuracy when one does only 16 epochs (as reported in Table 5). That is, the decay in test performance is likely attributable to overfitting in each task, and while L2 overfits as well, it is unsurprisingly more resilient. Put a different way, this overfitting can be controlled by the usual tools to combat overfitting. The top three curves of Figure 6 show that SNR performs equivalently to L2 showing that early stopping would accomplish this comfortably in this context.

\begin{figure}
    \centering
    \includegraphics[width=0.48\linewidth]{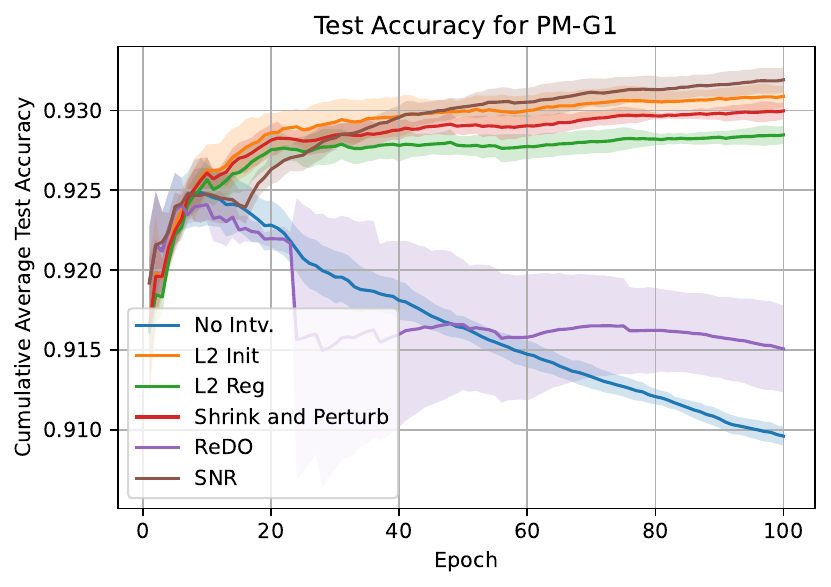}
    \includegraphics[width=0.48\linewidth]{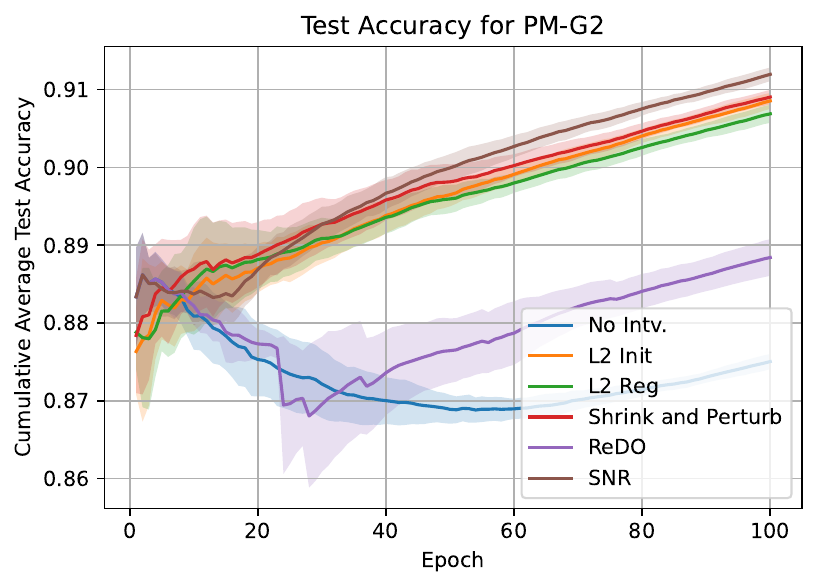}
    \caption{Cumulative average test accuracy for PM-G1 and PM-G2 experiments.}
    \label{fig:PM_generalization}
\end{figure}

\begin{figure}[htbp]
    \centering
    \includegraphics[width=0.5\linewidth]{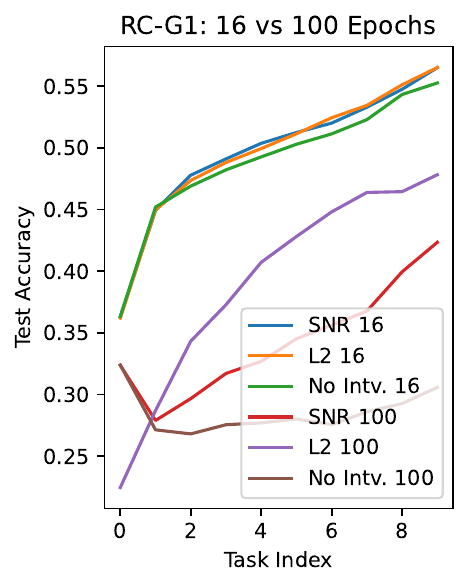}
    \caption{Test accuracy for the RC-G1 experiments evaluated with 16 and 100 epochs per task.}
    \label{fig:CIFAR_16_100_compare}
\end{figure}

% \begin{figure}[htbp]
%     \centering
%     \includegraphics[width=0.48\linewidth]{ICLR 2025 Template/figures/generalization/Gen_CIFAR_compare.pdf}
%     \caption{Plot of the average test accuracy on the final epoch of each task for the Noisey Continual CIFAR-10 experiments with both 16 and 100 epochs per task.}
%     \label{fig:CIFAR_16_100}
% \end{figure}

\begin{table}[htbp]
\centering
\begin{tabular}{@{}lccc@{}}
\hline
\textbf{Algorithm} & \textbf{PM-G1} & \textbf{PM-G2} & \textbf{RC-G1} \\ \hline
No Intv.  & 0.8999 \hspace{0.5em} (0.0041) & 0.8931 \hspace{0.5em} (0.0013) & \textbf{0.5526} \hspace{0.5em} (0.0034) \\
\hline
L2 Init   & 0.9320 \hspace{0.5em} (0.0037) & 0.9284 \hspace{0.5em} (0.0032) & N/A        \\
\hline
L2    & 0.9297 \hspace{0.5em} (0.0034) & 0.9258 \hspace{0.5em} (0.0032) & \textbf{0.5651} \hspace{0.5em} (0.0023)          \\
\hline
S\&P      & 0.9313 \hspace{0.5em} (0.0033) & 0.9281 \hspace{0.5em} (0.0027) & N/A        \\
\hline
ReDO      & 0.9086 \hspace{0.5em} (0.0028) & 0.9083 \hspace{0.5em} (0.0021) & N/A        \\
\hline
SNR       & \textbf{0.9318} \hspace{0.5em} (0.0001) & \textbf{0.9325 }\hspace{0.5em} (0.0029) & \textbf{0.5650} \hspace{0.5em} (0.0013) \\
\hline
\end{tabular}
 \caption{Test accuracy over the last $10\%$ of tasks (mean and standard deviation) for PM-G1, PM-G2, and RC-G1.}
\label{tab:generalization_results}
\end{table}

\section{Complete Results for Benchmark Problems}

To maintain brevity in the main body of the paper, we include here the complete results for the benchmark problems PM, RC, RM, and CI which include both the mean and standard deviation of terminal task accuracies over random seeds.

\begin{table}[htbp]
\centering
\begin{tabular}{lcccc}
\hline
\textbf{Optimizer} & \multicolumn{4}{c}{\textbf{SGD}} \\
\hline
\textbf{Algorithm} & \textbf{PM} & \textbf{RM} & \textbf{RC} & \textbf{CI} \\
\hline
No Intv.     & $0.710$ (0.007) & $0.113 (0.004)$  &  0.180 (0.011) & $0.784$ (0.019) \\
SNR          & $\textbf{0.851} (0.002)$    & $\textbf{0.975}$ (0.001) &  \textbf{0.987} (0.002) & $\textbf{0.888}$ (0.010) \\
CBP          & $0.844 (0.002)$    & $0.951$ (0.007) &  0.961 (0.011) & $0.840$ (0.015)\\
ReDO         & $0.831$ (0.013)   & $0.716$ (0.024) & 0.981 (0.003) & $0.869$ (0.038) \\
L2 Reg.      & $0.818$ (0.001)    & $0.803$ (0.011) & 0.952 (0.006) & $0.833$ (0.011) \\
L2 Init      & $0.829$ (0.001)  & $0.913$ (0.001) & 0.966 (0.002) & $0.832$ (0.010) \\
S\&P         & $0.826$ (0.002) & 0.920 (0.009) & 0.971 (0.004) & $0.853$ (0.006)\\
Layer Norm.  & $0.687$ (0.009)   & $0.143$ (0.015) & 0.959 (0.005) & $0.819$ (0.009) \\
\hline
\textbf{Optimizer} & \multicolumn{4}{c}{\textbf{Adam}} \\
\hline
\textbf{Algorithm} & \textbf{PM} & \textbf{RM} & \textbf{RC} & \textbf{CI} \\
\hline
No Intv.     & $0.641$ (0.007)    & $0.114$ (0.005) &  0.151 (0.005)& $0.581$ (0.081)\\
SNR          & $\textbf{0.889} (0.001)$  & $\textbf{0.982} (0.001) $  &  \textbf{0.976} (0.002) & $\textbf{0.847}$ (0.005) \\
CBP          & $0.876$ (0.001)   & $0.948$ (0.003)  &  0.331 (0.312) & $0.818$ (0.005)\\
ReDO         & $0.846$ (0.002)   & $0.671$ (0.021) & 0.744 (0.131) & $0.803$ (0.063)\\
L2 Reg.      & $0.876 (0.002)$    & $0.948$ (0.002) & 0.967 (0.011) & $0.803$ (0.009)\\
L2 Init      & $0.883$ (0.002)   & $0.961$ (0.003) & \textbf{0.976} (0.002) & $0.827$ (0.008) \\
S\&P         & $0.876$ (0.002) & 0.955 (0.006) & 0.971 (0.005) & $0.814$ (0.005)\\
Layer Norm.  & $0.662$ (0.001)   & $0.113$ (0.005) &  0.955 (0.005) & $0.651$ (0.053)\\
\hline
\end{tabular}
\caption{Average accuracy on the last 10\% of tasks on the benchmark continual learning problems with standard deviations over 5 seeds.}
\label{tab:average_accuracy_complete}
\end{table}

\section{Additional Experimental Details and Hyperparameter Sweep}\label{sec:hyperparam_sweep}

With SGD we train with learning rate $10^{-2}$ on all problems except Random Label MNIST, for which we train with learning rate $10^{-1}$. With Adam we train with learning rate $10^{-3}$ on all problems, including Permuted Shakespeare and we use the standard parameters of $\beta_1 = 0.9, \beta_2 = 0.999$, and $\epsilon = 10^{-7}$. For Permuted Shakespeare we train our networks solely with Adam. The learning rates were selected after an initial hyperparameter sweep.

For each algorithm we vary its hyperparameter(s) by an appropriate constant over 7 choices, effectively varying the hyperparameters over a log scale. With the exception of the Permuted Shakespeare experiment, we limit over hyperparameter search to 5 choices. In Table \ref{tab:hyperparam_sweep_benchmarks} we provide the hyperparameter sweep for the 4 benchmark problems. CBP's replacement rate $r$ is to be interpreted as one replacement per layer every $r^{-1}$ training examples, as presented in \citet{dohare2024loss}. ReDO's reset frequency $r$ determines the frequency of resets in units of tasks, as implemented and evaluated in \citet{kumar2023maintaining}. 

\begin{table}[hbp]
  \centering
  \begin{tabular}{lllllllll}
    \multicolumn{6}{r}{Hyperparameter Strength}                   \\
    \hline
    Algorithm  & 0  & 1 & 2 & 3 & 4 & 5 & 6    \\
    \hline
    L2 Reg. ($\lambda$)  &  $10^{-6}$ & $10^{-5}$ & $10^{-4}$ & 
    $10^{-3}$ & $10^{-2}$ & $10^{-1}$  & $10^{0}$  \\
    \hline
    L2 Init ($\lambda$)& $10^{-6}$ & $10^{-5}$ & $10^{-4}$ & $10^{-3}$ & $10^{-2}$ & $10^{-1}$  & $10^{0}$ \\
    \hline
    S\&P (1-p) & $10^{-8}$ & $10^{-7}$  & $10^{-6}$ & $10^{-5}$ & $10^{-4}$ & $10^{-3}$ & $10^{-2}$    \\
    \hline
    S\&P $(\sigma)$ & $10^{-6}$ & $10^{-5}$ & $10^{-4}$ & $10^{-3}$ & $10^{-2}$ & $10^{-1}$ & $10^{0}$   \\
    \hline
    CBP $(r)$ & $10^{-7}$  & $10^{-6}$ & $10^{-5}$ & $10^{-4}$ & $10^{-3}$ & $10^{-2}$ & $10^{-1}$   \\
    \hline
    ReDO $(\tau)$ & 0 & 0.01 & 0.02 & 0.04 & 0.08 &  0.16 & 0.32 \\ 
    \hline
    ReDO $(r)$ & 64 & 32 & 16 & 8 & 4 & 2 & 1 \\
    \hline
    % 2048 4096 8192 16384 32768 For NRT in PM
    SNR $(\eta)$ & 0.08 & 0.04 & 0.02 & 0.01 & 0.005 & 0.0025 & 0.00125 \\
    \hline
  \end{tabular}
    \caption{Hyperparameter sweep for the Permuted MNIST (PM), Random Label MNIST (RM), Random Label CIFAR (RC), and Continual ImageNet (CI) problems.}
\label{tab:hyperparam_sweep_benchmarks}
\end{table}

\begin{table}[hbp]
  \centering
  \begin{tabular}{llllll}
    \multicolumn{5}{r}{Hyperparameter Strength} \\
    \hline
    Algorithm  & 0  & 1 & 2 & 3 & 4  \\
    \hline
    L2 Reg. ($\lambda$)  &  $10^{-6}$ & $10^{-5}$ & $10^{-4}$ & 
    $10^{-3}$ & $10^{-2}$   \\
    \hline
    L2 Init ($\lambda$)& $10^{-6}$ & $10^{-5}$ & $10^{-4}$ & $10^{-3}$ & $10^{-2}$ \\
    \hline
    SNR $(\eta)$ & 0.1 & 0.05 & 0.03 & 0.01 & 0.001 \\
    \hline
    SNR + L2 Reg $(\eta)$ & 0.1 & 0.05 & 0.03 & 0.01 & 0.001 \\
    \hline
  \end{tabular}
   \caption{Hyperparameter sweep for the Permuted Shakespeare problem. For the combination of SNR and L2 regularization we use the regularization strength of $10^{-4}$, the best performing regularization strength for L2 regularization, and vary the rejection percentile threshold $\eta$}.
\label{tab:hyperparam_sweep_PS}
\end{table} 

\begin{table}
  \centering
  \begin{tabular}{l|l|l|l|l|l}
    \multicolumn{5}{r}{Hyperparameter Strength} \\
    \hline
    Algorithm  & 0  & 1 & 2 & 3 & 4  \\
    \hline
    L2 Reg. ($\lambda$)  & $1.284$ & $0.883$ & $0.348$ & $0.853$ & $4.269$\\
  & $(0.037)$ & $(0.059)$ & $(0.093)$ & $(0.061)$ & $(0.013)$\\
    \hline
    L2 Init ($\lambda$)& $1.791$ & $1.481$ & $1.641$ & $2.429$ & $5.598$\\
  & $(0.056)$ & $(0.050)$ & $(0.067)$ & $(0.124)$ & $(0.030)$\\
    \hline
    SNR $(\eta)$ & $3.034$ & $3.024$ & $3.012$ & $3.027$ & $3.025$\\
 & $(0.022)$ & $(0.041)$ & $(0.025)$ & $(0.029)$ & $(0.045)$\\
    \hline
    SNR + L2 Reg $(\eta)$  & $0.315$ & $0.276$ & $0.325$ & $0.293$ & $0.269$\\
 & $(0.084)$ & $(0.078)$ & $(0.068)$ & $(0.065)$ & $(0.025)$\\
    \hline
  \end{tabular}
   \caption{Average loss of the final epoch of each task in Permuted Shakespeare, averaged over the final 50 tasks, reported as mean (standard deviation) over 5 random seeds.} 
\label{tab:hyperparam_sweep_losses}
\end{table} 

\begin{figure}[htbp]
    \centering
    \includegraphics[width=0.48\linewidth]{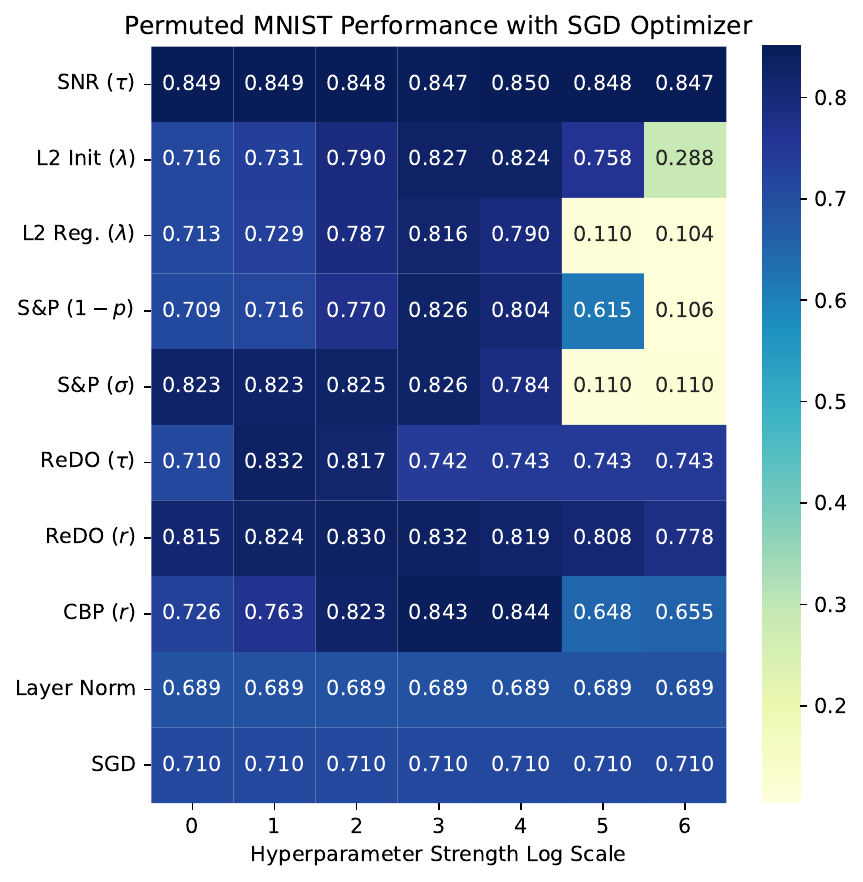}
    \includegraphics[width=0.48\linewidth]{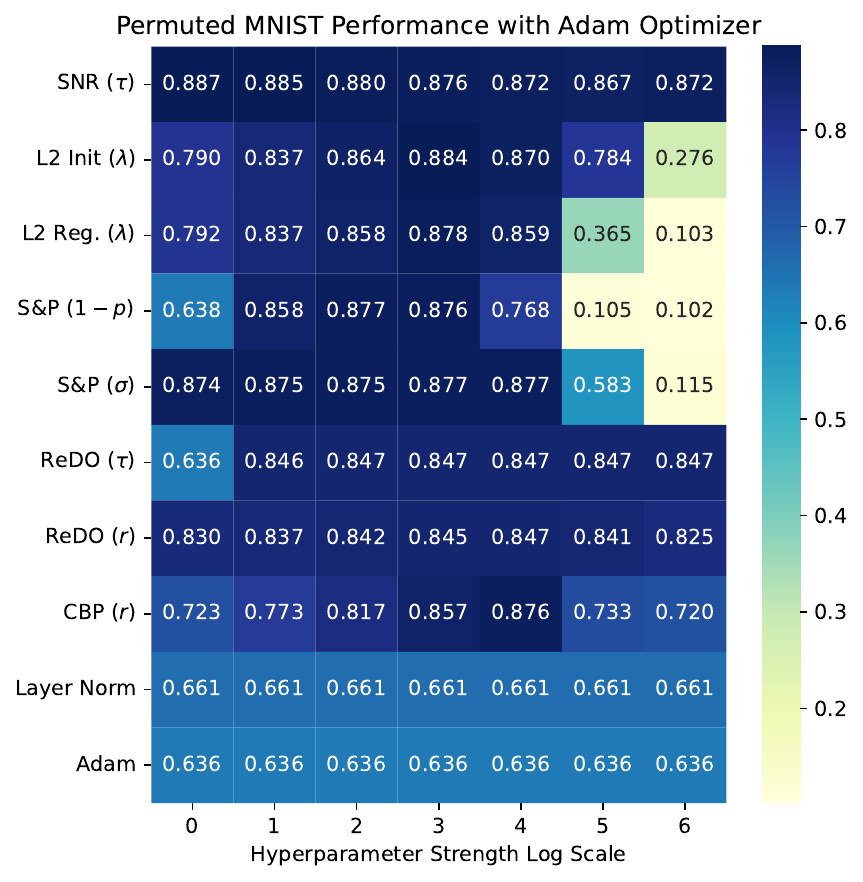}
    \caption{Permuted MNIST hyperparameter sweep results over 5 seeds.}
    \label{fig:hyperparam_sweep_PM}
\end{figure}

\begin{figure}[htbp]
    \centering
    \includegraphics[width=0.48\linewidth]{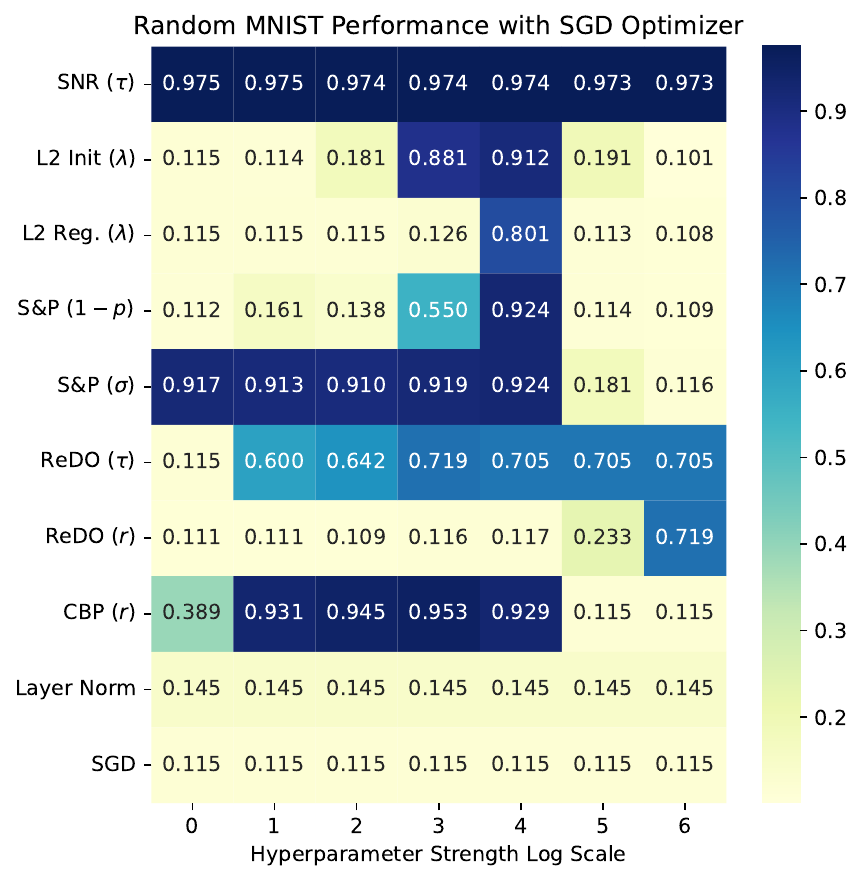}
    \includegraphics[width=0.48\linewidth]{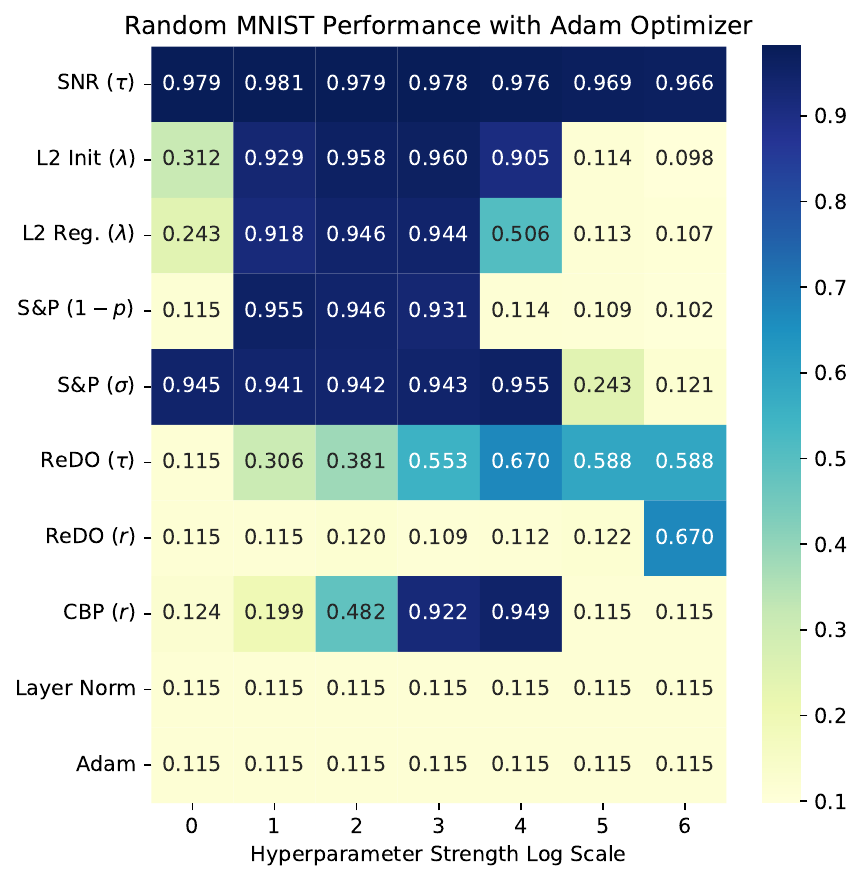}
    \caption{Random LabeL MNIST hyperparameter sweep result over 5 seeds.}
    \label{fig:hyperparam_sweep_RM}
\end{figure}

\begin{figure}[htbp]
    \centering
    \includegraphics[width=0.48\linewidth]{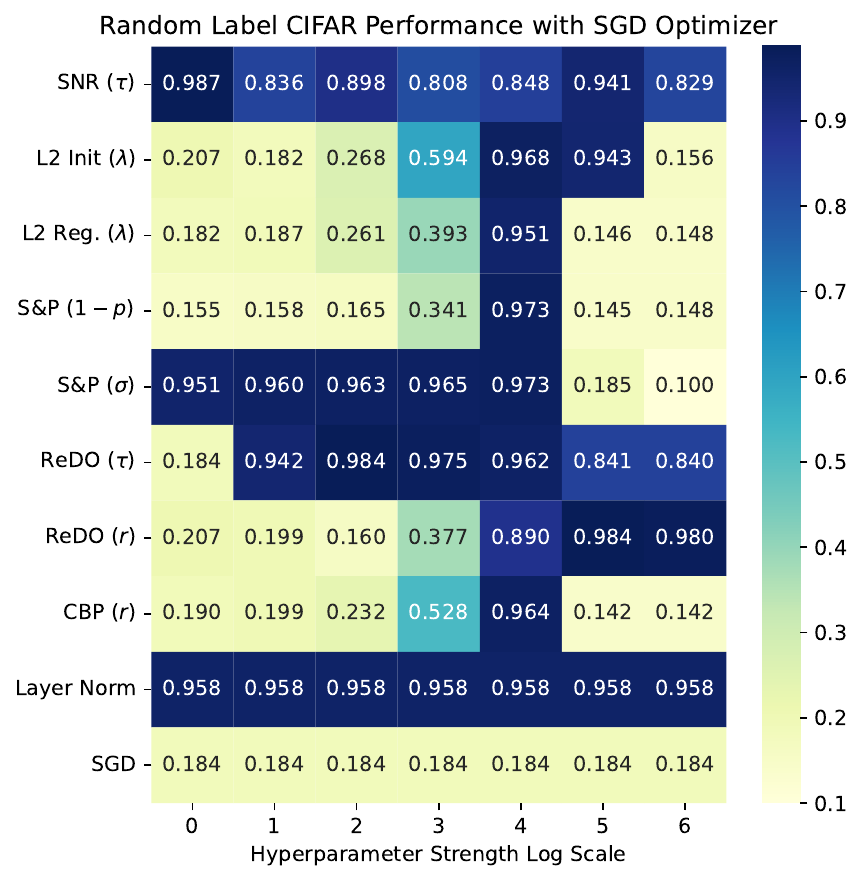}
    \includegraphics[width=0.48\linewidth]{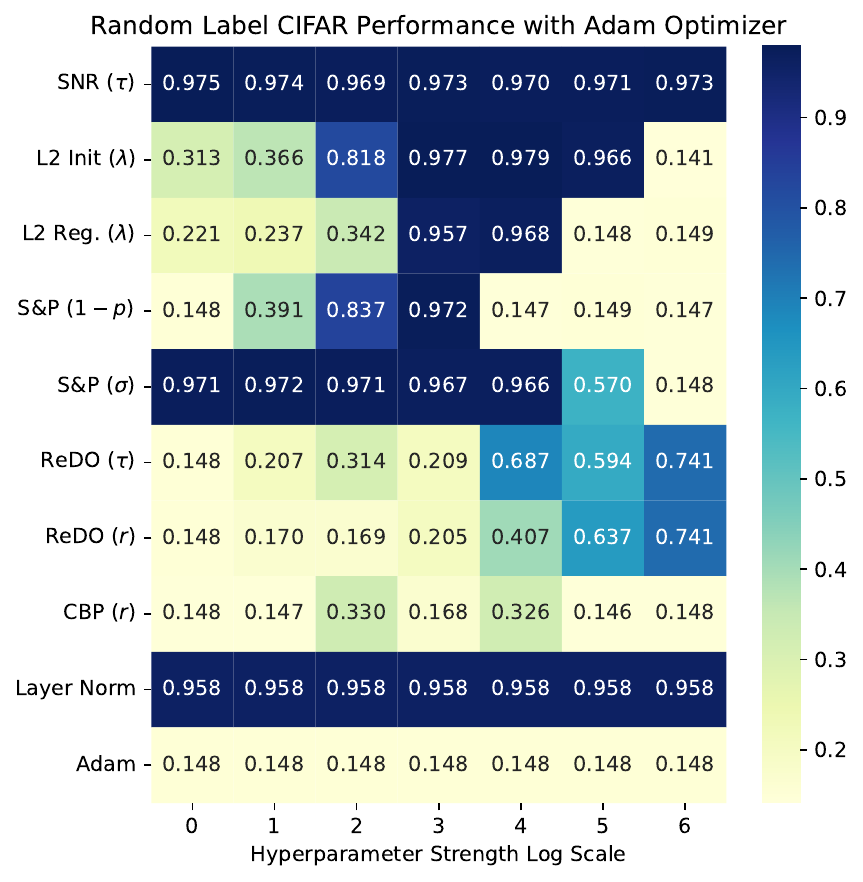}
    \caption{Random Label CIFAR hyperparameter sweep results over 5 seeds.}
    \label{fig:hyperparam_sweep_RC}
\end{figure}

\begin{figure}[htbp]
    \centering
    \includegraphics[width=0.48\linewidth]{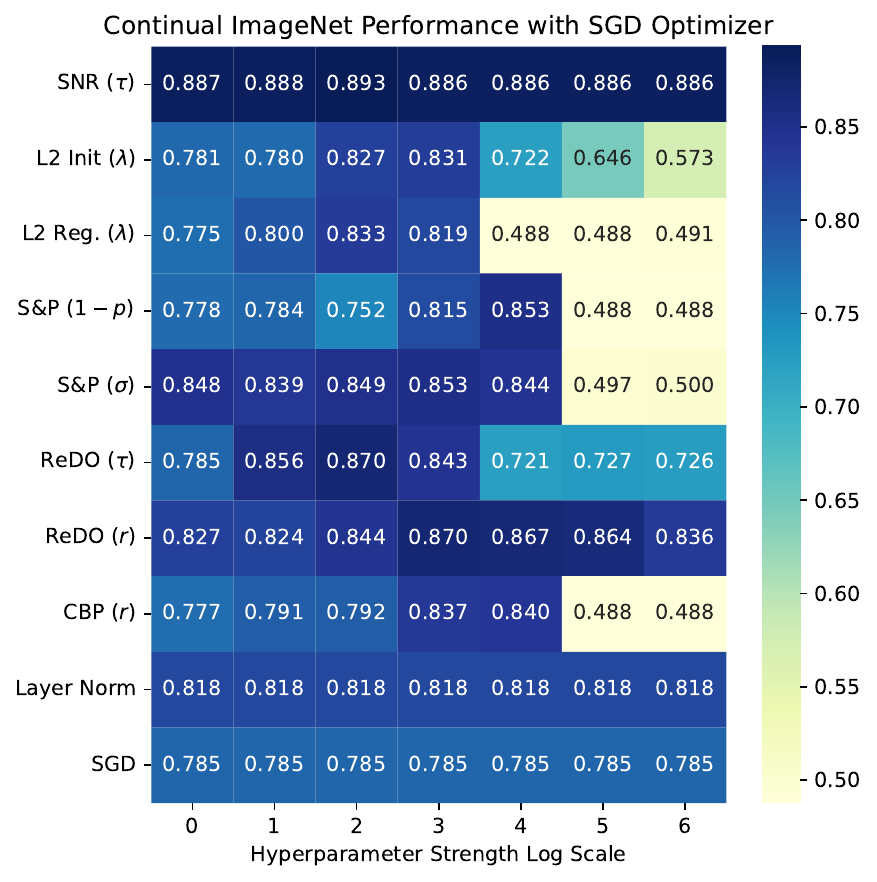}
    \includegraphics[width=0.48\linewidth]{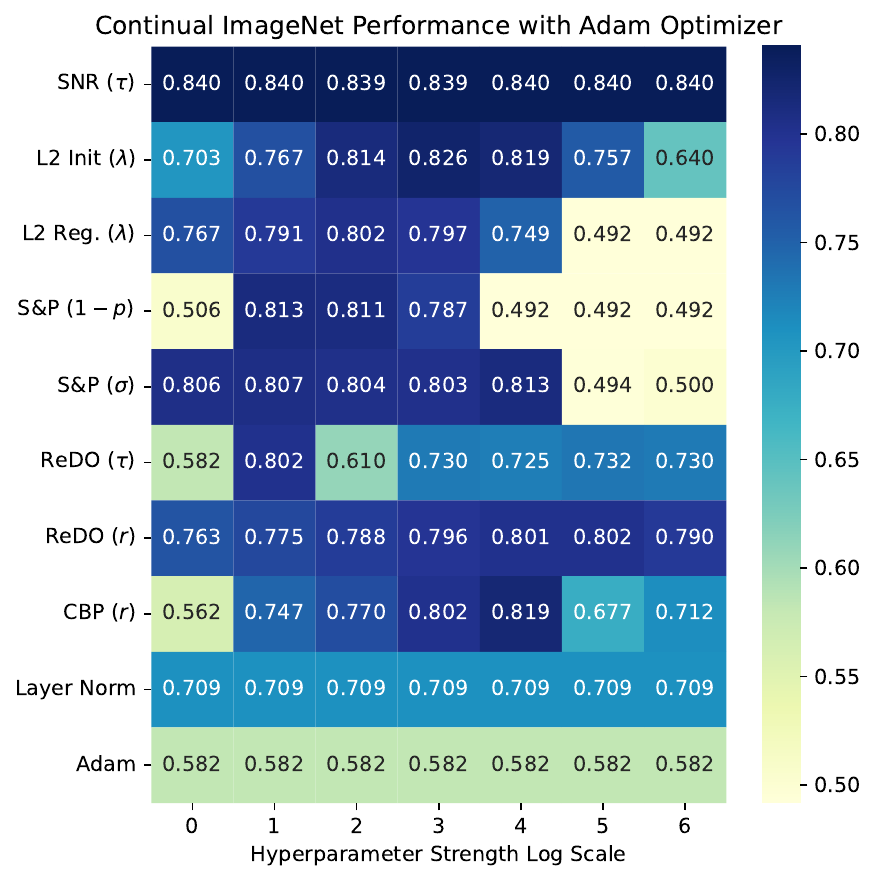}
    \caption{Continual ImageNet hyperparameter sweep results over 5 seeds.}
    \label{fig:hyperparam_sweep_CI}
\end{figure}

\begin{table}[htbp]
\centering
\begin{tabular}{|l|c|c|c|c|c|c|c|}
\hline
 & \multicolumn{7}{c|}{Hyperparameter Strength Log Scale} \\
\hline
 Algorithm & 0 & 1 & 2 & 3 & 4 & 5 & 6 \\
\hline \hline
SNR $(\eta)$ & 0.849 & 0.849 & 0.848 & 0.847 & 0.850 & 0.848 & 0.847 \\
 & (0.002) & (0.003) & (0.003) & (0.004) & (0.002) & (0.002) & (0.002) \\
\hline
L2 Init $(\lambda)$ & 0.716 & 0.731 & 0.790 & 0.827 & 0.824 & 0.758 & 0.288 \\
 & (0.001) & (0.001) & (0.001) & (0.002) & (0.001) & (0.001) & (0.017) \\
\hline
L2 Reg. $(\lambda)$ & 0.713 & 0.729 & 0.787 & 0.816 & 0.790 & 0.110 & 0.104 \\
 & (0.002) & (0.002) & (0.001) & (0.002) & (0.003) & (0.003) & (0.002) \\
\hline
S\&P $(1-p)$ & 0.709 & 0.716 & 0.770 & 0.826 & 0.804 & 0.615 & 0.106 \\
 & (0.001) & (0.002) & (0.001) & (0.002) & (0.004) & (0.001) & (0.003) \\
\hline
S\&P $(\sigma)$ & 0.823 & 0.823 & 0.825 & 0.826 & 0.784 & 0.110 & 0.110 \\
 & (0.001) & (0.002) & (0.002) & (0.002) & (0.002) & (0.002) & (0.002) \\
\hline
ReDO ($\tau$) & 0.710 & 0.832 & 0.817 & 0.742 & 0.743 & 0.743 & 0.743 \\
 & (0.007) & (0.001) & (0.011) & (0.018) & (0.018) & (0.018) & (0.018) \\
\hline
ReDO $(r)$ & 0.815 & 0.824 & 0.830 & 0.832 & 0.819 & 0.808 & 0.778 \\
 & (0.004) & (0.002) & (0.001) & (0.001) & (0.005) & (0.006) & (0.013) \\
\hline
CBP $(r)$ & 0.726 & 0.763 & 0.823 & 0.843 & 0.844 & 0.648 & 0.655 \\
 & (0.006) & (0.002) & (0.002) & (0.001) & (0.002) & (0.000) & (0.000) \\
\hline
Layer Norm & 0.689 & 0.689 & 0.689 & 0.689 & 0.689 & 0.689 & 0.689 \\
 & (0.003) & (0.003) & (0.003) & (0.003) & (0.003) & (0.003) & (0.003) \\
\hline
SGD & 0.710 & 0.710 & 0.710 & 0.710 & 0.710 & 0.710 & 0.710 \\
 & (0.007) & (0.007) & (0.007) & (0.007) & (0.007) & (0.007) & (0.007) \\
\hline
\end{tabular}
\caption{Performance on Permuted MNIST with different hyperparameters reported as mean (standard deviation).}
\end{table}

\begin{table}[htbp]
\centering
\begin{tabular}{|l|c|c|c|c|c|c|c|}
\hline
 & \multicolumn{7}{c|}{Hyperparameter Strength Log Scale} \\
\hline
 Algorithm & 0 & 1 & 2 & 3 & 4 & 5 & 6 \\
\hline \hline
SNR $(\eta)$ & 0.887 & 0.885 & 0.880 & 0.876 & 0.872 & 0.867 & 0.872 \\
 & (0.002) & (0.001) & (0.002) & (0.002) & (0.001) & (0.002) & (0.012) \\
\hline
L2 Init $(\lambda)$ & 0.790 & 0.837 & 0.864 & 0.884 & 0.870 & 0.784 & 0.276 \\
 & (0.001) & (0.002) & (0.002) & (0.002) & (0.002) & (0.001) & (0.008) \\
\hline
L2 Reg. $(\lambda)$ & 0.792 & 0.837 & 0.858 & 0.878 & 0.859 & 0.365 & 0.103 \\
 & (0.002) & (0.002) & (0.002) & (0.002) & (0.004) & (0.008) & (0.002) \\
\hline
S\&P $(1-p)$ & 0.638 & 0.858 & 0.877 & 0.876 & 0.768 & 0.105 & 0.102 \\
 & (0.005) & (0.002) & (0.002) & (0.002) & (0.001) & (0.003) & (0.001) \\
\hline
S\&P $(\sigma)$ & 0.874 & 0.875 & 0.875 & 0.877 & 0.877 & 0.583 & 0.115 \\
 & (0.002) & (0.002) & (0.002) & (0.002) & (0.002) & (0.003) & (0.001) \\
\hline
ReDO ($\tau$) & 0.636 & 0.846 & 0.847 & 0.847 & 0.847 & 0.847 & 0.847 \\
 & (0.009) & (0.002) & (0.001) & (0.001) & (0.001) & (0.001) & (0.001) \\
\hline
ReDO $(r)$ & 0.830 & 0.837 & 0.842 & 0.845 & 0.847 & 0.841 & 0.825 \\
 & (0.004) & (0.002) & (0.002) & (0.002) & (0.001) & (0.001) & (0.001) \\
\hline
CBP $(r)$ & 0.723 & 0.773 & 0.817 & 0.857 & 0.876 & 0.733 & 0.720 \\
 & (0.012) & (0.004) & (0.002) & (0.002) & (0.002) & (0.001) & (0.008) \\
\hline
Layer Norm & 0.661 & 0.661 & 0.661 & 0.661 & 0.661 & 0.661 & 0.661 \\
 & (0.001) & (0.001) & (0.001) & (0.001) & (0.001) & (0.001) & (0.001) \\
\hline
Adam & 0.636 & 0.636 & 0.636 & 0.636 & 0.636 & 0.636 & 0.636 \\
 & (0.009) & (0.009) & (0.009) & (0.009) & (0.009) & (0.009) & (0.009) \\
\hline
\end{tabular}
\caption{Performance on Permuted MNIST with different hyperparameters reported as mean (standard deviation).}
\end{table}

\begin{table}[htbp]
\centering
\begin{tabular}{|l|c|c|c|c|c|c|c|}
\hline
 & \multicolumn{7}{c|}{Hyperparameter Strength Log Scale} \\
\hline
 Algorithm & 0 & 1 & 2 & 3 & 4 & 5 & 6 \\
\hline \hline
SNR $(\eta)$ & 0.975 & 0.975 & 0.974 & 0.974 & 0.974 & 0.973 & 0.973 \\
 & (0.001) & (0.001) & (0.000) & (0.000) & (0.001) & (0.001) & (0.001) \\
\hline
L2 Init $(\lambda)$ & 0.115 & 0.114 & 0.181 & 0.881 & 0.912 & 0.191 & 0.101 \\
 & (0.006) & (0.009) & (0.027) & (0.017) & (0.006) & (0.008) & (0.010) \\
\hline
L2 Reg. $(\lambda)$ & 0.115 & 0.115 & 0.115 & 0.126 & 0.801 & 0.113 & 0.108 \\
 & (0.006) & (0.006) & (0.006) & (0.013) & (0.010) & (0.007) & (0.009) \\
\hline
S\&P $(1-p)$ & 0.112 & 0.161 & 0.138 & 0.550 & 0.924 & 0.114 & 0.109 \\
 & (0.008) & (0.009) & (0.008) & (0.023) & (0.003) & (0.006) & (0.009) \\
\hline
S\&P $(\sigma)$ & 0.917 & 0.913 & 0.910 & 0.919 & 0.924 & 0.181 & 0.116 \\
 & (0.005) & (0.004) & (0.007) & (0.004) & (0.003) & (0.005) & (0.006) \\
\hline
ReDO ($\tau$) & 0.115 & 0.600 & 0.642 & 0.719 & 0.705 & 0.705 & 0.705 \\
 & (0.006) & (0.023) & (0.030) & (0.021) & (0.021) & (0.021) & (0.021) \\
\hline
ReDO $(r)$ & 0.111 & 0.111 & 0.109 & 0.116 & 0.117 & 0.233 & 0.719 \\
 & (0.008) & (0.008) & (0.008) & (0.019) & (0.019) & (0.022) & (0.021) \\
\hline
CBP $(r)$ & 0.389 & 0.931 & 0.945 & 0.953 & 0.929 & 0.115 & 0.115 \\
 & (0.031) & (0.051) & (0.012) & (0.005) & (0.002) & (0.006) & (0.006) \\
\hline
Layer Norm & 0.145 & 0.145 & 0.145 & 0.145 & 0.145 & 0.145 & 0.145 \\
 & (0.013) & (0.013) & (0.013) & (0.013) & (0.013) & (0.013) & (0.013) \\
\hline
SGD & 0.115 & 0.115 & 0.115 & 0.115 & 0.115 & 0.115 & 0.115 \\
 & (0.006) & (0.006) & (0.006) & (0.006) & (0.006) & (0.006) & (0.006) \\
\hline
\end{tabular}
\caption{Performance on Random Label MNIST with different hyperparameters reported as mean (standard deviation).}
\end{table}

\begin{table}[htbp]
\centering
\begin{tabular}{|l|c|c|c|c|c|c|c|}
\hline
 & \multicolumn{7}{c|}{Hyperparameter Strength Log Scale} \\
\hline
 Algorithm & 0 & 1 & 2 & 3 & 4 & 5 & 6 \\
\hline \hline
SNR $(\eta)$ & 0.979 & 0.981 & 0.979 & 0.978 & 0.976 & 0.969 & 0.966 \\
 & (0.001) & (0.001) & (0.002) & (0.001) & (0.002) & (0.001) & (0.003) \\
\hline
L2 Init $(\lambda)$ & 0.312 & 0.929 & 0.958 & 0.960 & 0.905 & 0.114 & 0.098 \\
 & (0.050) & (0.011) & (0.007) & (0.001) & (0.003) & (0.004) & (0.009) \\
\hline
L2 Reg. $(\lambda)$ & 0.243 & 0.918 & 0.946 & 0.944 & 0.506 & 0.113 & 0.107 \\
 & (0.038) & (0.008) & (0.009) & (0.004) & (0.043) & (0.007) & (0.009) \\
\hline
S\&P $(1-p)$ & 0.115 & 0.955 & 0.946 & 0.931 & 0.114 & 0.109 & 0.102 \\
 & (0.006) & (0.004) & (0.008) & (0.005) & (0.006) & (0.009) & (0.004) \\
\hline
S\&P $(\sigma)$ & 0.945 & 0.941 & 0.942 & 0.943 & 0.955 & 0.243 & 0.121 \\
 & (0.009) & (0.009) & (0.007) & (0.013) & (0.004) & (0.015) & (0.006) \\
\hline
ReDO ($\tau$) & 0.115 & 0.306 & 0.381 & 0.553 & 0.670 & 0.588 & 0.588 \\
 & (0.006) & (0.017) & (0.023) & (0.028) & (0.019) & (0.031) & (0.031) \\
\hline
ReDO $(r)$ & 0.115 & 0.115 & 0.120 & 0.109 & 0.112 & 0.122 & 0.670 \\
 & (0.006) & (0.006) & (0.009) & (0.008) & (0.008) & (0.028) & (0.019) \\
\hline
CBP $(r)$ & 0.124 & 0.199 & 0.482 & 0.922 & 0.949 & 0.115 & 0.115 \\
 & (0.030) & (0.055) & (0.077) & (0.007) & (0.004) & (0.006) & (0.006) \\
\hline
Layer Norm & 0.115 & 0.115 & 0.115 & 0.115 & 0.115 & 0.115 & 0.115 \\
 & (0.006) & (0.006) & (0.006) & (0.006) & (0.006) & (0.006) & (0.006) \\
\hline
Adam & 0.115 & 0.115 & 0.115 & 0.115 & 0.115 & 0.115 & 0.115 \\
 & (0.006) & (0.006) & (0.006) & (0.006) & (0.006) & (0.006) & (0.006) \\
\hline
\end{tabular}
\caption{Performance on Random Label MNIST with different hyperparameters reported as mean (standard deviation).}
\end{table}

\begin{table}[htbp]
\centering
\begin{tabular}{|l|c|c|c|c|c|c|c|}
\hline
 & \multicolumn{7}{c|}{Hyperparameter Strength Log Scale} \\
\hline
 Algorithm & 0 & 1 & 2 & 3 & 4 & 5 & 6 \\
\hline \hline
SNR $(\eta)$ & 0.887 & 0.888 & 0.893 & 0.886 & 0.886 & 0.886 & 0.886 \\
 & (0.009) & (0.007) & (0.006) & (0.004) & (0.004) & (0.004) & (0.004) \\
\hline
L2 Init $(\lambda)$ & 0.781 & 0.780 & 0.827 & 0.831 & 0.722 & 0.646 & 0.573 \\
 & (0.028) & (0.016) & (0.011) & (0.008) & (0.014) & (0.014) & (0.018) \\
\hline
L2 Reg. $(\lambda)$ & 0.775 & 0.800 & 0.833 & 0.819 & 0.488 & 0.488 & 0.491 \\
 & (0.010) & (0.013) & (0.012) & (0.007) & (0.000) & (0.000) & (0.001) \\
\hline
S\&P $(1-p)$ & 0.778 & 0.784 & 0.752 & 0.815 & 0.853 & 0.488 & 0.488 \\
 & (0.023) & (0.024) & (0.073) & (0.010) & (0.008) & (0.000) & (0.000) \\
\hline
S\&P $(\sigma)$ & 0.848 & 0.839 & 0.849 & 0.853 & 0.844 & 0.497 & 0.500 \\
 & (0.002) & (0.010) & (0.008) & (0.008) & (0.010) & (0.005) & (0.000) \\
\hline
ReDO ($\tau$) & 0.785 & 0.856 & 0.870 & 0.843 & 0.721 & 0.727 & 0.726 \\
 & (0.024) & (0.012) & (0.010) & (0.015) & (0.041) & (0.043) & (0.040) \\
\hline
ReDO $(r)$ & 0.827 & 0.824 & 0.844 & 0.870 & 0.867 & 0.864 & 0.836 \\
 & (0.027) & (0.038) & (0.029) & (0.010) & (0.016) & (0.011) & (0.009) \\
\hline
CBP $(r)$ & 0.777 & 0.791 & 0.792 & 0.837 & 0.840 & 0.488 & 0.488 \\
 & (0.022) & (0.016) & (0.013) & (0.012) & (0.008) & (0.000) & (0.000) \\
\hline
Layer Norm & 0.818 & 0.818 & 0.818 & 0.818 & 0.818 & 0.818 & 0.818 \\
 & (0.011) & (0.011) & (0.011) & (0.011) & (0.011) & (0.011) & (0.011) \\
\hline
SGD & 0.785 & 0.785 & 0.785 & 0.785 & 0.785 & 0.785 & 0.785 \\
 & (0.024) & (0.024) & (0.024) & (0.024) & (0.024) & (0.024) & (0.024) \\
\hline
\end{tabular}
\caption{Performance on Continual ImageNet with different hyperparameters reported as mean (standard deviation).}
\end{table}

\begin{table}[htbp]
\centering
\begin{tabular}{|l|c|c|c|c|c|c|c|}
\hline
 & \multicolumn{7}{c|}{Hyperparameter Strength Log Scale} \\
\hline
 Algorithm & 0 & 1 & 2 & 3 & 4 & 5 & 6 \\
\hline \hline
SNR $(\eta)$ & 0.840 & 0.840 & 0.839 & 0.839 & 0.840 & 0.840 & 0.840 \\
 & (0.002) & (0.005) & (0.006) & (0.006) & (0.006) & (0.006) & (0.006) \\
\hline
L2 Init $(\lambda)$ & 0.703 & 0.767 & 0.814 & 0.826 & 0.819 & 0.757 & 0.640 \\
 & (0.052) & (0.025) & (0.005) & (0.009) & (0.005) & (0.011) & (0.014) \\
\hline
L2 Reg. $(\lambda)$ & 0.767 & 0.791 & 0.802 & 0.797 & 0.749 & 0.492 & 0.492 \\
 & (0.018) & (0.013) & (0.012) & (0.009) & (0.021) & (0.001) & (0.000) \\
\hline
S\&P $(1-p)$ & 0.506 & 0.813 & 0.811 & 0.787 & 0.492 & 0.492 & 0.492 \\
 & (0.025) & (0.010) & (0.008) & (0.008) & (0.001) & (0.000) & (0.000) \\
\hline
S\&P $(\sigma)$ & 0.806 & 0.807 & 0.804 & 0.803 & 0.813 & 0.494 & 0.500 \\
 & (0.012) & (0.004) & (0.011) & (0.007) & (0.010) & (0.000) & (0.001) \\
\hline
ReDO ($\tau$) & 0.582 & 0.802 & 0.610 & 0.730 & 0.725 & 0.732 & 0.730 \\
 & (0.075) & (0.018) & (0.120) & (0.012) & (0.016) & (0.016) & (0.013) \\
\hline
ReDO $(r)$ & 0.763 & 0.775 & 0.788 & 0.796 & 0.801 & 0.802 & 0.790 \\
 & (0.033) & (0.018) & (0.013) & (0.007) & (0.008) & (0.018) & (0.006) \\
\hline
CBP $(r)$ & 0.562 & 0.747 & 0.770 & 0.802 & 0.819 & 0.677 & 0.712 \\
 & (0.087) & (0.032) & (0.014) & (0.004) & (0.003) & (0.092) & (0.039) \\
\hline
Layer Norm & 0.709 & 0.709 & 0.709 & 0.709 & 0.709 & 0.709 & 0.709 \\
 & (0.014) & (0.014) & (0.014) & (0.014) & (0.014) & (0.014) & (0.014) \\
\hline
Adam & 0.582 & 0.582 & 0.582 & 0.582 & 0.582 & 0.582 & 0.582 \\
 & (0.075) & (0.075) & (0.075) & (0.075) & (0.075) & (0.075) & (0.075) \\
\hline
\end{tabular}
\caption{Performance on Continual ImageNet with different hyperparameters reported as mean (standard deviation).}
\end{table}

\begin{table}[htbp]
\centering
\begin{tabular}{|l|c|c|c|c|c|c|c|}
\hline
 & \multicolumn{7}{c|}{Hyperparameter Strength Log Scale} \\
\hline
 Algorithm & 0 & 1 & 2 & 3 & 4 & 5 & 6 \\
\hline \hline
SNR $(\eta)$ & 0.975 & 0.974 & 0.969 & 0.973 & 0.970 & 0.971 & 0.973 \\
 & (0.003) & (0.001) & (0.006) & (0.004) & (0.004) & (0.004) & (0.002) \\
\hline
L2 Init $(\lambda)$ & 0.313 & 0.366 & 0.818 & 0.977 & 0.979 & 0.966 & 0.141 \\
 & (0.076) & (0.096) & (0.165) & (0.005) & (0.003) & (0.004) & (0.006) \\
\hline
L2 Reg. $(\lambda)$ & 0.221 & 0.237 & 0.342 & 0.957 & 0.968 & 0.148 & 0.149 \\
 & (0.049) & (0.029) & (0.035) & (0.012) & (0.002) & (0.004) & (0.004) \\
\hline
S\&P $(1-p)$ & 0.148 & 0.391 & 0.837 & 0.972 & 0.147 & 0.149 & 0.147 \\
 & (0.004) & (0.228) & (0.170) & (0.002) & (0.005) & (0.004) & (0.005) \\
\hline
S\&P $(\sigma)$ & 0.971 & 0.972 & 0.971 & 0.967 & 0.966 & 0.570 & 0.148 \\
 & (0.004) & (0.002) & (0.002) & (0.009) & (0.004) & (0.039) & (0.008) \\
\hline
ReDO ($\tau$) & 0.148 & 0.207 & 0.314 & 0.209 & 0.687 & 0.594 & 0.741 \\
 & (0.004) & (0.115) & (0.330) & (0.079) & (0.125) & (0.089) & (0.128) \\
\hline
ReDO $(r)$ & 0.148 & 0.170 & 0.169 & 0.205 & 0.407 & 0.637 & 0.741 \\
 & (0.004) & (0.044) & (0.045) & (0.077) & (0.165) & (0.155) & (0.128) \\
\hline
CBP $(r)$ & 0.148 & 0.147 & 0.330 & 0.168 & 0.326 & 0.146 & 0.148 \\
 & (0.004) & (0.007) & (0.302) & (0.026) & (0.074) & (0.006) & (0.004) \\
\hline
Layer Norm & 0.958 & 0.958 & 0.958 & 0.958 & 0.958 & 0.958 & 0.958 \\
 & (0.006) & (0.006) & (0.006) & (0.006) & (0.006) & (0.006) & (0.006) \\
\hline
Adam & 0.148 & 0.148 & 0.148 & 0.148 & 0.148 & 0.148 & 0.148 \\
 & (0.004) & (0.004) & (0.004) & (0.004) & (0.004) & (0.004) & (0.004) \\
\hline
\end{tabular}
\caption{Performance on Random Label CIFAR with different hyperparameters reported as mean (standard deviation).}
\end{table}

\begin{table}[htbp]
\centering
\begin{tabular}{|l|c|c|c|c|c|c|c|}
\hline
 & \multicolumn{7}{c|}{Hyperparameter Strength Log Scale} \\
\hline
 Algorithm & 0 & 1 & 2 & 3 & 4 & 5 & 6 \\
\hline \hline
SNR $(\eta)$ & 0.987 & 0.836 & 0.898 & 0.808 & 0.848 & 0.941 & 0.829 \\
 & (0.002) & (0.242) & (0.112) & (0.124) & (0.183) & (0.064) & (0.173) \\
\hline
L2 Init $(\lambda)$ & 0.207 & 0.182 & 0.268 & 0.594 & 0.968 & 0.943 & 0.156 \\
 & (0.040) & (0.012) & (0.068) & (0.135) & (0.004) & (0.007) & (0.005) \\
\hline
L2 Reg. $(\lambda)$ & 0.182 & 0.187 & 0.261 & 0.393 & 0.951 & 0.146 & 0.148 \\
 & (0.012) & (0.024) & (0.097) & (0.265) & (0.005) & (0.004) & (0.004) \\
\hline
S\&P $(1-p)$ & 0.155 & 0.158 & 0.165 & 0.341 & 0.973 & 0.145 & 0.148 \\
 & (0.005) & (0.007) & (0.009) & (0.081) & (0.003) & (0.005) & (0.004) \\
\hline
S\&P $(\sigma)$ & 0.951 & 0.960 & 0.963 & 0.965 & 0.973 & 0.185 & 0.100 \\
 & (0.007) & (0.006) & (0.004) & (0.003) & (0.003) & (0.068) & (0.014) \\
\hline
ReDO ($\tau$) & 0.184 & 0.942 & 0.984 & 0.975 & 0.962 & 0.841 & 0.840 \\
 & (0.013) & (0.054) & (0.004) & (0.020) & (0.038) & (0.039) & (0.057) \\
\hline
ReDO $(r)$ & 0.207 & 0.199 & 0.160 & 0.377 & 0.890 & 0.984 & 0.980 \\
 & (0.057) & (0.067) & (0.025) & (0.220) & (0.121) & (0.004) & (0.001) \\
\hline
CBP $(r)$ & 0.190 & 0.199 & 0.232 & 0.528 & 0.964 & 0.142 & 0.142 \\
 & (0.021) & (0.024) & (0.058) & (0.266) & (0.012) & (0.006) & (0.006) \\
\hline
Layer Norm & 0.958 & 0.958 & 0.958 & 0.958 & 0.958 & 0.958 & 0.958 \\
 & (0.008) & (0.008) & (0.008) & (0.008) & (0.008) & (0.008) & (0.008) \\
\hline
SGD & 0.184 & 0.184 & 0.184 & 0.184 & 0.184 & 0.184 & 0.184 \\
 & (0.013) & (0.013) & (0.013) & (0.013) & (0.013) & (0.013) & (0.013) \\
\hline
\end{tabular}
\caption{Performance on Random Label CIFAR with different hyperparameters reported as mean (standard deviation).}
\end{table}
\newpage
\section{Learning a Single ReLU-Activated Neuron Under Adversarial Selection}
\subsection{Supporting Propositions}

\begin{proposition}\label{prop:4thmoment}
        Let $v\in\mathbb{R}^{d+1}$ be a unit-norm vector and suppose that $x\sim \mathcal{N}(0,I_d)\times\{1\}$. Then, \begin{equation*}
            \mathbb{E}[\sigma(v^{\top}x)^4]\leq 3
        \end{equation*}
    \end{proposition} \begin{proof}
        First, we note that \begin{equation*}
            \mathbb{E}[\sigma(v^{\top}x)^4] \leq \mathbb{E}[(v^{\top}x)^4]
        \end{equation*} Then $v^\top x = z + \beta$ where $z \sim \mathcal{N}(0,\alpha^2)$ such that $\alpha^2 = \norm{v_{1:d}}^2$ and $\alpha^2 + \beta^2 = 1$ by the fact that $\norm{v}^2  = 1$. Furthermore, \begin{align*}
            \mathbb{E}[(v^\top x)^4] & = \mathbb{E}[(z+\beta)^4] \\
            & = \mathbb{E}[z^4 +4\beta z^3 + 6\beta^2z^2 + 4\beta^3 z + \beta^4]
        \end{align*} By symmetry, $\mathbb{E}[4\beta z^3] = \mathbb{E}[4\beta^3 z] = 0$. For a zero-mean Gaussian with variance $\alpha^2$, $\mathbb{E}[z^4] = 3\alpha^4$ and $\mathbb{E}[6\beta^2z^2] = 6\beta^2\alpha^2$. Therefore, \begin{equation*}
            \mathbb{E}[(v^\top x)^4] = 3\alpha^4 + 6\beta^2\alpha^2 + \beta^4
        \end{equation*} We define $a = \alpha^2$ and by $\beta^2 = 1 -\alpha^2$ we have that \begin{align*}
            \mathbb{E}[(v^\top x)^4]&  = 3a^2 + 6a(1-a) + (1-a)^2 \\
            & = -2a^2 + 4a+1
        \end{align*} By standard analysis of quadratic functions, the maximum is attained at $a=1$ with a value of 3. Therefore, $\mathbb{E}[\sigma(v^\top x)^4]\leq 3$
    \end{proof}

    \begin{proposition}\label{prop:2ndmoment}
        Let $v\in\mathbb{R}^{d+1}$ be a unit-norm vector and suppose that $x\sim \mathcal{N}(0,I_d)\times\{1\}$. Then, \begin{equation*}
            \mathbb{E}[\sigma(v^{\top}x)^2]\leq 1
        \end{equation*}
    \end{proposition} \begin{proof}
        First, we note that \begin{equation*}
            \mathbb{E}[\sigma(v^{\top}x)^2] \leq \mathbb{E}[(v^{\top}x)^2]
        \end{equation*} Then $v^\top x = z + \beta$ where $z \sim \mathcal{N}(0,\alpha^2)$ such that $\alpha^2 = \norm{v_{1:d}}^2$ and $\alpha^2 + \beta^2 = 1$ by the fact that $\norm{v}^2  = 1$. We then observe that, \begin{align*}
            \mathbb{E}[(v^\top x)^2] & = \mathbb{E}[(z+\beta)^2] \\
            & = \mathbb{E}[z^2 + 2\beta z + \beta^2] \\
            & = \alpha^2 + \beta^2  & \text{by $z\sim \mathcal{N}(0,\alpha^2)$} \\
            & = 1 & \text{by $\alpha^2+\beta^2 = 1$}
        \end{align*}
            \end{proof}

            \begin{proposition}\label{prop:SigmaLowerBound}
        If $x,y \in \mathbb{R}$ such that $y \leq 0$ then $\sigma(x+y) \geq \sigma(x) +y$.
    \end{proposition}
    \begin{proof}
        We consider two cases. If $x \geq 0$ then \begin{equation*}
            \sigma(x) + y = x +y \leq \sigma(x+y)
        \end{equation*} On the other hand, if $x < 0$ then \begin{equation*}
            \sigma(x) + y = y \leq 0 \leq \sigma(x+y)
        \end{equation*}
    \end{proof}

\subsection{Convergence in Norm}

In the following lemma, we show that $L$ is bounded above by a quadratic function, which in turn guarantees the convergence of gradient descent when minimizing $L$.

    \begin{lemma}\label{lemma:QuadraticUpperBound}
        For any $w,u \in \mathbb{R}^{d+1}$ then \begin{equation}\label{eqn:QuadraticUpperBound}
            L(w) \leq L(u) + \nabla L(u)^{\top}(w-u) + M\norm{w-u}^2
        \end{equation} where $M = d+1$.
    \end{lemma}
    \begin{proof}
        We let $x \in \mathbb{R}^{d+1}$ such that $x_{d+1}=1$ be arbitrary and we define \begin{equation*}
            L_x(w) = (\sigma(w^{\top}x) - \sigma(v^{\top}x))^2
        \end{equation*} and \begin{equation*}
            \bar{L}_x(w) = (w^{\top}x - \sigma(v^{\top}x))^2
        \end{equation*} such that $L(w) = \mathbb{E}[L_x(w)]$ and $\nabla L(w) = \mathbb{E}[\nabla L_x(w)]$. In order to prove the desired result, (\ref{eqn:QuadraticUpperBound}), by the linearity of expectation it suffices to prove that \begin{equation}\label{eqn:QUBLx}
            L_x(w) \leq L_x(u) + \nabla L_x(u)^{\top}(w-u) + \norm{x}^2\norm{w-u}^2
        \end{equation} since $\mathbb{E}[\norm{x}^2] = d+1$. We first prove that $\bar{L}_x$ has a $2\norm{x}^2$-Lipschitz continuous gradient. \begin{align*}
            ||\nabla \bar{L}_x(w) - \nabla\bar{L}_x(u)|| & = 2\norm{((w^{\top}x - \sigma(v^{\top}x)) - (u^{\top}x - \sigma(v^{\top}x)))x} \\
            & = 2 \norm{x} |w^{\top}x - u^{\top}x| \\
            & \leq 2\norm{x}^2\norm{w-u} & \text{by Cauchy-Schwarz} \\
        \end{align*}
    Therefore, $\bar{L}_x$ has a $2\norm{x}^2$-Lipschitz continuous gradient, and moreover, it follows by standard analysis that $\bar{L}_x$ has a quadratic upper bound of the form \begin{equation}\label{eqn:QUBbarLx}
        \bar{L}_x(w) \leq \bar{L}_x(u) + \nabla\bar{L}_x(u)^{\top}(w-u) + \norm{x}^2\norm{w-u}^2
    \end{equation} We additionally observe that $L_x(w) \leq \bar{L}_x(w)$ by the following argument. If $w^{\top}x \geq 0$ then clearly $L_x(w) = \bar{L}_x(w)$. Otherwise, if $w^{\top}x  <0 $ then we have that \begin{equation*}
        L_x(w)  =(-\sigma(v^{\top}x))^2 \leq (w^{\top}x - \sigma(v^{\top}x))^2 = \bar{L}_x(w)
    \end{equation*} where the inequality follows by the fact that $w^{\top}x< 0$ and $-\sigma(v^{\top}x) \leq 0$. To establish (\ref{eqn:QUBLx}), we consider several cases. Firstly, if $u^{\top}x \geq 0$ then \begin{equation}\label{eqn:QUBLxEquality}
        L_x(u) = \bar{L}_x(u) \text{ and } \nabla L_x(u) = \nabla \bar{L}_x(u)
    \end{equation} Then it follows that \begin{align*}
        L_x(w) & \leq \bar{L}_x(w) & \text{as shown above} \\
        & \leq \bar{L}_x(u) + \nabla \bar{L}_x(u)^{\top}(w-u) + \norm{x}^2\norm{w-u}^2 & \text{by (\ref{eqn:QUBbarLx})} \\
        & = L_x(u) + \nabla L_x(u)^{\top}(w-u) + \norm{x}^2\norm{w-u}^2 & \text{by (\ref{eqn:QUBLxEquality})}
    \end{align*} Next, we consider the case where $u^{\top}x <0$ and $w^{\top}x < 0$. Then it follows that $L_x(w) = L_x(u) = \sigma(v^{\top}x)^2$ and $\nabla L_x(u) = 0$. Hence, \begin{equation*}
        L_x(w) = L_x(u) \leq L_x(u) + \nabla L_x(u)^{\top}(w-u) + \norm{x}^2\norm{w-u}^2
    \end{equation*} Finally, we consider the case where $u^{\top}x<0$ and $w^{\top}x \geq 0$. Then, we let $\bar{w}$ be the vector in the convex combination of $u$ and $w$ such that $\bar{w}^{\top}x = 0$. Then we have by the choice of $\bar{w}$ and $u^{\top}x < 0$\begin{equation}\label{eqn:QUBnormineq}
        \norm{w-\bar{w}} \leq \norm{w-u}
    \end{equation} 
    \begin{equation}\label{eqn:QUB1}
    L_x(u) = \bar{L}_x(\bar{w}) = \sigma(v^{\top}x)^2   
    \end{equation} \begin{equation}\label{eqn:QUB2}
        \nabla L_x(u) = 0
    \end{equation} and 
    \begin{equation}\label{eqn:QUB3}
        \nabla \bar{L}_x(\bar{w}) = 2(\sigma(\bar{w}^{\top}x) - \sigma(v^{\top}x))x\mathbbm{1}_{\{\bar{w}^{\top}x \geq 0\}} = -2\sigma(v^{\top}x)x
    \end{equation} 
    Then we argue as follows \begin{align*}
        L_x(w) & \leq \bar{L}_x(w) & \text{as shown above} \\
        & \leq \bar{L}_x(\bar{w}) + \nabla \bar{L}_x(\bar{w})^{\top}(w - \bar{w}) + \norm{x}^2\norm{w - \bar{w}}^2 & \text{by (\ref{eqn:QUBbarLx})} \\
        & = L_x(u) + \nabla \bar{L}_x(\bar{w})^{\top}(w - \bar{w}) + \norm{x}^2\norm{w - \bar{w}}^2 & \text{by (\ref{eqn:QUB1})} \\
        & \leq L_x(u) + \nabla \bar{L}_x(\bar{w})^{\top}(w - \bar{w}) + \norm{x}^2\norm{w - u}^2 & \text{by (\ref{eqn:QUBnormineq})} \\
        & = L_x(u) - 2 \sigma(v^{\top}x)(w^{\top}x - \bar{w}^{\top}x) + \norm{x}^2\norm{w - u}^2 & \text{by (\ref{eqn:QUB3})} \\
        & = L_x(u) - 2 \sigma(v^{\top}x)(w^{\top}x) + \norm{x}^2\norm{w - u}^2 & \text{by $\bar{w}^{\top}x = 0$} \\
        & \leq L_x(u) + \norm{x}^2\norm{w - u}^2 & \text{by $w^{\top}x, \sigma(v^{\top}x) \geq 0$} \\
        & = L_x(u) + \nabla L_x(u)^{\top}(w-u) + \norm{x}^2\norm{w - u}^2 & \text{by (\ref{eqn:QUB2})} \\
    \end{align*} This completes the proof.\end{proof} 

    \begin{lemma}\label{lemma:GDAnalysis}
        Let $w_0, w_1, \ldots, w_T$ be iterates of gradient descent applied to minimizing $L$, equation (\ref{eqn:Objective}),  with initialization at $w_0$ and step size $\alpha \leq \frac{1}{2(d+1)}$. Then there exists an iterate $w_t \in \{w_0, w_1, \ldots, w_{T-1}\}$ such that \begin{equation}
            \norm{\nabla L(w_t)}^2 \leq \frac{L(w_0)}{T(\alpha - (d+1)\alpha^2)} \leq \frac{\norm{w_0}^2+\norm{v}^2}{T(\alpha - (d+1)\alpha^2)}
        \end{equation} and \begin{equation}\label{eqn:MonotoneDecreaseOfLoss}
            L(w_{t+1}) \leq L(w_t), \,\forall t\geq 0
        \end{equation}
    \end{lemma} \begin{proof}
        According to Lemma \ref{lemma:QuadraticUpperBound}
        \begin{align*}
            L(w_{t+1}) & \leq L(w_t) + \nabla L(w_t)^{\top}(w_{t+1}-w_t) + (d+1)\norm{w_{t+1}-w_t}^2 \\
            & = L(w_t) - \alpha \norm{\nabla L(w_t)}^2 + (d+1)\alpha^2\norm{\nabla L(w_t)}^2 & \text{by GD update} \\
            & = L(w_t) + ((d+1)\alpha^2-\alpha) \norm{\nabla L(w_t)}^2 \\
        \end{align*} Then, \begin{align*}
            (\alpha - (d+1)\alpha^2)\sum_{t=0}^{T-1}\norm{\nabla L(w_t)}^2 & \leq \sum_{t=0}^{T-1} L(w_t) - L(w_{t+1})\\
            & = L(w_0) - L(w_{T})\\
            & \leq L(w_0) & \text{by $L(w_T) \geq 0$}
        \end{align*} This implies that there exists some index $t\in \{0,1,\ldots,T-1\}$ such that \begin{equation*}
             \norm{\nabla L(w_t)}^2 \leq \frac{L(w_0)}{T(\alpha - (d+1)\alpha^2)}
        \end{equation*} Then we bound $L(w_0)$ as follows. \begin{align*}
            L(w_0) & = \mathbb{E}[(\sigma(w_0^{\top}x) - \sigma(v^{\top}x))^2] \\
            & = \mathbb{E}[\sigma(w_0^{\top}x)^2 + \sigma(v^{\top}x)^2 - 2\sigma(w_0^{\top}x)\sigma(v^{\top}x)] \\
            & \leq \mathbb{E}[\sigma(w_0^{\top}x)^2] +\mathbb{E}[\sigma(v^{\top}x)^2] \\
            & = \norm{w_0}^2\mathbb{E}[\sigma(\frac{w_0}{\norm{w_0}}^{\top}x)^2]\mathbbm{1}_{\{w_0\neq0\}} +\norm{v}^2\mathbb{E}[\sigma(\frac{v}{\norm{v}}^{\top}x)^2]\mathbbm{1}_{\{v\neq 0\}} \\
            & \leq \norm{w_0}^2+\norm{v}^2 & \text{by Proposition \ref{prop:2ndmoment}}
        \end{align*} Finally, (\ref{eqn:MonotoneDecreaseOfLoss}) follows by $(d+1)\alpha^2 - \alpha \leq 0$ for all $\alpha \in [0,\frac{1}{d+1}]$. Therefore, by the choice of step size $\alpha \leq \frac{1}{2(d+1)}$ we have that $(d+1)\alpha^2 - \alpha\leq 0$ and so \begin{equation*}
            L(w_{t+1}) \leq L(w_t) + ((d+1)\alpha^2-\alpha)\norm{\nabla L(w_t)}^2 \leq L(w_t)
        \end{equation*}
        
    \end{proof}

        \subsection{Convergence in Iterates}

        We define the directional derivative of $L(w)$ as \begin{equation}
        D_dL(w)= \nabla L(w)^{\top}d = 2\mathbb{E}[(\sigma(w^{\top}x) -\sigma(v^{\top}x))x^{\top}d\mathbbm{1}_{\{w^{\top}x \geq 0\}}]
    \end{equation} For any set $A \subseteq \mathbb{R}^{d+1}$, we define the restricted ``OLS'' objective to be \begin{equation}
        L^r_A(w) = \mathbb{E}[(w^{\top}x-v^{\top}x)^2 \mathbbm{1}_{\{x\in A\}}]
    \end{equation}

    \begin{lemma}\label{lemma:DirectionalDeriv}
        Let $A = \{x: w^{\top}x, v^{\top}x \geq 0\}$. Then \begin{equation*}
            D_{w-v}L(w) = 2L^r_A(w) + c
        \end{equation*} where $c \geq 0$ is some nonnegative value. 
    \end{lemma} \begin{proof}
        Let $x\in \mathbb{R}^{d+1}$ be arbitrary. We break the proof into several cases. If $w^{\top}x, v^{\top}x \geq 0$ then \begin{align*}
            2(\sigma(w^{\top}x)-\sigma(v^{\top}x))x^{\top}(w-v)\mathbbm{1}_{\{w^{\top}x \geq 0\}}
            & = 2( w^{\top}x-v^{\top}x)^2 \\
            & = 2( w^{\top}x-v^{\top}x)^2\mathbbm{1}_{\{x\in A\}}
        \end{align*}
        
        % If $w^{\top}x<0$ then \begin{equation*}
        %     (\sigma(w^{\top}x)-\sigma(v^{\top}x))x^{\top}(w-v)\mathbbm{1}_{\{w^{\top}x \geq 0\}} = 0
        % \end{equation*} 
        
        If $w^{\top}x \geq 0$ and $v^{\top}x < 0$ then \begin{align*}
            2(\sigma(w^{\top}x)-\sigma(v^{\top}x))x^{\top}(w-v)\mathbbm{1}_{\{w^{\top}x \geq 0\}} & = 2(w^{\top}x)(w^{\top}x + |v^{\top}x|) & \text{by $w^{\top}x\geq 0$ and $v^{\top}x< 0$} \\
            & \geq 0 
        \end{align*} For brevity, we define $B = \{x: w^{\top}x \geq 0, v^{\top}x < 0\}$. The preceding case analysis implies that \begin{align*}
            D_{w-v}L(w) & = 2\mathbb{E}[(\sigma(w^{\top}x)-\sigma(v^{\top}x))x^{\top}(w-v)\mathbbm{1}_{\{w^{\top}x \geq 0\}}] \\
            & = 2\mathbb{E}[(\sigma(w^{\top}x)-\sigma(v^{\top}x))x^{\top}(w-v)(\mathbbm{1}_{\{x\in A\}} + \mathbbm{1}_{\{x\in B\}})] \\
            & = 2\mathbb{E}[(w^{\top}x - v^{\top}x)^2\mathbbm{1}_{\{x\in A\}}] + 2\mathbbm{E}[w^{\top}x(w^{\top}x+ |v^{\top}x|)\mathbbm{1}_{\{x\in C\}}] \\
            & = 2L^r_A(w) + c
        \end{align*} where we let $c = 2\mathbb{E}[w^{\top}x(w^{\top}x+ |v^{\top}x|)\mathbbm{1}_{\{x\in C\}}]$. Then $c\geq 0$ by the preceding cases analysis. This completes the proof.
    \end{proof}

    \begin{proposition}\label{prop:prop1_appendix}
        $D_{w-v}(w) \leq \norm{\nabla L(w)} \norm{w-v}$
    \end{proposition} \begin{proof}
        By Cauchy-Schwarz we have that \begin{equation*}
            D_{w-v}(w) = \nabla L(w)^{\top} (w-v) \leq \norm{\nabla L(w)} \norm{w-v}
        \end{equation*}
    \end{proof}

    \begin{proposition}\label{prop:prop2_appendix}
        Let $A \subseteq \mathbb{R}^{d+1}$, then \begin{equation*}
            L^r_A(w) \geq \norm{w-v}^2 \lambda_{\text{min}}(\mathbb{E}[xx^{\top}\mathbbm{1}_{\{x\in A\}}])
        \end{equation*} where $\lambda_{\text{min}}(\cdot)$ is the minimum eigenvalue of a matrix.
    \end{proposition} \begin{proof}
    \begin{align*}
        L^r_A(w) & = \mathbb{E}[(w^{\top}x-v^{\top}x)^2\mathbbm{1}_{\{x\in A\}}] \\
        & = (w-v)^{\top}\mathbb{E}[xx^{\top}\mathbbm{1}_{\{x\in A\}}](w-v) \\
        & \geq  \norm{w-v}^2 \lambda_{\text{min}}(\mathbb{E}[xx^{\top}\mathbbm{1}_{\{x\in A\}}])
        \end{align*} where the last line holds by the fact that $ \mathbb{E}[xx^{\top}\mathbbm{1}_{\{x\in A\}}]$ is a real symmetric matrix.
    \end{proof}

    \begin{lemma}\label{lemma:LowerBoundMinEval}
        Let $x$ be sampled from $\mathcal{N}(0,I_d)\times\{1\}$. Suppose $A = \{x: w^{\top}x \geq 0, v^{\top}x \geq 0\}$ satisfies $\mathbb{P}(A) \geq p$ for some $p > 0$. Then it follows that \begin{equation*}
            \lambda_{\min}\left(\mathbb{E}[xx^{\top}\mathbbm{1}_{\{x\in A\}}] \right) \geq \frac{p^3}{4\sqrt{2\pi e}}
        \end{equation*}
    \end{lemma}
    \begin{proof}
        We begin by noting that \begin{equation*}
            \lambda_{\min}\left(\mathbb{E}[xx^{\top}\mathbbm{1}_{\{x\in A\}}]\right) = \min_{u:\norm{u}=1} \mathbb{E}[(u^{\top}x)^2\mathbbm{1}_{\{x\in A\}}]
        \end{equation*} We proceed by showing that for any $u$ such that $\norm{u} = 1$ then $\mathbb{E}[(u^{\top}x)^2\mathbbm{1}_{\{x\in A\}}] \geq \frac{p^3}{4\sqrt{2\pi e}}$. We let $u$ be an arbitrary vector such that $\norm{u} = 1$. We have that $Z_u = u^{\top}x \sim \mathcal{N}(u_{d+1}, \norm{u_{1:d}}^2)$. Then $\mathbb{E}[Z_u^2\mathbbm{1}_A]$ is minimized when the measure of $A$ contains as much of $u_{1:d}^\perp$ as possible, where 
        \begin{equation*}
            u^\perp = \{x\in\mathbb{R}^d:x^{\top}u_{1:d} = 0\}
        \end{equation*} Therefore, we suppose that \begin{equation}\label{eqn:slab}
            A \supseteq (u_{1:d}^\perp + \{Z_u: Z_u \in [a,b]\}) \times \{1\}
        \end{equation} for some $a < b$. The set above takes its particular form since $A = \{x: w^{\top}x \geq 0, v^{\top}x \geq 0\}$ is a closed convex set, and hence, among all convex sets of measure $\mathbb{P}(A)$, the one that minimizes $\mathbb{E}[Z_u^2\mathbbm{1}_A]$ is necessarily a slab of the form (\ref{eqn:slab}). Then, \begin{equation*}
            \mathbb{E}[Z_u^2\mathbbm{1}_A] \geq \mathbb{E}[Z_u^2\mathbbm{1}_{\{a\leq Z\leq b\}}]
        \end{equation*} We can furthermore minimize the above expectation by considering the interval $[a,b] = [-r,r]$ for any $r > 0$ such that $\mathbb{P}(-r\leq Z_u \leq r) \leq \mathbb{P}(A)$.

        Hence, \begin{equation*}
            \mathbb{E}[Z_u^2\mathbbm{1}_{\{a\leq Z_u\leq b\}}] \geq \mathbb{E}[Z_u^2\mathbbm{1}_{\{-r\leq Z_u\leq r\}}]
        \end{equation*} Furthermore, the above expectation is minimized when $Z_u^2$ is concentrated as much as possible around zero, which occurs precisely when $u_{d+1} = 0 $ and $Z_u = Z \sim \mathcal{N}(0,\norm{u_{1:d}}^2) = \mathcal{N}(0,1)$, where we recall that $\norm{u} = 1$. Therefore, \begin{equation*}
            \mathbb{E}[Z_u^2\mathbbm{1}_{\{-r\leq Z_u\leq r\}}] \geq \mathbb{E}[Z^2\mathbbm{1}_{\{-r\leq Z\leq r\}}]
        \end{equation*} Then we make the following lower bound \begin{equation}\label{eqn:BoundingEvalPause}
            \mathbb{E}[Z^2\mathbbm{1}_{\{-r\leq Z\leq r\}}] \geq \frac{r^2}{4}\mathbb{P}\left(|Z| \in [\frac{r}{2},r]\right)
        \end{equation} We return to the choice of $r$, recalling that $r$ can be as large as possible provided that $\mathbb{P}(-r \leq Z \leq r) \leq \mathbb{P}(A)$. We therefore, choose $ r = p$ and verify that this inequality holds. \begin{align*}
            \mathbb{P}(-r\leq Z\leq r) & = \int_{-r}^r \frac{1}{\sqrt{2\pi}}e^{-\frac{x^2}{2}}\,dx \\
            & \leq \int_{-r}^r \frac{1}{\sqrt{2\pi}} \, dx \\
            & = \sqrt{\frac{2}{\pi}}r \\
            & = \sqrt{\frac{2}{\pi}}p & \text{by choice of $r$} \\
            & \leq p & \text{by $2\leq \pi$} \\
            & \leq \mathbb{P}(A) & \text{by assumption}
        \end{align*} Hence, we return to (\ref{eqn:BoundingEvalPause}) \begin{align*}
            \frac{r^2}{4}\mathbb{P}(|Z| \in [\frac{r}{2},r]) & = \frac{p^2}{4} \mathbb{P}(|Z| \in [\frac{p}{2},p])  & \text{by $r =p$} \\
            & = \frac{p^2}{2} \int_{\frac{p}{2}}^p \frac{1}{\sqrt{2\pi}}e^{-\frac{x^2}{2}}\,dx \\
            & \geq \frac{p^2}{2}\int_{\frac{p}{2}}^p \frac{1}{\sqrt{2\pi}}e^{-\frac{1}{2}}\,dx & \text{by $|p| \leq 1$}\\ 
            & = \frac{p^3}{4\sqrt{2\pi e}}
        \end{align*} Hence, we have our desired result \begin{equation*}
            \lambda_{\min}\left(\mathbb{E}[xx^{\top}\mathbbm{1}_A]\right) \geq \frac{p^3}{4\sqrt{2\pi e}}
        \end{equation*}
    
    \end{proof}

    \begin{lemma}\label{lemma:PuttingTogetherForConvergence}
        If $\norm{\nabla L(w)}^2 \leq \epsilon$, then \begin{equation*}
            \norm{w-v} \leq \sqrt{\epsilon}\frac{2\sqrt{2\pi e}}{\mathbb{P}(w^{\top}x\geq 0,v^{\top}x\geq 0)^3}
        \end{equation*}
    \end{lemma} \begin{proof}
        We let $A = \{x: w^{\top}x\geq 0, v^{\top}x \geq 0\}\times\{1\}$ then
        \begin{align*}
            \sqrt{\epsilon} \norm{w-v} & \geq \norm{\nabla L(w)}\norm{w-v} & \text{by $\norm{\nabla L(w)}^2 \leq \epsilon$} \\
            & \leq D_{w-v}L(w) & \text{by Proposition \ref{prop:prop1_appendix}} \\
            & \geq 2L^r_A(w) & \text{by Lemma \ref{lemma:DirectionalDeriv}} \\
            & \geq 2\norm{w-v}^2 \lambda_{\text{min}}(\mathbb{E}[xx^{\top}\mathbbm{1}_{\{x\in A\}}]) & \text{by Proposition \ref{prop:prop2_appendix}} \\
            & \geq \norm{w-v}^2\frac{\mathbb{P}(w^{\top}x\geq 0,v^{\top}x\geq 0)^3}{2\sqrt{2\pi e}} & \text{by Lemma \ref{lemma:LowerBoundMinEval}} \\
        \end{align*}
        Then the desired result follows \begin{equation*}
            \norm{w-v} \leq \sqrt{\epsilon}\frac{2\sqrt{2\pi e}}{\mathbb{P}(w^{\top}x\geq 0,v^{\top}x\geq 0)^3}
        \end{equation*}
    \end{proof}

    \subsection{Resets and The Reset Oracle}

    % We define the reset oracle $\mathcal{R}_{\delta}(w)$ as the following function. \begin{equation*}
    %     \mathcal{R}_{\delta}(w) = \begin{cases}
    %         1 & \text{ if } \mathbb{E}[\sigma(w^{\top}x)] \leq \delta \\
    %         0 & \text{ if } \mathbb{E}[\sigma(w^{\top}x)] > \delta
    %     \end{cases}
    % \end{equation*} The reset oracle $\mathcal{R}_{\delta}$ is defined by a single threshold $\delta$ and returns 1 or True if and only if the neuron, defined by the weight vector $w$, has its expected activation value fall below the threshold $\delta$.  

    \begin{lemma}\label{lemma:ResetCondition}
        If $\norm{\nabla L(w)}^2\leq \delta^2$, and $\mathbb{P}(w^{\top}x\geq 0,v^{\top}x\geq 0) \leq \frac{\delta^2}{4\norm{v}^2}$ then \begin{equation*}
            \mathcal{R}_{\delta}(w) = 1
        \end{equation*}
    \end{lemma}
    \begin{proof} By $\norm{\nabla L(w)}^2 \leq \delta^2$ we have that \begin{equation}\label{eqn:resetlemma1}
         2|\mathbb{E}[(\sigma(w^{\top}x) - \sigma(v^{\top}x))\mathbbm{1}_{\{w^{\top}x\geq 0\}}]| = |\nabla L(w)_{d+1}| \leq \delta
    \end{equation} Then,
        \begin{align*}
            \mathbb{E}[\sigma(w^{\top}x)] & = \mathbb{E}[(\sigma(w^{\top}x)-\sigma(v^{\top}x))\mathbbm{1}_{\{w^{\top}x \geq 0\}}] + \mathbb{E}[\sigma(v^{\top}x)\mathbbm{1}_{\{w^{\top}x\geq 0\}}] \\
            & = \mathbb{E}[(\sigma(w^{\top}x)-\sigma(v^{\top}x))\mathbbm{1}_{\{w^{\top}x \geq 0\}}] + \mathbb{E}[\sigma(v^{\top}x)\mathbbm{1}_{\{w^{\top}x,v^{\top}x\geq 0\}}] \\
            & \leq \frac{\delta}{2} + \mathbb{E}[\sigma(v^{\top}x)\mathbbm{1}_{\{w^{\top}x,v^{\top}x\geq 0\}}] & \text{by (\ref{eqn:resetlemma1})} \\
            & \leq \frac{\delta}{2} + \sqrt{\mathbb{E}[\sigma(v^\top x)^2]\mathbb{P}(w^{\top}x, v^{\top}x \geq 0)} & \text{by Cauchy-Schwarz} \\
            & = \frac{\delta}{2} + \norm{v}\sqrt{\mathbb{E}[\sigma(\frac{v}{\norm{v}}^\top x)^2]\mathbb{P}(w^{\top}x, v^{\top}x \geq 0)}\mathbbm{1}_{\{v\neq 0\}} \\
            & \leq \frac{\delta}{2} + \norm{v}\sqrt{\mathbb{P}(w^{\top}x, v^{\top}x \geq 0)}\mathbbm{1}_{\{v\neq 0\}} & \text{by Proposition \ref{prop:2ndmoment}} \\
            & \leq \delta & \text{by assumption on $\mathbb{P}(w^{\top}x\geq0,v^{\top}x\geq0)$}
        \end{align*} Therefore, the condition of the reset oracle $\mathcal{R}_{\delta}(w)$ is satisfied and so $\mathcal{R}_{\delta}(w) = 1$.
    \end{proof}

    \begin{lemma}\label{lemma:LowerBoundLossOnResetCond}
        If $E[\sigma(w^{\top}x)] \leq \delta$ then for any $\epsilon > 0$ \begin{equation*}
            \mathbb{P}(\sigma(w^{\top}x) \geq \epsilon) \leq \frac{\delta}{\epsilon}
        \end{equation*} and \begin{equation*}
            L(w) \geq L(0) - \sqrt{\frac{\delta}{\epsilon}}\sqrt{\mathbb{E}[\sigma(v^\top x)^4]} -2\epsilon\sqrt{L(0)}
        \end{equation*}
    \end{lemma}
    \begin{proof}
        The first inequality follows immediately by Markov's inequality \begin{equation}\label{eqn:ResetProbCond}
            \mathbb{P}(\sigma(w^{\top}x) \geq \epsilon) \leq \frac{\mathbb{E}[\sigma(w^{\top}x)]}{\epsilon} \leq \frac{\delta}{\epsilon}
        \end{equation} As for the second inequality, we proceed as follows. \begin{align*}
            L(w) & = \mathbb{E}[(\sigma(w^{\top}x)-\sigma(v^{\top}x))^2] \\
            &  \geq \mathbb{E}[(\sigma(w^{\top}x)-\sigma(v^{\top}x))^2\mathbbm{1}_{\{\sigma(w^{\top}x) \leq \epsilon\}}] \\ 
            & \geq \mathbb{E}[\sigma(v^{\top}x)^2\mathbbm{1}_{\{\sigma(w^{\top}x) \leq \epsilon\}}] -2\mathbb{E}[\sigma(w^{\top}x)\sigma(v^{\top}x)\mathbbm{1}_{\{\sigma(w^{\top}x) \leq \epsilon\}}] \\
            & \geq \mathbb{E}[\sigma(v^{\top}x)^2\mathbbm{1}_{\{\sigma(w^{\top}x) \leq \epsilon\}}] -2\epsilon\sqrt{\mathbb{E}[\sigma(v^{\top}x)^2]\mathbb{P}(\sigma(w^{\top}x) \leq \epsilon)} & \text{by Cauchy Schwarz} \\
            & \geq \mathbb{E}[\sigma(v^{\top}x)^2\mathbbm{1}_{\{\sigma(w^{\top}x) \leq \epsilon\}}] -2\epsilon\sqrt{\mathbb{E}[\sigma(v^{\top}x)^2]}  \\
            & = \mathbb{E}[\sigma(v^{\top}x)^2] - \mathbb{E}[\sigma(v^{\top}x)^2\mathbbm{1}_{\{\sigma(w^{\top}x) > \epsilon\}}]-2\epsilon\sqrt{\mathbb{E}[\sigma(v^{\top}x)^2]}  \\
            & \geq \mathbb{E}[\sigma(v^{\top}x)^2] - \sqrt{\mathbb{E}[\sigma(v^\top x)^4]\mathbb{P}(\sigma(w^{\top}x) > \epsilon)}-2\epsilon\sqrt{\mathbb{E}[\sigma(v^{\top}x)^2]}  & \text{by Cauchy Schwarz} \\
            & \geq \mathbb{E}[\sigma(v^{\top}x)^2] - \sqrt{\frac{\delta}{\epsilon}}\sqrt{\mathbb{E}[\sigma(v^\top x)^4]}-2\epsilon\sqrt{\mathbb{E}[\sigma(v^{\top}x)^2]}  & \text{by (\ref{eqn:ResetProbCond})} \\
            & = \mathbb{E}[\sigma(v^{\top}x)^2] - \sqrt{\frac{\delta}{\epsilon}}\sqrt{\mathbb{E}[\sigma(v^\top x)^4]}-2\epsilon \sqrt{L(0)}  &
        \end{align*}
    \end{proof}

    \begin{corollary}\label{cor:LowerBoundLossResetCond}
        If $E[\sigma(w^{\top}x)] \leq \delta$ then  \begin{equation*}
            \mathbb{P}(\sigma(w^{\top}x) \geq \delta^{\frac{1}{3}}) \leq \delta^{\frac{2}{3}}
        \end{equation*} and \begin{equation*}
            L(w) \geq L(0) - \delta^{\frac{1}{3}}\left(\sqrt{\mathbb{E}[\sigma(v^{\top}x)^4]} +2\sqrt{L(0)}\right)
        \end{equation*}
    \end{corollary}\begin{proof}
        The proof follows immediately by Lemma \ref{lemma:LowerBoundLossOnResetCond} and setting $\epsilon = \delta^{\frac{1}{3}}$.
    \end{proof}

    \subsection{Convergence with High Probability}

    %In this section, we present lemmas that are necessary for establishing that gradient descent with the reset oracle $\mathcal{R}_{\delta}$, with an appropriate reset threshold $\delta$, converges in iterates to the target neuron's weights with high probability over intializations. To this end, we show that with high probability over $w_0$'s initialization, the reset oracle $\mathcal{R}_{\delta}$ will not trigger a premature reset, ensuring convergence.

    \begin{lemma}\label{lemma:BoundingLossAtInitWHP}
        Let $v \in \mathbb{R}^{d+1}$ be the target neuron's weights such that $-v_{d+1} \leq \frac{\norm{v_{1:d}}}{2\sqrt{2\pi}}$. Let $w$ be sampled such that $w_{d+1} = 0$ almost surely and $w_{1:d}$ is sampled from $rS^{d-1}$, the sphere with radius $r = \frac{\norm{v_{1:d}}}{2\pi}$. Then with probability at least $\frac{1}{2}$ \begin{equation*}
            L(w) \leq L(0) - \frac{\norm{v_{1:d}}^2}{8\pi^2}
        \end{equation*}
        \end{lemma}
        \begin{proof}
            \begin{align*}
                L(w) & = \mathbb{E}[(\sigma(w^{\top}x)-\sigma(v^{\top}x))^2] \\
                & = \mathbb{E}[\sigma(v^{\top}x)^2] + \mathbb{E}[\sigma(w^{\top}x)^2]-2\mathbb{E}[\sigma(w^{\top}x)\sigma(v^{\top}x)] \\
                & = L(0) + \norm{w}^2\mathbb{E}[\sigma((\frac{w}{\norm{w}})^{\top}x)^2]-2\mathbb{E}[\sigma(w^{\top}x)\sigma(v^{\top}x)] \\
            \end{align*}  Since $w_{d+1} = 0$, then $(\frac{w}{\norm{w}})^{\top}x = \sum_{i=1}^d \frac{w_ix_i}{\norm{w}}$ and so $(\frac{w}{\norm{w}})^{\top}x$ is a zero-mean normal random variable with variance $\sum_{i=1}^d\frac{w_i^2}{\norm{w}^2} = 1$. Hence, $(\frac{w}{\norm{w}})^{\top}x = Z \sim \mathcal{N}(0,1)$ and we have that  \begin{align*}
            \mathbb{E}[\sigma((\frac{w}{\norm{w}})^{\top}x)^2] & = \mathbb{E}[Z^2\mathbbm{1}_{\{Z \geq 0\}}] \\
            & = \frac{1}{2}\mathbb{E}[Z^2] & \text{by symmetery} \\
            & = \frac{1}{2}
            \end{align*} Thus, we have that \begin{align*}
                L(w) & = L(0) + \frac{\norm{w}^2}{2} -2\mathbb{E}[\sigma(w^{\top}x)\sigma(v^{\top}x)] \\
                & = L(0) + \frac{\norm{w}^2}{2} -2\mathbb{E}[\sigma(w^{\top}x)\sigma(v_{1:d}^{\top}x_{1:d} + v_{d+1})] \\
                & \leq L(0) + \frac{\norm{w}^2}{2} -2\mathbb{E}[\sigma(w^{\top}x)\sigma(v_{1:d}^{\top}x_{1:d} + v_{d+1}\mathbbm{1}_{\{v_{d+1}\leq 0\}})] \\
                & \leq L(0) + \frac{\norm{w}^2}{2} -2\mathbb{E}[\sigma(w^{\top}x)(\sigma(v_{1:d}^{\top}x_{1:d}) + v_{d+1}\mathbbm{1}_{\{v_{d+1}\leq 0\}})] \\
                & = L(0) + \frac{\norm{w}^2}{2} -2\mathbb{E}[\sigma(w^{\top}x)\sigma(v_{1:d}^{\top}x_{1:d})] -2\mathbb{E}[\sigma(w^{\top}x)]v_{d+1}\mathbbm{1}_{\{v_{d+1}\leq 0\}} \\
            \end{align*} where the second-last line above follows by Proposition \ref{prop:SigmaLowerBound}. Then, we bound the fourth term above as follows \begin{align*}
                -2\mathbb{E}[\sigma(w^{\top}x)]v_{d+1}\mathbbm{1}_{\{v_{d+1}\leq 0\}} & = -2\norm{w}\mathbb{E}[Z\mathbbm{1}_{\{Z\geq 0\}}]v_{d+1}\mathbbm{1}_{\{v_{d+1}\leq 0\}} & \text{where $Z\sim \mathcal{N}(0,1)$} \\
                & = - \norm{w}\mathbb{E}[|Z|]v_{d+1}\mathbbm{1}_{\{v_{d+1}\leq 0\}} & \text{by symmetry}\\
                & = -\norm{w}\sqrt{\frac{2}{\pi}}v_{d+1}\mathbbm{1}_{\{v_{d+1}\leq 0\}} & \text{by $\mathbb{E}[|Z|] = \sqrt{\frac{2}{\pi}}$} \\
                & \leq \frac{1}{2\pi}\norm{w}\norm{v_{1:d}}\mathbbm{1}_{\{v_{d+1}\leq 0\}} & \text{by assumption $-v_{d+1} \leq \frac{\norm{v_{1:d}}}{2\sqrt{2\pi}}$} \\
                & \leq \frac{1}{2\pi}\norm{w}\norm{v_{1:d}}
            \end{align*} With this upper bound, we continue our proof \begin{align*}
                L(w) & \leq L(0) + \frac{\norm{w}^2}{2} + \frac{1}{2\pi}\norm{w}\norm{v_{1:d}} - 2\mathbb{E}[\sigma(w^{\top}x)\sigma(v_{1:d}^{\top}x)]
            \end{align*} In order to bound $- 2\mathbb{E}[\sigma(w^{\top}x)\sigma(v_{1:d}^{\top}x)]$ we note the following closed form solution due to \cite{cho2009kernel} \begin{equation*}
                \mathbb{E}[\sigma(w^{\top}x)\sigma(v_{1:d}^{\top}x)] = \frac{1}{2\pi}\norm{w}\norm{v_{1:d}}(\sin\theta +(\pi-\theta)\cos\theta)
            \end{equation*} where $\theta = \arccos\left(\frac{w^{\top}v_{1:d}}{\norm{w}\norm{v_{1:d}}}\right)$. Defining $f(\theta) = \sin\theta+(\pi-\theta)\cos\theta$, we briefly argue that $f(\theta) \geq 1, \forall \theta \in[0,\frac{\pi}{2}]$. First we note that $f'(\theta) = -(\pi -\theta)\sin\theta$ and thus $f'(\theta) \leq 0, \forall \theta \in [0,\pi]$. Since $f(\frac{\pi}{2}) = 1$ then it follows that $f(\theta) \geq 1, \forall \theta \in[0,\frac{\pi}{2}]$. Then with probability $\frac{1}{2}$, $w^{\top}v \geq 0$ and hence $(\sin\theta +(\pi-\theta)\cos\theta) \geq 1$. Thus, \begin{align*}
                L(w) & \leq L(0) + \frac{\norm{w}^2}{2} + \frac{1}{2\pi}\norm{w}\norm{v_{1:d}} - \frac{1}{\pi}\norm{w}\norm{v_{1:d}}(\sin\theta +(\pi-\theta)\cos\theta) \\
                & \leq L(0) + \frac{\norm{w}^2}{2} + \frac{1}{2\pi}\norm{w}\norm{v_{1:d}} - \frac{1}{\pi}\norm{w}\norm{v_{1:d}} & \text{by $w^{\top}v \geq 0$}\\
                & \leq L(0) + \frac{\norm{w}^2}{2} - \frac{1}{2\pi}\norm{w}\norm{v_{1:d}} \\
            \end{align*} Then by the assumption that $w_{1:d}$ is sampled from $rS^{d-1}$ where $r = \frac{\norm{v_{1:d}}}{2\pi}$ we have our desired result, with probability at least $\frac{1}{2}$ \begin{equation*}
                L(w) \leq L(0) + \frac{\norm{v_{1:d}}^2}{8\pi^2} - \frac{\norm{v_{1:d}}^2}{4\pi^2} = L(0) - \frac{\norm{v_{1:d}}^2}{8\pi^2}
            \end{equation*} 
        \end{proof}

        \textbf{Remarks.} For simplicity of exposition, in Lemma~\ref{lemma:BoundingLossAtInitWHP} we assume that $w_{1:d}$ is sampled uniformly from the sphere of radius $\frac{\|v_{1:d}\|}{2\pi}$; alternatively, one may sample $w_{1:d}$ from $\mathcal{N}(0,\sigma^2 I_d)$ with $\sigma^2 = \frac{\|v_{1:d}\|^2}{4\pi^2 d}$, in which case standard concentration of measure arguments imply that 
\[
\mathbb{P}\!\Bigl(\|w\|^2 \in \Bigl[(1-\epsilon)\frac{\|v_{1:d}\|^2}{4\pi^2}, (1+\epsilon)\frac{\|v_{1:d}\|^2}{4\pi^2}\Bigr]\Bigr) \geq 1 - e^{-\Omega(d\epsilon^2)}
\]
We note that the $-v_{d+1} \leq \frac{\|v_{1:d}\|}{2\sqrt{2\pi}}$ in Lemma~\ref{lemma:BoundingLossAtInitWHP} ensures that the bias $v_{d+1}$ is not arbitrarily negative relative to $\|v_{1:d}\|$, an assumption nearly identical to that used in the gradient descent analysis of \cite{vardi2021learning}.

\begin{lemma}\label{lemma:LossComparison}
        Suppose that $C_1\leq \norm{v}$ and $C_2\norm{v} \leq \norm{v_{1:d}}$ for some constants $C_1,C_2 > 0$. Let $\delta < \left(\frac{C_2^2}{8\pi^2(\sqrt{3}+\frac{2}{C_1})}\right)^3$ and let $w\in \mathbb{R}^{d+1}$ such that $\mathbb{E}[\sigma(w^{\top}x)] \leq \delta$ then \begin{equation*}
            L(w) > L(0) - \frac{\norm{v_{1:d}}^2}{8\pi^2}
        \end{equation*}
    \end{lemma} \begin{proof}
            According to Corollary \ref{cor:LowerBoundLossResetCond}, $\mathbb{E}[\sigma(w^{\top}x)]\leq \delta$ implies that  \begin{align*}
                L(w) & \geq L(0)- \delta^{\frac{1}{3}}\left(\sqrt{\mathbb{E}[\sigma(v^{\top}x)^4]} +2\sqrt{L(0)}\right)
            \end{align*} Then the desired result holds if \begin{equation*}
                \delta^{\frac{1}{3}}\left(\sqrt{\mathbb{E}[\sigma(v^{\top}x)^4]} +2\sqrt{L(0)}\right) < \frac{\norm{v_{1:d}}^2}{8\pi^2}
            \end{equation*} By the assumption that $C_2\norm{v}\leq \norm{v_{1:d}}$, the above holds if \begin{equation*}
                \delta^{\frac{1}{3}}\left(\sqrt{\mathbb{E}[\sigma(v^{\top}x)^4]} +2\sqrt{L(0)}\right) < \frac{C_2^2}{8\pi^2}\norm{v}^2
            \end{equation*} We then observe that \begin{align*}
                \frac{\frac{C_2^2}{8\pi^2}\norm{v}^2}{\sqrt{\mathbb{E}[\sigma(v^{\top}x)^4]} +2\sqrt{L(0)}} & = \frac{\frac{C_2^2}{8\pi^2}}{\sqrt{\mathbb{E}[\sigma(\frac{v}{\norm{v}}^{\top}x)^4]} +\frac{2}{\norm{v}}\sqrt{\mathbb{E}[\sigma(\frac{v}{\norm{v}}^\top x)^2]}} \\
                & \geq \frac{C_2^2}{8\pi^2(\sqrt{3} +\frac{2}{\norm{v}})} & \text{by Propositions \ref{prop:4thmoment} and \ref{prop:2ndmoment}}\\
                & \geq \frac{C_2^2}{8\pi^2(\sqrt{3} +\frac{2}{C_1})} & \text{by $C_1 \leq \norm{v}$}\\
                & > \delta^{\frac{1}{3}} & \text{by choice of $\delta$}
            \end{align*} Therefore, we have our desired result, \begin{equation*}
                L(w) > L(0) - \frac{\norm{v_{1:d}}^2}{8\pi^2}
            \end{equation*}
            \end{proof}

        \begin{lemma}\label{lemma:UpperBoundLoss1}
            If $\norm{\nabla L(w)}^2 \leq \epsilon$ then \begin{equation*}
                L(w) \leq \epsilon \frac{8\pi e (d+1)}{\mathbb{P}(w^{\top}x\geq, v^{\top}x\geq 0)^6}
            \end{equation*}
        \end{lemma}
        \begin{proof}
            \begin{align*}
                L(w) & = \mathbb{E}[(\sigma(w^{\top}x)-\sigma(v^{\top}x))^2] \\
                & \leq \mathbb{E}[(w^\top x - v^\top x)^2] \\
                & \leq \norm{w-v}^2\mathbb{E}[\norm{x}^2] & \text{by Cauchy-Schwarz} \\
                & = (d+1)\norm{w-v}^2 \\
                & \leq \epsilon \frac{8\pi e(d+1)}{\mathbb{P}(w^{\top}x\geq0,v^{\top}x\geq 0)^6} & \text{by Lemma \ref{lemma:PuttingTogetherForConvergence}}
            \end{align*}
        \end{proof}

\begin{lemma}\label{lemma:MaxResetTime}
            Suppose that $C_1 \leq \norm{v}$ and $C_2\norm{v} \leq \norm{v_{1:d}}$ for some constants $C_1,C_2 > 0$. Let $\delta < \left(\frac{C_2^2}{8\pi^2(\sqrt{3}+\frac{2}{C_1})}\right)^3$. We define \begin{equation}\label{eqn:tr}
                t_r = \max\left\{\frac{\norm{w_0}^2+\norm{v}^2}{\delta^2(\alpha-(d+1)\alpha^2)}, \frac{2\cdot4^7\pi e(\norm{w_0}^2+\norm{v}^2)\norm{v}^{12}(d+1)}{(L(0)-\frac{\norm{v_{1:d}}^2}{8\pi^2})\delta^{12}(\alpha-(d+1)\alpha^2)} \right\}
            \end{equation} Then the maximum number of gradient descent steps, with step size $\alpha \leq \frac{1}{2(d+1)}$ and weights initialized at $w_0$, until a reset is triggered by the reset oracle $\mathcal{R}_\delta$ is at most $t_r$.  
\end{lemma}
\begin{proof}
    By Lemma \ref{lemma:GDAnalysis}, there exists a time step $t \in \{0,1,\ldots, t_{r-1}\}$ such that
    \begin{equation*}
        \norm{\nabla L(w_t)}^2  \leq \frac{\norm{w_0}^2+\norm{v}^2}{t_r(\alpha - (d+1)\alpha^2)}
    \end{equation*} Then by the choice of $t_r$, we have that $\norm{\nabla L(w_t)}^2 \leq \delta^2$. If $\mathbb{P}(w_t^{\top}x \geq 0, v^{\top}x \geq0) \leq \frac{\delta^2}{4\norm{v}^2}$ then according to Lemma \ref{lemma:ResetCondition} we have that $\mathcal{R}_\delta(w_t) = 1$ and a reset of the weights is triggered at time $t$. As for the case where 
    \begin{equation}\label{eqn:secondcase}
        \mathbb{P}(w_t^{\top}x \geq 0, v^{\top}x \geq0) > \frac{\delta^2}{4\norm{v}^2}
    \end{equation} we argue that $\mathcal{R}_\delta(w_{t'}) = 0$ for all $t'\geq t$, or in other words, no resets are triggered by the reset oracle $\mathcal{R}_\delta$ beyond time $t$. By taking the contrapositive of Lemma \ref{lemma:LossComparison} we have that \begin{equation}\label{eqn:MaxResetTime1}
        L(w_t) \leq L(0) - \frac{\norm{v_{1:d}}^2}{8\pi^2}
    \end{equation} implies that $\mathbb{E}[\sigma(w_t^{\top}x)] >\delta$, or equivalently, $\mathcal{R}_\delta(w_t) = 0$. Furthermore, Lemma \ref{lemma:GDAnalysis} ensures that $L(w_{t'}) \leq L(w_t)$ for all $t' \geq t$, and so it suffices to establish (\ref{eqn:MaxResetTime1}) for iterate $w_t$. Then according to Lemma \ref{lemma:UpperBoundLoss1} we have that \begin{align*}
        L(w_t) & \leq \frac{\norm{w_0}^2+\norm{v}^2}{t_r(\alpha - (d+1)\alpha^2)} \frac{8\pi e(d+1)}{\mathbb{P}(w_t^{\top}x\geq 0,v^{\top}x\geq 0)^6} \\
        & < \frac{2\cdot4^7\pi e(\norm{w_0}^2+\norm{v}^2)\norm{v}^{12}(d+1)}{t_r\delta^{12}(\alpha-(d+1)\alpha^2)} & \text{by (\ref{eqn:secondcase})}\\
        & \leq L(0)-\frac{\norm{v_{1:d}}^2}{8\pi^2} & \text{by chocie of $t_r$}
    \end{align*}
\end{proof}

\subsection{Regret Guarantees}

% Given that the optimal parameter for (\ref{eqn:Objective}) is $w^* = v$ and thus \begin{equation*}
%     L(w^*) = L(v) = 0
% \end{equation*} then the average regret is defined to be \begin{equation}\label{eqn:regretdefn}
%     R_T = \frac{1}{T}\sum_{t=0}^{T-1}L(w_t) -L(w^*) = \frac{1}{T}\sum_{t=0}^{T-1}L(w_t)
% \end{equation}

\begin{theorem}\label{thm:PositiveRegretGuarantee}
Let $w_0$ be sampled from \(\mathcal{D}_{w_0} = \text{Unif}(l\cdot S^{d-1}\times\{0\})\) for some positive constant $l>0$, where $l\cdot S^{d-1}\subseteq \mathbb{R}^d$ is the $r$-radius sphere. Let \(v\in\mathbb{R}^{d+1}\) (possibly chosen adversarially with knowledge of \(w_0\)) be the target neuron's weights satisfying
\[
-v_{d+1} \le \frac{\|v_{1:d}\|}{2\sqrt{2\pi}},\quad \|v_{1:d}\|=2\pi l,\quad \|v\|\ge C_1,\quad \|v_{1:d}\|\ge C_2\|v\|,
\]
for some constants \(C_1,C_2>0\), and suppose
\[
\delta<\left(\frac{C_2^2}{8\pi^2\Bigl(\sqrt{3}+\frac{2}{C_1}\Bigr)}\right)^3.
\]
Then, minimizing \(L(\cdot)\) via gradient descent with step size \(\alpha\le \frac{1}{2(d+1)}\), initialized at \(w_0\), and employing a reset oracle \(\mathcal{R}_\delta\), yields
\[
\forall\, T\ge0,\quad \mathbb{E}[R_T]\le  \mathcal{O}\!\Bigl(\frac{d^2\ln T}{T}\Bigr),
\]
where the expectation is taken over the randomness of the the initialization of $w_0$ and the resets.
\end{theorem}
\begin{proof}
We let $A_r$ be the event that exactly $r$ resets occur while minimizing the objective $L(\cdot)$ with gradient descent using step size $\alpha$ and employing a reset oracle $\mathcal{R}_{\delta}$. Therefore, we have that \begin{equation*}
    \mathbb{E}[R_T] = \frac{1}{T}\sum_{r=0}^{T}\mathbb{E}[R_T\mid A_r]\mathbb{P}(A_r)
\end{equation*} We begin by bounding $\mathbb{P}(A_r)$. We let $p$ be the probability that a reset does not ever occur when (re)initializing $w_t$ from $\mathcal{D}_{w_0}$. We proceed to show that $p \geq \frac{1}{2}$. To this end, we suppose that $w_t$ is sampled from $\mathcal{D}_{w_0}$, and according to Lemma \ref{lemma:BoundingLossAtInitWHP}, with probability at least $\frac{1}{2}$ we have that $L(w_t) \leq L(0) - \frac{\norm{v_{1:d}}^2}{8\pi^2}$. Then for any iterate $w_{t'}$, where $t' \geq t$, we have that $L(w_{t'}) \leq L(w_t)$ due to Lemma \ref{lemma:GDAnalysis}. By the contrapositive of Lemma \ref{lemma:LossComparison}, we have that $\mathbb{E}[\sigma(w_{t'}^\top x)]>\delta$, or equivalently, $\mathcal{R}_{\delta}(w_{t'}) = 0$. Hence, (re)sampling $w_t$ from $\mathcal{D}_{w_0}$ ensures that no resets occur for $w_t$ and any future iterates with probability $p \geq \frac{1}{2}$.

We return to bounding $\mathbb{P}(A_r)$. In the worst case, if $v$ is chosen adversarially with the knowledge of $w_0$ such that a reset occurs then $\mathbb{P}(A_0) = 0$ and $\forall r \geq 1, \mathbb{P}(A_r) = p(1-p)^{r-1} \leq \frac{1}{2^{r-1}}$. Hence, \begin{equation*}
    \mathbb{E}[R_T] \leq \frac{1}{T}\sum_{r=1}^{T}\frac{1}{2^{r-1}}\mathbb{E}[R_T\mid A_r]
\end{equation*} Next, we bound $\mathbb{E}[R_T\mid A_r]$. We let $\mathbb{E}[R_T\mid A_r] = M_{1,r}+M_{2,r}$, where $M_{1,r}$ is the total loss of all iterates preceding the final ($r^{\text{th}}$) reset, and where $M_{2,r}$ is the total loss of all iterates after the final reset. We begin by bounding $M_{1,r}$. We let $w_t$ be an arbitrary (re)initialization. Then for any subsequent iterate $w_{t'}$ that precedes the next reset, we have that \begin{align*}
    L(w_{t'}) & \leq L(w_t) & \text{by Lemma \ref{lemma:GDAnalysis}} \\
    & \leq \mathbb{E}[\sigma(v^\top x)^2] + \mathbb{E}[\sigma(w_t^\top x)^2] \\
    & \leq \norm{v}^2 + \norm{w_t}^2 & \text{by Proposition \ref{prop:2ndmoment}} \\
    & = \norm{v}^2 + l^2 & \text{by $w_t \sim \mathcal{D}_{w_0}$}
\end{align*} According to Lemma \ref{lemma:MaxResetTime} \begin{equation}
                t_{\text{reset}} = \max\left\{\frac{\norm{w_t}^2+\norm{v}^2}{\delta^2(\alpha-(d+1)\alpha^2)}, \frac{2\cdot4^7\pi e(\norm{w_t}^2+\norm{v}^2)\norm{v}^{12}(d+1)}{(L(0)-\frac{\norm{v_{1:d}}^2}{8\pi^2})\delta^{12}(\alpha-(d+1)\alpha^2)} \right\}
            \end{equation} is the maximum number of gradient descent updates before a reset occurs. Therefore, \begin{equation}\label{eqn:UpperBoundOnM1}
                M_{1,r} \leq r t_{\text{reset}}(\norm{v}^2+l^2)
            \end{equation} As for $M_{2,r}$, we have the following upper bound \begin{equation*}
                M_{2,r} \leq \sum_{t=0}^{T-1-r}L(w_t)
            \end{equation*} where we suppose that $w_{0},w_1,\ldots,w_{T-1-r}$ are arbitrary iterates that do not trigger the reset oracle $\mathcal{R}_{\delta}$ and $w_0$ is sampled from $\mathcal{D}_{w_0}$. We proceed to bound $L(w_t)$. To this end, we let \begin{equation}\label{eqn:PosRegretTbar}
                \bar{t} = \lceil\frac{\norm{w_0}^2+\norm{v}^2}{\delta^2(\alpha-(d+1)\alpha^2)}\rceil = \lceil\frac{l^2+\norm{v}^2}{\delta^2(\alpha-(d+1)\alpha^2)}\rceil
            \end{equation} For $t \leq \bar{t}$ we crudely bound $L(w_t) \leq \norm{v}^2+l^2$, as argued earlier in the proof. For $t \geq \bar{t}$ we have, according to Lemma \ref{lemma:GDAnalysis} and the choice of $\bar{t}$, that for some $t' \leq t$, 
            \begin{equation}\label{eqn:PosRegretGradBound}
                \norm{\nabla L(w_{t'})}^2 \leq \frac{\norm{w_0}^2+\norm{v}^2}{t(\alpha-(d+1)\alpha^2)} \leq \frac{\norm{w_0}^2+\norm{v}^2}{\bar{t}(\alpha-(d+1)\alpha^2)}\leq \delta^2
            \end{equation} By the fact that $\mathcal{R}_\delta(w_{t'}) = 0$ and the contrapositive of Lemma \ref{lemma:ResetCondition}, it follows that \begin{equation*}
                \mathbb{P}(w_{t'}^\top x\geq 0, v^\top x\geq 0) > \frac{\delta^2}{4\norm{v}^2}
            \end{equation*} Then we have that \begin{align*}
                L(w_t) & \leq L(w_{t'}) & \text{by Lemma \ref{lemma:GDAnalysis}} \\
                 & \leq  \norm{\nabla L(w_{t'})}^2\frac{8\pi e(d+1)}{\mathbb{P}(w_{t'}^{\top}x\geq 0,v^{\top}x\geq 0)^6} & \text{by Lemma \ref{lemma:UpperBoundLoss1}}\\ 
                 & \leq \frac{1}{t}\cdot\frac{2\cdot4^{7}\pi e(\norm{w_0}^2+\norm{v}^2)\norm{v}^{12}(d+1)}{\delta^{12}(\alpha-(d+1)\alpha^2)} & \text{by (\ref{eqn:PosRegretTbar}) and (\ref{eqn:PosRegretGradBound})}\\
                 & =  \frac{1}{t}\cdot\frac{2\cdot4^{7}\pi e(l^2+\norm{v}^2)\norm{v}^{12}(d+1)}{\delta^{12}(\alpha-(d+1)\alpha^2)} & \text{by $w_0 \sim \mathcal{D}_{w_0}$}
            \end{align*} Therefore, \begin{align*}
                M_{2,r} & \leq (\bar{t}+1)(\norm{v}^2+l^2) + \sum_{t=\bar{t}+1}^{T-1-r}\frac{1}{t}\cdot\frac{2\cdot4^{7}\pi e(l^2+\norm{v}^2)\norm{v}^{12}(d+1)}{\delta^{12}(\alpha-(d+1)\alpha^2)} \\
                & \leq (\bar{t}+1)(\norm{v}^2+l^2) + \ln T\cdot\frac{2\cdot4^{7}\pi e(l^2+\norm{v}^2)\norm{v}^{12}(d+1)}{\delta^{12}(\alpha-(d+1)\alpha^2)} 
            \end{align*} where the last line follows by the standard upper bound on the Harmonic sum. Putting everything together,
            \begin{align*}
                \mathbb{E}[R_T] & \leq \frac{1}{T}\sum_{r=1}^{T}\frac{1}{2^{r-1}}(M_{1,r}+M_{2,r}) \\
                 % & \leq \frac{1}{T}\sum_{r=1}^{T-1}\frac{1}{2^{r-1}}(r t_{\text{reset}}(\norm{v}^2+l^2)+\bar{t}(\norm{v}^2+l^2) + \ln T\cdot\frac{2\cdot4^{7}\pi e(l^2+\norm{v}^2)\norm{v}^{12}}{\delta^{12}(\alpha-(d+1)\alpha^2)} ) \\
            \end{align*} and proceed by bounding each term separately. \begin{align*}
                \frac{1}{T}\sum_{r=1}^{T}\frac{1}{2^{r-1}}M_{1,r} & \leq \frac{1}{T}\sum_{r=1}^{T}\frac{1}{2^{r-1}}(r t_{\text{reset}}(\norm{v}^2+l^2) \\
                & \leq \frac{4t_{\text{reset}}(\norm{v}^2+l^2)}{T} 
            \end{align*} where the last line above follows by upper bounding $\sum_{r=1}^{T}\frac{r}{2^{r-1}}\leq 4$ by a standard analysis of the arithmetico-geometric series. Similarly, we upper bound \begin{align*}
                 \frac{1}{T}\sum_{r=1}^{T}\frac{1}{2^{r-1}}M_{2,r} & \leq \frac{2(\bar{t}+1)(\norm{v}^2+l^2)}{T} + \frac{\ln T}{T}\cdot\frac{4^{8}\pi e(l^2+\norm{v}^2)\norm{v}^{12}(d+1)}{\delta^{12}(\alpha-(d+1)\alpha^2)} 
            \end{align*} where we bound $\sum_{r=1}^{T-1}\frac{1}{2^{r-1}}\leq 2$ by a standard analysis of a geometric series. Then the expected average regret is \begin{equation*}
                \mathbb{E}[R_T] \leq \frac{4t_{\text{reset}}(\norm{v}^2+l^2)}{T}  + \frac{2\bar{t}(\norm{v}^2+l^2)}{T} + \frac{\ln T}{T}\cdot\frac{4^{8}\pi e(l^2+\norm{v}^2)\norm{v}^{12}(d+1)}{\delta^{12}(\alpha-(d+1)\alpha^2)} 
            \end{equation*}Then taking $\alpha = \frac{c}{2(d+1)}$ for some constant $c\in (0,1]$ we have that \begin{equation*}
                \alpha - (d+1)\alpha^2 = \frac{c}{2(d+1)} - \frac{c^2}{4(d+1)} \geq \frac{c}{4(d+1)}
            \end{equation*} Furthermore, if we take $\norm{v}$, $\norm{w_0}= l$, and $\delta$ as constants bounded according to the assumptions of the theorem statement, we have that $t_{\text{reset}}\leq \mathcal{O}(d^2)$ and $\bar{t}\leq\mathcal{O}(d)$ and, moreover, \begin{equation*}
                \mathbb{E}[R_T] \leq \mathcal{O}(\frac{d^2}{T} + \frac{d}{T} + \frac{d^2\ln{T}}{T}) = \mathcal{O}(\frac{d^2\ln{T}}{T})
            \end{equation*}
            
\end{proof}

\begin{corollary}\label{cor:NegativeResultFull}
    Let \(L(\cdot)\) be the objective function for learning a single neuron (see Eq.~\eqref{eqn:Objective}), and let \(w_0 \in \mathbb{R}^{d+1}\) be an initialization satisfying \((w_0)_{d+1} \le 0\). Define the target weight vector \(v \in \mathbb{R}^{d+1}\) so that \(v_{1:d} = -(w_0)_{1:d}\) and \(v_{d+1} < -(w_0)_{d+1}\). Then, for any step size \(\alpha < \frac{1}{2(d+1)}\) and any regularization parameter \(\lambda < 2d\), performing gradient descent on the L2-regularized objective \(\hat{L}(\cdot)\) from \(w_0\) satisfies
    \[
        \forall T \ge 0, \quad R_T \;\ge\; L(0) \;>\; 0.
    \]
\end{corollary}
\begin{proof}
    The result follows immediately by Theorem \ref{thm:NegativeResultFull} and noting that \begin{equation*}
        R_T = \frac{1}{T}\sum_{t=0}^{T-1}L(w_t) \geq \frac{1}{T}\sum_{t=0}^{T-1}L(0) = L(0)
    \end{equation*}
\end{proof}

    \subsection{Negative Results}

% We define the L2-regularized objective function \(\hat{L}(w)\) as
% \begin{equation}\label{eqn:RegularizedObjective}
%     \hat{L}(w) \;=\; \mathbb{E}\bigl[(\sigma(w^{\top} x) - \sigma(v^{\top} x))^2\bigr] \;+\; \frac{\lambda}{2}\,\|w\|^2,
% \end{equation}
% where \(\lambda \ge 0\) is the regularization parameter. Notice that when \(\lambda = 0\), \(\hat{L}(w)\) reduces to the original loss \(L(w)\). Because adding \(\tfrac{\lambda}{2}\|w\|^2\) contributes \(\lambda w\) to the gradient, we obtain
% \begin{equation*}
%     \nabla \hat{L}(w) 
%     \;=\;
%     \nabla L(w) \;+\; \lambda\,w
%     \;=\;
%     \mathbb{E}\bigl[
%         2(\sigma(w^{\top} x) - \sigma(v^{\top} x))\,x\,\mathbbm{1}\{w^{\top} x \ge 0\}
%     \bigr]
%     \;+\;
%     \lambda\,w.
% \end{equation*}

    \begin{theorem}\label{thm:NegativeResultFull}
    Let \( L(\cdot) \) be the objective function for learning a single neuron (see Eq.~(\ref{eqn:Objective})), and let \( w_0 \in \mathbb{R}^{d+1} \) be an initialization satisfying \( (w_0)_{d+1} \leq 0 \). Define the target weight vector \( v \in \mathbb{R}^{d+1} \) such that \( v_{1:d} = -(w_0)_{1:d} \) and \( v_{d+1} < -(w_0)_{d+1} \). For any step size \( \alpha < \frac{1}{2(d+1)} \) and any regularization parameter \( \lambda < 2d \), every iterate \( w_t \) of gradient descent (starting from \( w_0 \)) that minimizes the L2-regularized objective \( \hat{L}(\cdot) \) satisfies
    \[
        L(w_t) \geq L(0) > 0.
    \]
\end{theorem}
\begin{proof}
        We first prove that for any iterate $w_t$, \begin{equation}\label{eqn:InductionHypothesis1}
            \{x\in\mathbb{R}^{d}\times\{1\}: w_t^{\top}x\geq 0\}\cap\{x\in\mathbb{R}^{d}\times\{1\}:v^{\top}x\geq0\} = \emptyset \text{ and } (w_t)_{d+1} \leq 0 
        \end{equation} We prove this claim by induction. For the base case of $w_0$, we let $x\in\mathbb{R}^d\times\{1\}$ be any vector such that $w_0^{\top}x \geq 0$, then \begin{align*}
            v^{\top}x & = v_{1:d}^{\top}x_{1:d}+v_{d+1} \\
            & = -(w_t)_{1:d}^{\top}x_{1:d} + v_{d+1} & \text{by choice of $v$} \\
            & < -(w_t)_{1:d}^{\top}x_{1:d} - w_{d+1} & \text{by choice of $v$} \\ 
            & = -w_0^{\top}x \\
            & \leq 0 & \text{by choice of $x$}
            \end{align*} By assumption $(w_0)_{d+1}\leq 0$, and therefore, (\ref{eqn:InductionHypothesis1}) holds for the initialization $w_0$. Next, we consider an arbitrary iterate $w_t$ and suppose that (\ref{eqn:InductionHypothesis1}) holds for $w_t$ and we proceed to show that (\ref{eqn:InductionHypothesis1}) holds for $w_{t+1}$. We let $y \in \mathbb{R}^d\times\{1\}$ be an arbitrary vector such  that $v^{\top}y \geq 0$. Then, \begin{align*}
                w_{t+1}^{\top}y & = w_t^{\top}y - \alpha \nabla L(w_t)^{\top}y -\alpha\lambda w_t^{\top} y& \text{by gradient descent update} \\
                & = (1-\alpha\lambda)w_t^{\top}y - 2\alpha\mathbb{E}[(\sigma(w_t^{\top}x) - \sigma(v^{\top}x))\mathbbm{1}_{\{w_t^{\top}x\geq 0\}}x]^{\top}y \\
                & = (1-\alpha\lambda)w_t^{\top}y - 2\alpha\mathbb{E}[\sigma(w_t^{\top}x)x]^{\top}y
            \end{align*} where the last line follows by the fact (\ref{eqn:InductionHypothesis1}) implies that $(\sigma(w_t^{\top}y) - \sigma(v^{\top}x))\mathbbm{1}_{\{w_t^{\top}x\geq 0\}} = \sigma(w_t^{\top}x)$. We let $x = x_w + x_u$ where $x_w = (x_{\tilde{w}},1)$ such that $x_{\tilde{w}}$ is the projection of $x_{1:d}$ onto the subspace spanned by $(w_t)_{1:d}$ and $x_u = (x_{\tilde{u}},0)$ such that $x_{\tilde{u}}$ is the orthogonal complement and hence $w_t^{\top}x_u = 0$. Then \begin{align*}
                \mathbb{E}[\sigma(w_t^{\top}x)x] & = \mathbb{E}[\sigma(w_t^{\top}(x_w+x_u))(x_w+x_u)] \\
                & = \mathbb{E}[\sigma(w_t^{\top}x_w)(x_w+x_u)] & \text{by $w_t^{\top}x_u =0$} \\
                & = \mathbb{E}[\sigma(w_t^{\top}x_w)x_w]  + \mathbb{E}[\sigma(w_t^{\top}x_w)x_u]
            \end{align*} Because $x_{1:d}\sim \mathcal{N}(0,I_d)$ is rotationally symmetric, its projections onto orthogonal subspaces are independent; consequently $x_{\tilde{w}}$ and $x_{\tilde{u}}$ are independent. Moreover, $w_t^{\top}x_w = (w_t)_{1:d}^{\top}x_{\tilde{w}} + w_{d+1}$ is independent of $x_{u}^{\top}y = x_{\tilde{u}}^{\top}y_{1:d}$. Therefore, \begin{align*}
                \mathbb{E}[\sigma(w_t^{\top}x_w)x_u]^{\top}y & = \mathbb{E}[\sigma(w_t^{\top}x_w)x_u^{\top}y] \\
                &  = \mathbb{E}[\sigma(w_t^{\top}x_w)]\mathbb{E}[x_u^{\top}y] & \text{by independence} \\
                & = 0
            \end{align*} where the last line follows by the fact that $x_u^{\top}y = x_{\tilde{u}}^{\top}y_{1:d}$ is a zero-mean normal random variable. Therefore, \begin{align*}
                w_{t+1}^{\top}y & = (1-\alpha\lambda)w_t^{\top}y -2\alpha \mathbb{E}[\sigma(w_t^{\top}x_w)x_w^{\top}y] \\
               & = (1-\alpha\lambda)(w_t)_{d+1}-2\alpha\mathbb{E}[\sigma(w_t^{\top}x_w)] + (1-\alpha\lambda)(w_t)_{1:d}^{\top}y_{1:d}-2\alpha\mathbb{E}[\sigma(w_t^{\top}x_w)x_{\tilde{w}}^{\top}y_{1:d}] \\
               & \leq (1-\alpha\lambda)(w_t)_{d+1} + (1-\alpha\lambda)(w_t)_{1:d}^{\top}y_{1:d}-2\alpha\mathbb{E}[\sigma(w_t^{\top}x_w)x_{\tilde{w}}^{\top}y_{1:d}] \\
            \end{align*}  where the last line above follows by $2\alpha\mathbb{E}[\sigma(w_t^{\top}x_w)] \geq 0$. Since $x_{\tilde{w}}$ is the projection of $x_{1:d}$ onto the subspace spanned by $(w_t)_{1:d}$ then \begin{equation*}
                x_{\tilde{w}}^{\top}y_{1:d} = \frac{(w_t)_{1:d}^{\top}x_{1:d}}{\norm{(w_t)_{1:d}}^2}(w_t)_{1:d}^{\top}y_{1:d}
            \end{equation*} Hence, \begin{equation*}
                2\alpha\mathbb{E}[\sigma(w_t^{\top}x_w)x_{\tilde{w}}^{\top}y_{1:d}] = 2\alpha\mathbb{E}[\sigma(w_t^{\top}x_w)\frac{(w_t)_{1:d}^{\top}x_{1:d}}{\norm{(w_t)_{1:d}}^2}](w_t)_{1:d}^{\top}y_{1:d}
            \end{equation*} Next, we argue that \begin{equation}\label{eqn:negativeresult1}
                0\leq 2\alpha\mathbb{E}[\sigma(w_t^{\top}x_w)\frac{(w_t)_{1:d}^{\top}x_{1:d}}{\norm{(w_t)_{1:d}}^2}] < \frac{1}{d+1}
            \end{equation} Since $(w_t)_{d+1} \leq 0$ then $w_t^{\top}x\geq 0$ implies that $(w_t)_{1:d}^{\top}x_{1:d} \geq 0$. Therefore the nonnegativity of (\ref{eqn:negativeresult1}) follows immediately. Moreover, this also implies that \begin{equation*}
                \sigma(w_t^{\top}x)(w_t)_{1:d}^{\top}x_{1:d} \leq \sigma((w_t)_{1:d}^{\top}x_{1:d})(w_t)_{1:d}^{\top}x_{1:d} \leq ((w_t)_{1:d}^{\top}x_{1:d})^2
            \end{equation*} By the fact that $x_{1:d}\sim\mathcal{N}(0,I_d)$ we have \begin{equation}\label{eqn:negativeresultZ}
                \frac{(w_t)_{1:d}^{\top}x_{1:d}}{\norm{(w_t)_{1:d}}} = Z \sim \mathcal{N}(0,1)
            \end{equation} Hence \begin{align*}
                2\alpha\mathbb{E}[\sigma(w_t^{\top}x_w)\frac{(w_t)_{1:d}^{\top}x_{1:d}}{\norm{(w_t)_{1:d}}^2}]& \leq 2\alpha\mathbb{E}[\frac{((w_t)_{1:d}^{\top}x_{1:d})^2}{\norm{(w_t)_{1:d}}^2}] \\
                & = 2\alpha \mathbb{E}[Z^2] &\text{by (\ref{eqn:negativeresultZ})} \\
                & = 2\alpha \\
                & < \frac{1}{d+1} & \text{by choice of $\alpha$}\\
            \end{align*} Hence, (\ref{eqn:negativeresult1}) holds and therefore there exists some $D\in [0,\frac{1}{d+1})$ such that \begin{align*}
                (1-\alpha\lambda)(w_t)_{1:d}^{\top}y_{1:d} - 2\alpha\mathbb{E}[\sigma(w_t^{\top}x_w)x_{\tilde{w}}^{\top}y_{1:d}] & = (1-\alpha\lambda-D)(w_t)_{1:d}^{\top}y_{1:d}\\
                % 2\alpha\mathbb{E}[\sigma(w_t^{\top}x_w)\frac{(w_t)_{1:d}^{\top}x_{1:d}}{\norm{(w_t)_{1:d}}^2}](w_t)_{1:d}^{\top}y_{1:d} 
                % & = (1-\alpha\lambda - 2\alpha\mathbb{E}[\sigma(w_t^{\top}x_w)\frac{(w_t)_{1:d}^{\top}x_{1:d}}{\norm{(w_t)_{1:d}}^2}]) (w_t)_{1:d}^{\top}y_{1:d} \\
                % & = (1-\alpha\lambda-D)(w_t)_{1:d}^{\top}y_{1:d}
            \end{align*} Continuing from earlier \begin{align*}
                w_{t+1}^{\top}y & \leq (1-\alpha\lambda)(w_t)_{d+1} + (1-\alpha\lambda)(w_t)_{1:d}^{\top}y_{1:d}-2\alpha\mathbb{E}[\sigma(w_t^{\top}x_w)x_{\tilde{w}}^{\top}y_{1:d}] \\
                & = (1-\alpha\lambda)(w_t)_{d+1} + (1-\alpha\lambda-D)(w_t)_{1:d}^{\top}y_{1:d} \\
                & \leq (1-\alpha\lambda - D)(w_t)_{d+1} + (1-\alpha\lambda-D)(w_t)_{1:d}^{\top}y_{1:d} & \text{by $(w_t)_{d+1} \leq 0$ and $D \geq 0$}\\
                & = (1-\alpha\lambda - D)w_t^{\top}y \\
            \end{align*}
            Given that $\alpha < \frac{1}{2(d+1)}$, $\lambda < 2d$, and $D \in [0,\frac{1}{d+1})$ then \begin{equation*}
                0 < 1-\alpha\lambda -D \leq 1
            \end{equation*} and therefore \begin{equation*}
                w_{t+1}^\top y \leq (1-\alpha\lambda -D) w_t^\top y < 0
            \end{equation*} where the final inequality above follows by the inductive hypothesis (\ref{eqn:InductionHypothesis1}) and the fact that $v^\top y \geq 0$. We additionally note that \begin{equation*}
                (w_{t+1})_{d+1} = (1-\alpha\lambda)(w_t)_{d+1} - 2\alpha\mathbb{E}[\sigma(w_t^{\top}x)] \leq (1-\alpha\lambda)(w_t)_{d+1} \leq 0
            \end{equation*} where the finally inequality holds by $(w_t)_{d+1}\leq 0$ due to the inductive hypothesis (\ref{eqn:InductionHypothesis1}) and the fact that $1-\alpha\lambda \in (\frac{1}{d+1},1]$. Hence (\ref{eqn:InductionHypothesis1}) holds for the iterate $w_{t+1}$. Finally, we show that if $(\ref{eqn:InductionHypothesis1})$ holds then \begin{equation*}
                L(w_t) \geq L(0) > 0
            \end{equation*} We first note that (\ref{eqn:InductionHypothesis1}) implies \begin{equation}\label{eqn:negativeresult3}
                \forall x \in \mathbb{R}^d\times\{1\}, \, \sigma(w_t^{\top}x)\sigma(v^{\top}x) = 0
            \end{equation} and so \begin{align*}
                L(w_t) & = \mathbb{E}[(\sigma(w_t^{\top}x)-\sigma(v^{\top}x))^2] \\
                & = \mathbb{E}[\sigma(w_t^{\top}x)^2] + \mathbb{E}[\sigma(v^{\top}x)^2] - 2\mathbb{E}[\sigma(w_t^{\top}x)\sigma(v^{\top}x)] \\
                & = \mathbb{E}[\sigma(w_t^{\top}x)^2] + \mathbb{E}[\sigma(v^{\top}x)^2] & \text{by (\ref{eqn:negativeresult3})} \\
                & \geq \mathbb{E}[\sigma(v^{\top}x)^2] \\
                & = L(0) \\
                & > 0 & \text{by $v\neq 0$}
            \end{align*}
    \end{proof}
\section{Proofs of Proposition of \ref{prop:prop1} and \ref{prop:prop2}}

\begin{proposition}[Restatement of Proposition \ref{prop:prop1}]
Let $T$ be the $1- \frac{\lambda}{p-\lambda}$ percentile of a {\rm Geometric}$(p)$ distribution. Then the optimal hypothesis test takes the form $X_\tau = \mathbf{1}\{Z_\tau = 0\}$ where $\tau = {\rm min}(s:Z_s = 1) \wedge T$.   
\end{proposition}

\begin{proof}
We begin by assuming equal priors $\mathbb{P}(H_0) = \mathbb{P}(H_1) = \frac{1}{2}$. We note that for any time $s$, if $Z_s = 1$ then any optimal hypothesis test must declare $X_s = 0$ as $Z_s = 1$ is impossible under $H_1$ and waiting to make a future declaration will incur additional cost of at least $\lambda$. Therefore, it remains for us to derive an optimal stopping time for the collection of states $\{Z_1=\ldots=Z_s = 0: s\in \mathbb{Z}_+\}$.

Let $V(s)$ be the expected total future cost at time $s$ given that we have observed $Z_1=\ldots=Z_s = 0$. We define \begin{align*}
    \pi_s & = \mathbb{P}(H_0|Z_1=\ldots=Z_s=0) \\
    & = \frac{\mathbb{P}(H_0, Z_1=\ldots=Z_s=0)}{\mathbb{P}(Z_1=\ldots=Z_s=0)} \\
    & = \frac{(1-p)^s\mathbb{P}(H_0)}{(1-p)^s\mathbb{P}(H_0)+\mathbb{P}(H_1)} \\
    & = \frac{(1-p)^s}{(1-p)^s+1} & \text{by $\mathbb{P}(H_0)=\mathbb{P}(H_1)$}
\end{align*} 

If we stop at time $s$ and make a declaration, we choose the hypothesis with higher positive probability in order tom minimize the error probability \begin{equation*}
    \mathbb{P}(X_s = 1|H_0) + \mathbb{P}(X_s =0|H_1)
\end{equation*} 
Thus, the expected cost of stopping is \begin{equation*}
    C^{\text{stop}}(s) = \min\{\pi_s, 1-\pi_s\}
\end{equation*} We can simplify this further by noting that $\pi_s \leq 1-\pi_s$. We note that \begin{align*}
    \frac{1}{2}(1-p)^s & \leq \frac{1}{2}
\end{align*} by $1-p\in[0,1]$. This is equivalent to \begin{equation*}
    (1-p)^s \leq \frac{1}{2}((1-p)^s + 1)
\end{equation*} by adding $\frac{1}{2}(1-p)^s$, which in turn, is equivalent to \begin{equation*}
    \pi_s = \frac{(1-p)^s}{(1-p)^s+1} \leq \frac{1}{2}
\end{equation*} Therefore, $\pi_s \leq \frac{1}{2} \leq 1-\pi_s$ and so we have that \begin{equation*}
    C^{\text{stop}}(s)  =\min\{\pi_s,1-\pi_s\} = \pi_s
\end{equation*} This also implies that if we are to stop at some state $\{Z_1=\ldots=Z_s=0\}$, it is optimal to declare $X_s = 1$. 

If we continue at time $s$ to $s+1$, we incur an additional delay cost of $\lambda$, and the expected future cost depending on whether we see a $Z_{s+1}=1$ or $Z_{s+1} = 0$. \begin{itemize}
    \item With probability $p\pi_s$ we obserbes $Z_{s+1}$, under $H_0$, and we stop the process with $X_{s+1}=0$, incurring zero error cost since $Z_{s+1}=1$ cannot occur under $H_1$.
    \item With probability $(1-p)\pi_s +(1-\pi) = 1 -p\pi_s$ we observe $Z_{s+1}= 0$ and the process continues. 
\end{itemize} Therefore, the expected cost of continuing at time $s$ is \begin{equation*}
    C^{\text{cont}}(s) = \lambda + (1-p\pi_s)V(s+1)
\end{equation*} Then the Bellman equation for the optimal cost-to-go function is \begin{equation*}
    V(s) = \\min\{C^{\text{stop}}(s),C^{\text{cont}}(s)\}
\end{equation*} To determine an optimal stopping time, our goal is to find smallest $T$ for which \begin{equation}\label{eqn:Eqn1ofProp}
    C^{\text{stop}}(T) \leq C^{\text{cont}}(T)
\end{equation}

Assuming we stop at time $T$, \begin{equation*}
    V(T+1) = C^{\text{stop}}(T+1) = \pi_{T+1}
\end{equation*} Therefore, \begin{equation*}
    C^{\text{cont}}(T) = \lambda + (1-p\pi_s)V(T+1) =\lambda+(1-p\pi_s)\pi_{T+1}
\end{equation*} and to establish (\ref{eqn:Eqn1ofProp}) it suffices to show that \begin{equation}\label{eqn:Eqn2ofProp}
    \pi_T \leq \lambda + (1-p\pi_T)\pi_{T+1}
\end{equation} First, we write $\pi_{T+1}$ in terms of $\pi_T$. Under the updating rule for the posterior probability, we have that \begin{align*}
    \pi_{T+1} & - \mathbb{P}(H_0|Z_1=\ldots=Z_{T+1}=0) \\
    & = \frac{\mathbb{P}(Z_{T+1}=0|H_0)\mathbb{P}(H_0|Z_1=\ldots=Z_T=0)}{\mathbb{P}(Z_{T+1}=0|Z_1=\ldots=Z_T=0)} \\
    & = \frac{(1-p)\pi_T}{\mathbb{P}(Z_{T+1}=0|H_0)\pi_T + \mathbb{P}(Z_{T+1}=0|H_1)(1-\pi_T)} \\
    & = \frac{(1-p)\pi_T}{(1-p)\pi_T + (1-\pi_T)} \\
    & = \frac{(1-p)\pi_T}{1-p\pi_T}
\end{align*} Returning to (\ref{eqn:Eqn2ofProp}), we need to show that \begin{equation*}
    \pi_T \leq \lambda + (1-p\pi_T)\frac{(1-p)\pi_T}{1-p\pi_T} = \lambda + (1-p)\pi_T 
\end{equation*} Simplifying the above inequality, we have that \begin{equation*}
    \pi_T \leq \frac{\lambda}{p}
\end{equation*} Substituting in our formula for $\pi_T$, the above is equivalent to \begin{equation*}
    \frac{(1-p)^{\top}}{(1-p)^{\top}+1} \leq \frac{\lambda}{p}
\end{equation*} which after simplification is equivalent to \begin{equation*}
    (1-p)^{\top} \leq \frac{\frac{\lambda}{p}}{1-\frac{\lambda}{p}} = \frac{\lambda}{p-\lambda}
\end{equation*} 

Let $F$ be the CDF of the Geometric$(p)$ distribution. Let $T^*$ be the $1- \frac{\lambda}{p-\lambda}$ percentile of the Geometric$(p)$ distribution. Note, since $\lambda < \frac{p}{2}$ then $1- \frac{\lambda}{p-\lambda}\in(0,1)$ and is a valid percentile. Then for any $T\geq T^*$ we have that \begin{align*}
    1 -(1-p)^{\top} &=F(T) \\
    & \geq F(T^*) & \text{by $T\geq T^*$} \\
    & \geq 1 - \frac{\lambda}{p-\lambda} & \text{by choice of $T^*$}
\end{align*} which is equivalent to \begin{equation*}
    (1-p)^{\top} \leq \frac{\lambda}{p-\lambda}
\end{equation*} and therefore, the optimal hypothesis test is to declare $X_T = 1$ for any $T\geq T^*$ if $Z_1=\ldots=Z_T =0$. Hence, the optimal hypothesis takes the form of \begin{equation*}
    X_\tau = \mathbf{1}\{Z_\tau = 0\} \text{ where } \tau = \min(s:Z_s = 1) \land T^*
\end{equation*}
\end{proof}

\begin{proposition}Restatement of Proposition \ref{prop:prop2}]
The ratio of total error rate with a fixed threshold $r^*$ to that under SNR scales like
\[
\Omega\left(
\exp
\left(
\log(\alpha(1-\alpha)^{-1})
\left( 
-\frac{1}{2} + \frac{1}{2} \frac{\log (1-p_1)}{\log(1-p_2)}
\right)
\right)
\right)
\]
\end{proposition}

\begin{proof}
For notational convenience, define $\bar \alpha = \alpha (1-\alpha)^{-1}$. Notice that by Proposition~\ref{prop:prop1_appendix}, under SNR, neuron $i$, $(I=1,2)$, is reset if it inactive for any longer that time $\bar T = \log(\bar \alpha) )/\log (1-p_i)$.  Consequently, the expected delay penalty, $\E[\tau|H_0] + \E[\tau|H_1]$ for neuron $i$, is simply
\[
\frac{\log(\bar \alpha)}{\log(1-p_i)} 
+ \frac{1}{p_i} (1-\bar \alpha) 
\]
Letting the optimal fixed threshold be $r^*$, we must have that the expected total delay across both neurons under this fixed threshold is at least $2r^*$. This total expected delay can be no larger than that under SNR. Thus,
\[
2r^*
\leq 
\log(\bar \alpha)
\left(
\frac{1}{\log(1-p_1)} 
+
\frac{1}{\log(1-p_2)} 
\right)
+ 
(1-\bar \alpha)
\left(
\frac{1}{p_1} 
+\frac{1}{p_2}
\right)  
\]
But the sum of the error rates across the two neurons with the fixed threshold $r^*$ is at least $(1-p_1)^{r^*}+ \bar\alpha$, while the total error rate under SNR is precisely $\bar \alpha$. Dividing these two quantities and employing the upper bound derived on $r^*$ then yields the result.
\end{proof}

\end{document}